\documentclass{article}
\usepackage[final]{neurips_2022}

\usepackage[utf8]{inputenc} % allow utf-8 input
\usepackage[T1]{fontenc}    % use 8-bit T1 fonts
\usepackage{hyperref}       % hyperlinks
\usepackage{url}            % simple URL typesetting
\usepackage{booktabs}       % professional-quality tables
\usepackage{amsfonts}       % blackboard math symbols
\usepackage{nicefrac}       % compact symbols for 1/2, etc.
\usepackage{microtype}      % microtypography
\usepackage{xcolor}         % colors

\usepackage{microtype}
\usepackage{graphicx}
\usepackage{natbib}
\setcitestyle{square,sort,comma,numbers}

\usepackage{subfigure}
\usepackage{wrapfig}
\usepackage{capt-of}
\usepackage{relsize}

\usepackage{bbm}

\usepackage{amsthm}
\usepackage{amsmath, amssymb, amsthm, amscd, amstext}
\usepackage{algorithm, algorithmic}
\usepackage{mathtools}

\usepackage[capitalize,noabbrev]{cleveref}

\theoremstyle{plain}

\theoremstyle{definition}

\theoremstyle{remark}

\usepackage{wrapfig}
\usepackage{capt-of}
\usepackage{relsize}
\usepackage{multirow}

\newcommand{\fv}{{\boldsymbol f}}

\newcommand{\Wv}{{\boldsymbol W}}

\newcommand{\Xv}{{\boldsymbol X}}

\newcommand{\thetav}{{\boldsymbol \theta}}

\usepackage{soul}

\newcommand{\cov}{\text{cov}}
\newcommand{\EE}{\mathbb{E}}

\newcommand{\bs}[1]{\boldsymbol{#1}}
\newcommand{\BR}{\mathbb{R}}
\newcommand{\PP}{\mathbb{P}}

\newcommand{\ud}{\,\text{d}}
\newcommand{\CX}{\mathcal{X}}
\newcommand{\CY}{\mathcal{Y}}
\newcommand{\CZ}{\mathcal{Z}}

\newcommand{\CE}{\mathcal{E}}
\newcommand{\CF}{\mathcal{F}}

\newcommand{\KL}{\text{KL}}

\newcommand{\JSD}{\text{JSD}}

\newcommand{\BS}{\mathbb{S}}

\newcommand{\CM}{\mathcal{M}}

\DeclareMathOperator{\argmax}{\arg\max}

\newcommand{\CL}{\mathcal{L}}

\renewcommand{\CD}{\mathcal{D}}

\usepackage{tikz}
\theoremstyle{plain}
\newtheorem{thm}{Theorem}[section] % reset theorem numbering for each chapter
\theoremstyle{definition}
\newtheorem{defn}[thm]{Definition} % definition numbers are dependent on theorem numbers
\newtheorem{prop}[thm]{Proposition} % same for example numbers
\newtheorem{assumption}[thm]{Assumption}

\newtheorem{col}[thm]{Corollary}

\newcommand{\beq}{\begin{equation}}
\newcommand{\eeq}{\end{equation}}
\newcommand{\beqs}{\begin{eqnarray}}
\newcommand{\eeqs}{\end{eqnarray}}
\newcommand{\barr}{\begin{array}}
\newcommand{\earr}{\end{array}}

\newcommand{\ELBO}{\text{ELBO}}

\newcommand{\CV}{\mathcal{V}}

\newcommand{\Gv}{\bs{G}}

\newcommand{\BX}{\mathbb{X}}
\newcommand{\BY}{\mathbb{Y}}

\newcommand{\infonce}{\texttt{InfoNCE}}

\newcommand{\BA}{\texttt{BA}}
\newcommand{\UBA}{\texttt{UBA}}
\newcommand{\TUBA}{\texttt{TUBA}}
\newcommand{\NWJ}{\texttt{NWJ}}
\newcommand{\DV}{\texttt{DV}}
\newcommand{\MINE}{\texttt{MINE}}
\newcommand{\ANCE}{\texttt{$\alpha$-InfoNCE}}
\newcommand{\FLO}{\texttt{FLO}}
\newcommand{\PMI}{\text{PMI}}
\renewcommand{\JSD}{\texttt{JSD}}
\newcommand{\FLAT}{\texttt{FlatNCE}}

\newcommand{\FDV}{\texttt{FDV}}

\newcommand{\DP}{\text{DP}}

\newcommand{\SimCLR}{\texttt{SimCLR}}

\newcommand{\FCLR}{\texttt{FlatCLR}}

\newcommand{\SIR}{\text{SIR}}

\newcommand{\CA}{\mathcal{A}}
\newcommand{\meta}{\text{meta}}
\newcommand{\adapt}{\text{adapt}}

\newcommand{\metaflo}{\texttt{Meta-FLO}}
\newcommand{\upper}{\text{upper}}

\newcommand{\BQ}{\mathbb{Q}}
\renewcommand{\CD}{\mathcal{D}}
\newcommand{\te}{\text{test}}

\usepackage[textsize=tiny]{todonotes}

\title{Tight Mutual Information Estimation With Contrastive Fenchel-Legendre Optimization}

\author{%
	Qing Guo${}^{1,\dagger}$, Junya Chen${}^2$, Dong Wang${}^2$, Yuewei Wang${}^2$, Xinwei Deng${}^1$\\
	{\bf Lawrence Carin${}^{2,3}$, Fan Li${}^2$, Jing Huang${}^4$, Chenyang Tao${}^{2,4,\dagger}$} \\ ${}^1$Virginia Tech ${}^2$Duke University  ${}^3$KAUST ${}^4$Amazon \\
	\texttt {qguo0701@vt.edu},
	\texttt {chenyang.tao@duke.edu}
}

\begin{document}

	\maketitle

	\begin{abstract}
		Successful applications of InfoNCE (Information Noise-Contrastive Estimation) and its variants have popularized the use of contrastive variational mutual information (MI) estimators in machine learning. While featuring superior stability, these estimators crucially depend on costly large-batch training, and they sacrifice bound tightness for variance reduction. To overcome these limitations, we revisit the mathematics of popular variational MI bounds from the lens of unnormalized statistical modeling and convex optimization. Our investigation yields a new unified theoretical framework encompassing popular variational MI bounds, and leads to a new simple and powerful contrastive MI estimator we name $\FLO$. Theoretically, we show that the $\FLO$ estimator is tight, and it converges under stochastic gradient descent. Empirically, the $\FLO$ estimator overcomes the limitations of its predecessors and learns more efficiently. The utility of $\FLO$ is verified using extensive benchmarks, and we further inspire the community with novel applications in meta-learning. Our presentation underscores the foundational importance of variational MI estimation in data-efficient learning. 
	\end{abstract}

	\vspace{-1.2em}
	\section{Introduction}
	\vspace{-5pt}
	
	Assessing the dependence between pairs of variables is integral to many scientific and engineering endeavors \citep{reshef2011detecting, shannon1948mathematical}. {\it Mutual information} (MI) is a popular metric to quantify generic associations \citep{mackay2003information}, and its empirical estimators have been widely used in applications such as independent component analysis \citep{bach2002kernel}, fair learning \citep{gupta2021controllable}, neuroscience \citep{palmer2015predictive}, Bayesian optimization \citep{kleinegesse2020bayesian}, among others. Notably, the recent advances in deep {\it self-supervised learning} (SSL) heavily rely on nonparametric MI optimization  \citep{tishby2015deep, oord2018representation, he2020momentum, chen2020simple, grill2020bootstrap}. 
	In this study we investigate the likelihood-free variational approximation of MI using only paired samples, and improve the data-efficiency of current machine learning practices. 
	
	MI estimation has been extensively studied \citep{battiti1994using, maes1997multimodality, mackay2003information, paninski2003estimation, pluim2003mutual, torkkola2003feature}. While most classical estimators work reasonably well for low-dimensional cases, they scale poorly to big datasets: 
	na\"ive density-based estimator(s) and $k$-nearest neighbor estimators \citep{kraskov2004estimating, perez2008estimation, gao2015efficient} struggle with high-dimensional inputs, while kernel estimators are slow, memory demanding and sensitive to hyperparameters \citep{gretton2003kernel, gretton2005kernel}. Moreover, these estimators are usually either non-differentiable or need to hold all data in memory. Consequently, they are not well suited for emerging applications where the data representation needs to be differentiably optimized based on small-batch estimation of MI \citep{hjelm2019learning}. Alternatively, one can approach MI estimation through an estimated likelihood ratio \citep{suzuki2008approximating, hjelm2019learning}, but the associated numerical instability has raised concerns \citep{arjovsky2017towards}.

	To scale MI estimation to the growing size and complexity of modern datasets, and to accommodate the need for representation optimization \citep{bengio2013representation}, variational objectives have been widely utilized recently \citep{oord2018representation}. Instead of directly estimating data likelihoods, density ratios, or the corresponding gradients \citep{wen2020mutual}, variational approaches appeal to mathematical inequalities to construct tractable lower or upper bounds of the mutual information \citep{poole2019variational}, facilitated by the use of auxiliary critic functions\footnote{When estimates are sharp, these critic functions usually recover some transformation of the likelihood ratio.}. This practice turns MI estimation into an optimization problem. Prominent examples include the {\it Barber-Agakov} (BA) estimator \citep{agakov2004algorithm},  the {\it Donsker-Varadhan} (DV) estimator \citep{donsker1983asymptotic}, and the {\it Nguyen-Wainwright-Jordan} (NWJ) estimator \citep{nguyen2010estimating}. These variational estimators are closely connected to the variational objectives for likelihood inference \citep{alemi2018fixing}.

	Despite reported successes, these variational estimators have a major limitation: their estimation variance grows exponentially to the ground-truth MI \citep{mcallester2018formal}. This is especially harmful to applications involving deep neural nets, as it largely destabilizes training \citep{song2020understanding}. An effective fix is to leverage multi-sample contrastive estimators, pioneered by the work of \texttt{InfoNCE} \citep{oord2018representation}. However, the massive reduction in the variance comes at a price: the performance of the \texttt{InfoNCE} estimator is upper bounded by $\log K$, where $K$ is the number of {\it negative} samples used \citep{poole2019variational}. For a large MI, $K$ needs to be sufficiently large to allow for an adequate estimate, consequently placing a significant burden on computation and memory. 
	While variants of $\infonce$ have been motivated to achieve more controllable bias and variance tradeoffs \citep{poole2019variational, song2020understanding}, little research has been conducted on the cost-benefit aspect of contrastive learning. 
	
	\begin{wrapfigure}[17]{R}{0.38\textwidth}
		\vspace{-2.em}
		\scalebox{.93}{
			%\hspace{-.5em}
			\begin{minipage}{.39\textwidth}
				\begin{figure}[H]
					\begin{center}{
							\hspace{-10pt}\includegraphics[width=1.05\textwidth]{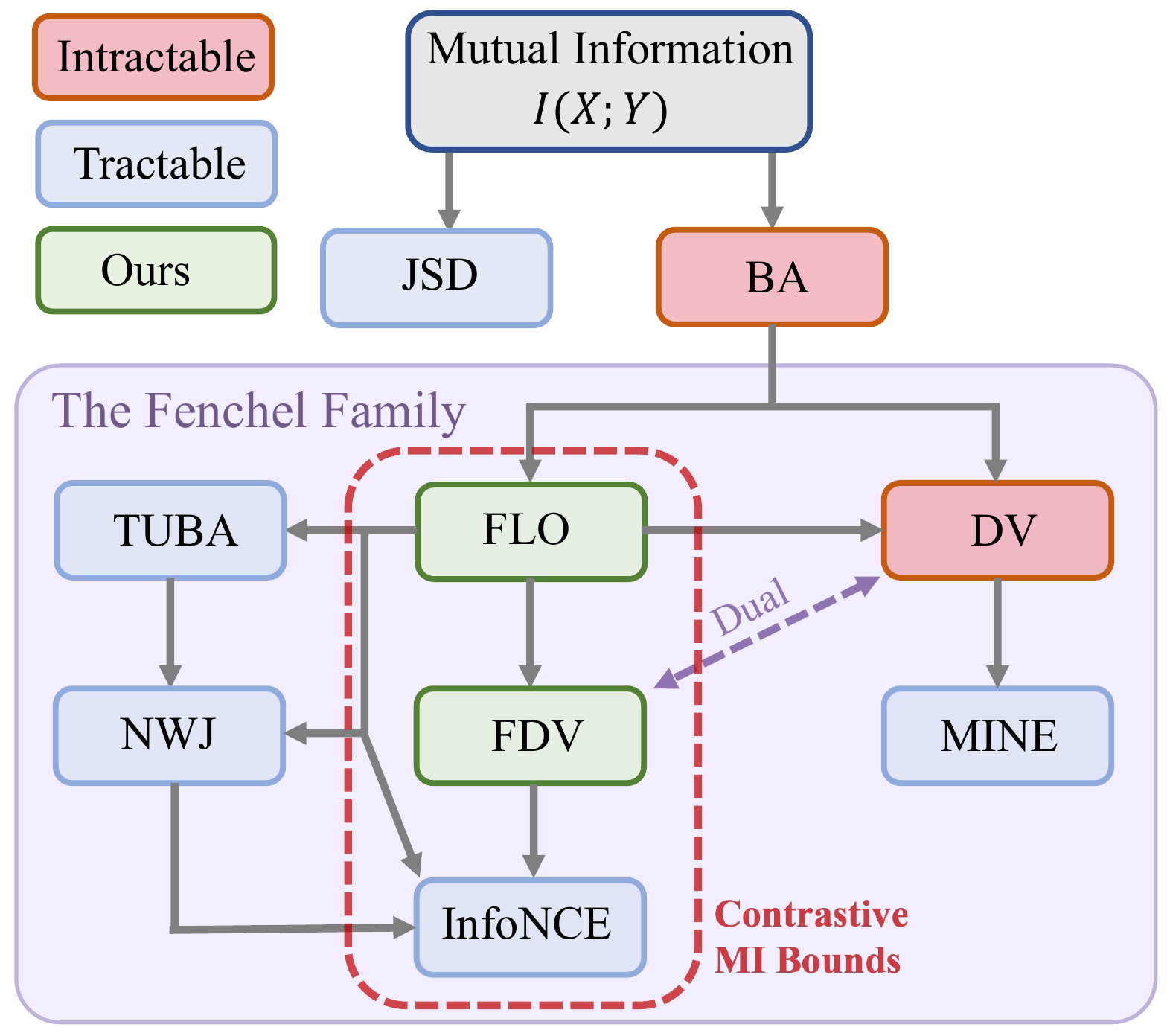}
							% \hspace{-1.3em}\includegraphics[width=1.1\textwidth]{figures/results/mi_opt}
						}
						\vspace{-1.5em}
						% \caption{\small {$\FLAT$} better optimizes the true MI with a finite sample loss. \label{fig:mi_opt}}
						\caption{Schematic of variational lower bounds of mutual information. $\FLO$ provides a novel unified framework to analyze contrastive MI bounds. \label{fig:schematic}}
						%{$\FLAT$} learns a representation with better ground-truth MI even if its $K$-sample MI loss is tied with that of $\infonce$.
					\end{center}
					% updated policy sampler
					\vspace{-1.5em}
				\end{figure}
			\end{minipage}
		}
	\end{wrapfigure}
	
	A critical insight enabled by $\infonce$ is that mutual information closely connects to contrastive learning \citep{gutmann2010noise, oord2018representation}. Paralleled by the empirical successes of instance discrimination-based self-supervision \citep{mnih2013learning, wu2018unsupervised, chen2020simple, he2020momentum} and multi-view supervision \citep{tian2019contrastive, radford2021learning}, $\infonce$ offers an {\it InfoMax} explanation to why the ability to discriminate naturally paired {\it positive} instances from the randomly paired {\it negative} instances leads to universal performance gains in these applications \citep{linsker1988self, shwartz2017opening, poole2019variational}. Despite these encouraging developments, the big picture of MI optimization and contrastive learning is not yet complete: ($i$) There is an ongoing debate about to what extent MI optimization helps to learn \citep{tschannen2020mutual}; ($ii$) how does the contrastive view reconcile with those non-contrastive MI estimators; crucial for practical applications, ($iii$) are the empirical tradeoffs made by estimators such as $\infonce$ absolutely necessary? And theoretically, ($iv$) formal guarantees on the statistical convergence of popular variational non-parametric MI estimation are missing currently.

	%In this work we seek to bridge the above gaps by approaching the MI estimation from the novel perspective of energy modeling. While this subject has recently been studied extensively using information-theoretic and variational inequalities, we embrace a new view from the lens of unnormalized statistical modeling. Our main contributions are four-fold: ($i$) unifying popular variational MI bounds under unnormalized statistical modeling; ($ii$) deriving a simple but powerful novel contrastive variational bound called $\FLO$; ($iii$) providing theoretical justification of the $\FLO$ bound, such as tightness and the first convergence result for variational estimators; and ($iv$) demonstrating strong empirical evidence of the superiority of the new $\FLO$ bound over its predecessors. We contribute in-depth discussion to bridge the gaps between contrastive learning and MI estimation, along with principled practical guidelines informed by theoretical insights. The importance of MI in data-efficient machine learning is highlighted with novel applications. 
	
	In this work we seek to bridge the above gaps by approaching the MI estimation from the novel perspective of energy modeling. While this subject has recently been studied extensively using information-theoretic and variational inequalities, we embrace a new view from the lens of unnormalized statistical modeling. Our main contributions include: 
	\begin{itemize}
		\vspace{-5pt}
		\item Unifying popular variational MI bounds under unnormalized statistical modeling; 
		\item Deriving a simple but powerful novel contrastive variational bound called $\FLO$; 
		\item Providing theoretical justification of the $\FLO$ bound (tightness and convergence);
		%, such as tightness and the first convergence result for variational estimators; 
		\item Demonstrating strong empirical evidence of the superiority of $\FLO$ over its predecessors.
		\item Highlighting the importance of MI in data-efficient learning with novel applications 
		\vspace{-5pt}
	\end{itemize} 
	We contribute in-depth discussion to bridge the gaps between contrastive learning and MI estimation, along with principled practical guidelines informed by theoretical insights. 
	%The importance of MI in data-efficient machine learning is highlighted.
	%The importance of MI in data-efficient machine learning is highlighted with novel applications. 

	\vspace{-8pt}
	\section{Fenchel-Legendre Optimization for Mutual Information Estimation}
	\label{sec:flo}
	\vspace{-5pt}
	
	\subsection{Preliminaries}
	\vspace{-4pt}
	
	This section briefly reviews the mathematical background needed for our subsequent developments. 
	
	{\bf Unnormalized statistical modeling} defines a rich class of models of general interest. Specifically, we are interested in  problems for which the system is characterized by an energy function $\tilde{p}_{\theta}(x) =  \exp(-\psi_\theta(x))$, where $\theta$ is the system parameters and $\psi_{\theta}(x)$ is known as the {\it potential function}. The goal is to find a solution that is defined by a normalized version of $\tilde{p}_{\theta}(x)$, {\it i.e.}, $\min_{\theta} \left\{ \CL\left( \frac{\tilde{p}_{\theta}}{\int \tilde{p}_{\theta}(x') \ud \mu(x')}\right)\right\}$, 
	% \beq
	% %\min_{\theta} \{ \CL(p_{\theta})\}, \text{ subject to } p_{\theta}(x) = \frac{\tilde{p}_{\theta}(x)}{Z(\theta)}, \,\, Z(\theta) \triangleq \int \tilde{p}_{\theta}(x') \ud \mu(x'), 
	% \min_{\theta} \left\{ \CL\left( \frac{\tilde{p}_{\theta}}{\int \tilde{p}_{\theta}(x') \ud \mu(x')}\right)\right\}, 
	% %\text{ subject to } p_{\theta}(x) = \frac{\tilde{p}_{\theta}(x)}{Z(\theta)}, \,\, Z(\theta) \triangleq \int \tilde{p}_{\theta}(x') \ud \mu(x'), 
	% \eeq
	where $\CL(\cdot)$ is the loss function, $\mu$ is the base measure on $\CX$ and $Z(\theta) \triangleq \int \tilde{p}_{\theta}(x') \ud \mu(x')$ is called the {\it partition function} for $\tilde{p}_{\theta}(x)$. 
	Problems in the above form arise naturally in statistical physics \citep{reichl2016modern}, Bayesian analysis \citep{berger2013statistical}, and maximal likelihood estimation \citep{tao2019fenchel}. A major difficulty with unnormalized statistical modeling is that the partition function $Z(\theta)$ is generally intractable for complex energy functions \footnote{In the sense that they do not render closed-from expressions.}, and in many applications $Z(\theta)$ is further composed by $\log Z(\theta)$, whose concavity implies any finite sample estimate Monte-Carlo of $Z(\theta)$ will render the loss function biased \citep{rainforth2018nesting, zheng2018robust}. Bypassing the difficulties caused by the intractable partition function is central to unnormalized statistical modeling \citep{geyer1994convergence, neal2001annealed, hinton2002training, hyvarinen2005estimation, gutmann2010noise}.
	
	{\bf Mutual information and unnormalized statistical models.} %We first establish the connection between MI estimation and unnormalized statistical modeling. 
	As a generic score assessing the dependency between two random variables $(X,Y)$, {\it mutual information} is formally defined as the {\it Kullback-Leibler divergence} (KL) between the joint distribution $p(x,y)$ and product of the respective marginals $p(x)p(y)$ \citep{shannon1948mathematical}, {\it i.e.}, $I(X;Y) \triangleq \EE_{p(x,y)}\left[\log \frac{p(x,y)}{p(x)p(y)}\right]$. The integrand $\log \frac{p(x,y)}{p(x)p(y)}$ is often known as the {\it point-wise mutual information} ($\PMI$) in the literature. Mutual information has a few appealing properties: ($i$) it is invariant wrt invertible transformations of $x$ and $y$, and ($ii$) it has the intuitive interpretation of reduced uncertainty of one variable given another variable\footnote{Formally, $I(X;Y) = H(X) - H(X|Y) = H(Y) - H(Y|X)$, where $H(X)$ (resp. $H(X|Y)$) denotes the Shannon entropy (resp. conditional Shannon entropy) of a random variable.}. 
	
	To connect MI to unnormalized statistical modeling, we consider the classical {\it Barber-Agakov} ($\BA$) estimator of MI \citep{barber2003information}.  To lower bound MI, $\BA$ introduces a variational approximation $q(y|x)$ for the posterior $p(y|x)$, and by rearranging the terms we obtain an inequality
	\beqs
	%& & I(X;Y) \nonumber \\
	%& = &  \EE_{p(x,y)}\left[ \log \frac{p(y|x)}{p(y)} \right] \nonumber\\
	%& & I(X;Y) \nonumber = \EE_{p(x,y)}\left[ \log \frac{p(y|x)}{p(y)} \right] \nonumber\\
	%& = & \EE_{p(x,y)}\left[ \log \frac{q(y|x)}{p(y)} \right] + \EE_{p(x)}[\KL(p(y|x) \parallel q(y|x))] \nonumber\\
	%& \geq & \EE_{p(x,y)}\left[ \log \frac{q(y|x)}{p(y)} \right]  \triangleq I_{\BA}(X;Y|q). \label{eq:ba}
	I(X;Y) & = & \EE_{p(x,y)}\left[ \log \frac{p(y|x)}{p(y)} \right] = \EE_{p(x,y)}\left[ \log \frac{q(y|x)}{p(y)} \right] + \EE_{p(x)}[\KL(p(y|x) \parallel q(y|x))] \nonumber\\
	& \geq & \EE_{p(x,y)}\left[ \log \frac{q(y|x)}{p(y)} \right]  \triangleq I_{\BA}(X;Y|q). \label{eq:ba}
	%\vspace{-1em}
	\eeqs
	Here we have used notation $I_{\BA}(X;Y|q)$ to highlight the dependence on $q(y|x)$, and when $q(y|x)=p(y|x)$ this bound is sharp. Unfortunately, this na\"ive $\BA$ bound is not useful for sample-based MI estimation, as we do not know the ground-truth $p(y)$. But we can bypass this difficulty by setting $q_{\theta}(y|x) = \frac{p(y)}{Z_{\theta}(x)}e^{g_{\theta}(x,y)}$, where we call $e^{g_{\theta}(x,y)}$ the {\it tilting function} and recognize $Z_{\theta}(x) = \EE_{p(y)}[e^{g_{\theta}(x,y)}]$ as the associated partition function. Substituting this $q_{\theta}(x|y)$ into (\ref{eq:ba}) gives the following {\it unnormalized $\BA$} bound ($\UBA$) that pertains to unnormalized statistical modeling \citep{poole2019variational}
	\beq
	\label{eq:uba}
	%\begin{array}{l}
	%I_{\UBA}(X;Y|g_{\theta}) \triangleq \EE_{p(x,y)}[g_{\theta}(x,y)] - \EE_{p(x)}[\log Z_{\theta}(x)] = \EE_{p(x)}\left[ \EE_{p(y|x)}\left[  \log \frac{\exp(g_{\theta}(x,y))}{Z_{\theta}(x)}\right]\right].
	I_{\UBA}(X;Y|g_{\theta}) \triangleq \EE_{p(x,y)}[g_{\theta}(x,y) - \log Z_{\theta}(x)] = \EE_{p(x)}\left[ \EE_{p(y|x)}\left[  \log \frac{e^{g_{\theta}(x,y)}}{Z_{\theta}(x)}\right]\right].
	%\end{array}
	\eeq
	While this $\UBA$ bound remains intractable, now with $Z_{\theta}(x)$ instead of $p(y)$ we can apply different techniques for empirical estimates of $Z_{\theta}(x)$ to render a tractable surrogate target. This has led to various popular MI bounds listed in Table \ref{tab:vb} (see Appendix \ref{sec:infonce_appendix} for derivations). 
	%Please refer to the Appendix for $\UBA$'s connections to other popular MI bounds listed in Table \ref{tab:vb}. 
	%This $\UBA$ bound is still intractable, but now with $Z_{\theta}(x)$ instead of $p(y)$ we can apply techniques introduced below to render a tractable surrogate. Please refer to the Appendix for $\UBA$'s connections to other popular MI bounds listed in Table \ref{tab:vb}. 
	%For its connections to other popular MI bounds listed in Table \ref{tab:vb}, further details are provided in the Appendix. 
	
	{\bf InfoNCE and noise contrastive estimation.} \texttt{InfoNCE} is a multi-sample mutual information estimator proposed in \citep{oord2018representation}, built on the idea of {\it noise contrastive estimation} (NCE) \citep{gutmann2010noise}. 
	NCE learns statistical properties of a target distribution by comparing the {\it positive} samples from the target distribution to the ``{\it negative}'' samples from a carefully crafted noise distribution, and this technique is also known as {\it negative sampling} in some contexts \citep{mnih2013learning, grover2016node2vec}. 
	%With a carefully crafted noise distribution, NCE learns statistical properties of a target distribution by comparing the {\it positive} samples from the target distribution to the ``{\it negative}'' samples from the noise distribution, and thus is known as {\it negative sampling} in some contexts \citep{mnih2013learning, grover2016node2vec}. 
	The $\infonce$ estimator implements this contrastive estimation idea via using the na\"ive empirical estimate of $Z_{\theta}(x)$ in $\UBA$\footnote{This estimator is technically equivalent to the original definition due to the symmetry of $K$ samples.}, {\it i.e.}
	%as follows  
	\beq
	\label{eq:infonce}
	%I_{\infonce}^K(X;Y|f)\triangleq \EE_{p^K(x,y)}\left[\log \frac{\exp(g_{\theta}(x_1,y_1))}{\frac{1}{K}\sum_{k'} \exp(g_{\theta}(x_{1}, y_{k'}))}\right], \,\, I_{\infonce}^K(X;Y) \triangleq \max_{f\in \CF} \{I_{\infonce}^K(X;Y|f)\}, 
	I_{\infonce}^K(X;Y|g_{\theta})\triangleq \EE_{p^K(x,y)}\left[\log \frac{e^{g_{\theta}(x_1,y_1)}}{\frac{1}{K}\sum_{j} e^{g_{\theta}(x_{1}, y_{j})}}\right], I_{\infonce}^K(X;Y) \triangleq \max_{g_{\theta}\in \CF} \{I_{\infonce}^K(X;Y|g_{\theta})\}, 
	\eeq
	%\beq
	%\label{eq:infonce}
	%\hspace{-.5em}\begin{array}{c}
		%I_{\infonce}^K(X;Y|g)\triangleq \EE_{p^K(x,y)}\left[\log \frac{e^{g(x_1,y_1)}}{\frac{1}{K}\sum_{k} e^{g(x_{1}, y_{k})}}\right], \\
		%[5pt]
		%I_{\infonce}^K(X;Y) \triangleq \max_{g\in \CF} \{I_{\infonce}^K(X;Y|g)\}, 
		%\end{array}
		%\eeq
		where $g_{\theta}$ is known as the {\it critic} in the nomenclature of contrastive learning,  and we have used $p^K(x,y)$ to denote $K$ independent draws from the joint density $p(x,y)$, and $\{(x_k, y_k)\}_{k=1}^K$ for each pair of samples. Here the positive and negative samples are respectively drawn from the joint $p(x,y)$ and product of marginals $p(x)p(y)$. Intuitively, $\infonce$ tries to accurately classify the positive samples when they are mixed with negative samples, and the Proposition below formally characterizes $\infonce$'s statistical properties as a MI estimator. 
		% \vspace{-3pt}
		\begin{prop}[\citep{poole2019variational}]
			\label{thm:infonce}
			\texttt{InfoNCE} is an asymptotically tight mutual information lower bound, {\it i.e.} $I_{\infonce}^K(X;Y|g_{\theta}) \leq I(X;Y)$, $\lim_{K\rightarrow\infty} I_{\infonce}^K(X;Y) \rightarrow I(X;Y).$
			% \beq
			% \begin{array}{c}
				% I(X;Y) \geq I_{\infonce}^K(X;Y|f), \\
				% [5pt]
				%  \lim_{K\rightarrow\infty} I_{\infonce}^K(X;Y) \rightarrow I(X;Y). 
				%  \end{array}
			% \eeq
		\end{prop}

		\begin{figure}
			\begin{minipage}{.63\textwidth}
				{\bf Fenchel-Legendre duality.} Our key idea is to exploit the convex duality for MI estimation. Let $f(t)$ be a proper convex, lower-semicontinuous function; then its {convex conjugate} function is defined as 
				$f^*(v) \triangleq \sup_{t\in \CD(f)}\{ t v - f(t) \}$, where $\CD(f)$ is the domain of function $f$ \citep{hiriart2012fundamentals}. We call $f^*(v)$ the \textit{Fenchel conjugate} of $f(t)$, which is also known as the {\it Legendre transform} in physics. 
				%, where functions of one quantity ({\it e.g.}, position, pressure, temperature) are converted into functions of its conjugate quantity ({\it e.g.}, momentum, volume, entropy, respectively).
				The Fenchel conjugate pair $(f, f^*)$ are dual to each other, in the sense that $f^{**} = f$, \textit{i.e.}, 
				$
				f(t) = \sup_{v\in \CD(f^*)}\{ v t - f^*(v) \}. 
				$
				%As a concrete example, $(-\log(t), -1-\log(-v))$ gives such a pair, which we will exploit in the next section. 
				For $f(t) = -\log(t)$ and its Fenchel conjugate $f^*(v) = -1-\log(-v)$, we have inequality
				\beq
				\label{eq:logt}
				% -\log(t) = \max_u\{ -u - e^{-u} t+1 \} = - \min_u\{u+e^{-u} t\}+1.
				\begin{array}{c}
					%-\log(t) = \max_{u\in \BR}\{ -u - e^{-u} t+1 \} 
					-\log(t) \geq -u - e^{-u} t+1, \quad \text{for $u\in\BR$}
					%= - \min_{u \in \BR}\{u+e^{-u} t\}+1.
				\end{array}
				\eeq
				with the equality holds when $u=\log(t)$.
			\end{minipage}
			\hspace{3pt}
			\begin{minipage}{.35\textwidth}
				\vspace{-.5em}
				\begin{center}{
						\includegraphics[width=1.0\textwidth]{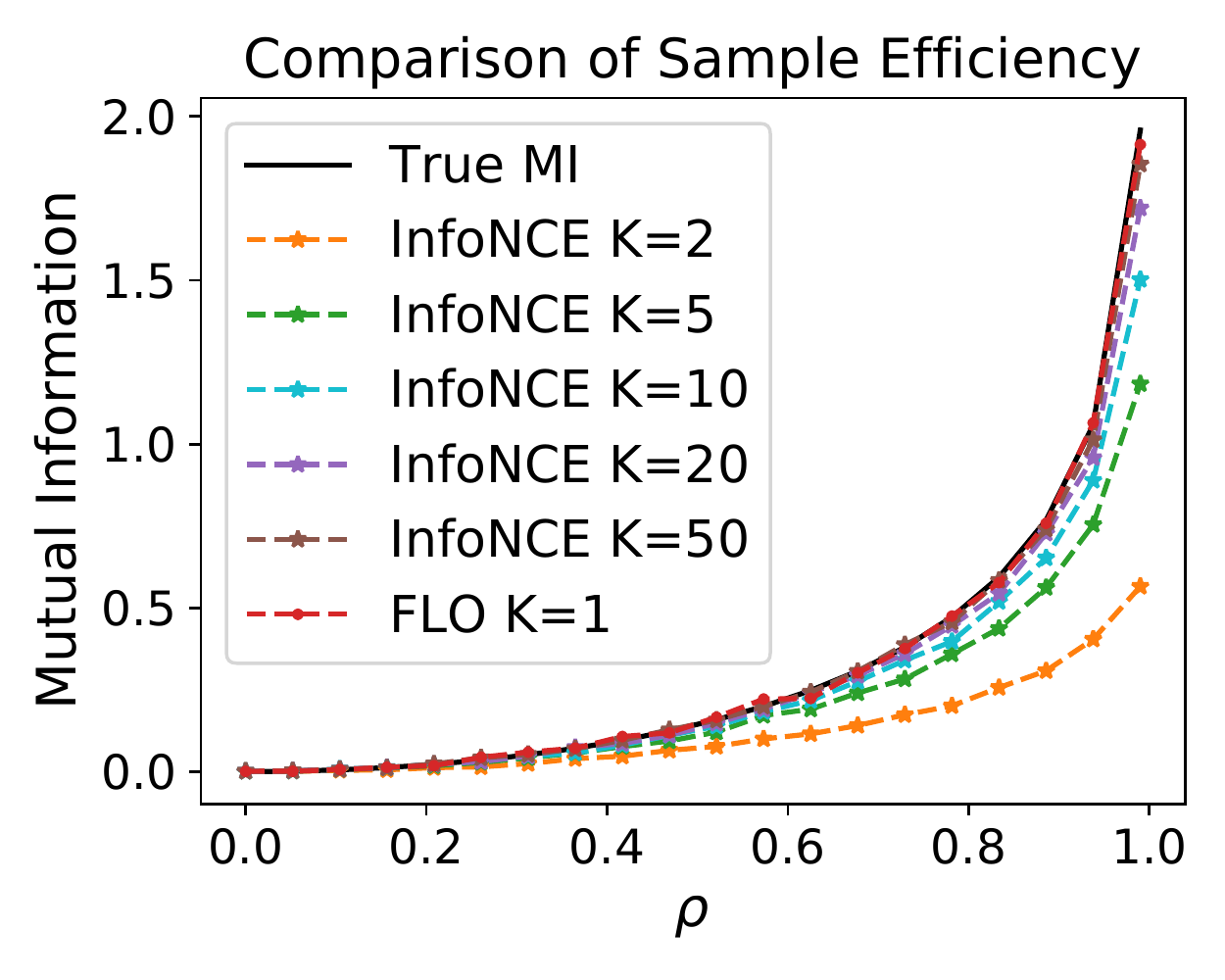}
					}
					\vspace{-2em}
					\caption{$K$-sample $\infonce$ and single-sample $\FLO$. Note $\FLO$ is tight regardless of sample-size. \label{fig:cmp_1d}}
				\end{center}
			\end{minipage}
			\vspace{-1.5em}
		\end{figure}

		%\begin{wrapfigure}[13]{R}{0.38\textwidth}
		%\vspace{-2.em}
		%\scalebox{.93}{
			%%\hspace{-.5em}
			%\begin{minipage}{.39\textwidth}
			%\begin{figure}[H]
			%\begin{center}{
			%\hspace{-10pt}\includegraphics[width=1.05\textwidth]{figures/toy/1dcompare.pdf}
			%}
		%\vspace{-1.5em}
		%\caption{Comparison of the $K$-sample $\infonce$ and single-sample $\FLO$. $\FLO$ is tight regardless of sample-size. \label{fig:cmp_1d}}
		%\end{center}
		%\vspace{-1.5em}
		%\end{figure}
		%\end{minipage}
		%}
	%\end{wrapfigure}

	\vspace{-8pt}
	\subsection{Fenchel-Legendre Optimization for tight mutual information estimation}
	\vspace{-5pt}
	\label{sec:flo}

		With the above mathematical tools, we are ready to present the main result of this paper: a tight, data-efficient variational MI lower bound that can be efficiently implemented. 
		%a tight variational MI lower bound that adopts unbiased mini-batch estimation. 
		%\begin{itemize}
		%\item A novel mutual information lower bound named $\FLO$;
		%\item 
		%\end{itemize}
		%: ($i$) a novel tight mutual information lower bound $I_{\FLO}$; ($ii$) an unbiased. 
		
		%This section presents the main result of this paper. To make the derivation more intuitive, and expose a direct connection to the multi-sample estimators, we will show how to get our new estimator from $\infonce$. See the Appendix for alternative derivations. 
		
		% \begin{wrapfigure}[8]{R}{0.4\textwidth}
			% \vspace{12.5em}
			% % \scalebox{1.}{
				% % %\hspace{-.5em}
				% % \begin{minipage}{.4\textwidth}
					% % \begin{figure}[H]
						% \vspace{-2em}
						% \begin{center}
							% \includegraphics[width=.4\textwidth]{figures/toy/1dcompare.pdf}
							% \end{center}
						% \vspace{-1.5em}
						% \caption{Comparison of $\infonce$ and $\FLO$ bounds with $1$-D Gaussian. \label{fig:cmp_1d}}
						% \vspace{-.5em}
						% % \end{figure}
					% % \end{minipage}
				% % }
			% \end{wrapfigure}
		
		{\bf Lower bounding MI with Fenchel-Legendre Optimization.} Our key insight is that MI estimation is essentially an unnormalized statistical model, which can be efficiently handled by the Fenchel-Legendre transform technique. Take the integrand from $\UBA$ in (\ref{eq:uba}) and we can rewrite it as
		\beq
		\log \frac{\exp(g_{\theta}(x,y))}{Z_{\theta}(x)} = - \log \left\{ \EE_{p(y')}[\exp(g(x,y') - g(x,y))] \right\}, 
		%\log \frac{\exp(g_{\theta}(x,y))}{Z_{\theta}(x)} = - \log \left\{ \int \exp(g_{\theta}(x,y') - g_{\theta}(x,y)) p(y') \ud y' \right\}, 
		%\log \frac{\exp(g(x,y))}{Z(x)} = -\log \frac{\exp(g(x,y))}{\int }
		%\int \frac{1}{e^{g(x,y)}} e^{g(x,y')) p(y') \ud y' \\
			\label{eq:log_integrand}
			\eeq
			%where $Y'$ is an independent copy of $Y$. 
			where $p(y')$ is the same probability density as $p(y)$ ({\it i.e.}, $Y'$ is an independent copy of $Y$). Now let us use the Fenchel inequality of $-\log(t)$ from (\ref{eq:logt}), plugging it into the above equation and then we have
			\beq
			%\log \frac{\exp(g_{\theta}(x,y))}{Z_{\theta}(x)} = -\min_{u \in \BR}\left\{u+e^{-u}\int \exp(g(x,y') - g(x,y)) p(y') \ud y'\right\}+1.
			\log \frac{\exp(g_{\theta}(x,y))}{Z_{\theta}(x)} \geq \left\{-u-e^{-u}\EE_{p(y')}[\exp(g(x,y') - g(x,y))]\right\}+1.
			\eeq
			for all $u \in \BR$. This implies for any function $u_{\phi}(x,y): \CX \times \CY \rightarrow \BR$, the following inequality holds 
			%for all $(x, y) \in \CX \times \CY$
			\beq
			%\log \frac{\exp(g_{\theta}(x,y))}{Z_{\theta}(x)} \geq u(x,y)+e^{-u(x,y)}\int \exp(g(x,y') - g(x,y)) p(y') \ud y'+1.
			\log \frac{\exp(g_{\theta}(x,y))}{Z_{\theta}(x)} \geq - \{ u_{\phi}(x,y)+e^{-u_{\phi}(x,y)} \EE_{p(y')}[\exp(g(x,y') - g(x,y))] \}+1.
			\label{eq:flo_local}
			\eeq
			By putting (\ref{eq:flo_local}) back to (\ref{eq:uba}), we obtain our new Fenchel-Legendre Optimization ($\FLO$) MI lower bound
			\beq
			I_{\FLO}(X;Y|g_{\theta}, u_{\phi}) \triangleq \EE_{p(x,y)}\left[ - \{ u_{\phi}(x,y)+e^{-u_{\phi}(x,y)} \EE_{p(y')}[e^{g_{\theta}(x,y') - g_{\theta}(x,y)}] \} \right] + 1,
			\label{eq:flo_defn}
			\eeq
			%\beqs
			%\log \frac{\exp(g(x,y))}{Z(x)} & = & \log \frac{\exp(g(x,y))}{\int \exp(g(x,y')) \ud y'} \\
			%& = & \log \frac{1}{\frac{1}{\exp(x,y)}\int \exp(g(x,y'))\ud y'} \\
			%& = & -\log \int \frac{\exp(g(x,y'))}{\exp(g(x,y))} \ud y' \\
			%& = & -\log \int \exp(g(x,y') - g(x,y)) \ud y'
			%\eeqs
			and concludes the proof for the following Proposition. 
			\begin{prop}
				\label{thm:flo}
				$I_{\FLO}(X;Y|g_{\theta}, u_{\phi}) \leq I_{\UBA}(X;Y|g_{\theta}) \leq I(X;Y)$.
				%I(X;Y) \geq I_{\UBA}(X;Y|g_{\theta}) \geq 
			\end{prop}

			%{\bf Unbiased empirical estimation.} Now we want to establish that for mini-batch esti. 
			In practice, $\FLO$ can be estimated with the following na\"ive empirical $K$-sample estimator
			\beq
			\hat{I}_{\FLO}^K(X;Y|g_{\theta}, u_{\phi}) \triangleq - \left\{ u_{\phi}(x_i, y_i)+e^{-u_{\phi}(x_i, y_i)}  \frac{1}{K-1}\sum_{j\neq i} e^{g_{\theta}(x_i, y_j) - g_{\theta}(x_i, y_i)} \right\} + 1.
			\eeq
			Since the summation in $\hat{I}_{\FLO}^K$ is not encapsulated by a convex $\log$ transformation, $I_{\FLO}^K \triangleq \EE_{p^K}[\hat{I}_{\FLO}^K]$ is an unbiased estimator for $I_{\FLO}(X;Y|g_{\theta}, u_{\phi})$ independent of the batch size $K$ (see Figure \ref{fig:cmp_1d}).

			%\begin{wrapfigure}[17]{R}{0.38\textwidth}
			%\vspace{-2.em}
			%\scalebox{.93}{
				%%\hspace{-.5em}
				%\begin{minipage}{.39\textwidth}
				%\begin{figure}[H]
				%\begin{center}{
				%\hspace{-10pt}\includegraphics[width=1.05\textwidth]{figures/cartoon/fancy_cartoon.pdf}
				%}
			%\vspace{-1.5em}
			%\caption{}
			%\end{center}
			%\vspace{-1.5em}
			%\end{figure}
			%\end{minipage}
			%}
		%\end{wrapfigure}
		
		{\bf Why is the $\FLO$ bound more appealing?} At first sight, it may appear counter-intuitive that $I_{\FLO}$ is a better MI bound compared to prior arts such as $\NWJ$ or $\infonce$: it seems to be more complicated as an extra variational function $u_{\phi}(x,y)$ has been introduced. To answer this question, we next explain the statistical meaning of the newly introduced $u_{\phi}(x,y)$, and establish some important statistical properties of $\FLO$ that makes it more favorable: that $I_{\FLO}$ is tight, meaning the ground-truth MI can be recovered for some specific choice of $g_{\theta}(x, y)$ and $u_{\phi}(x, y)$; and that $I_{\FLO}^K$ for any batch size $K$ is effectively optimizing $\infonce$ with an infinite batch size. And in Sec \ref{sec:flo_convergence}, we further justify $\FLO$'s advantages from optimization perspectives. 
		
		Given the close connection between $\FLO$ and $\UBA$, we first recall $\UBA$'s optimal critic that gives the tight MI estimate is $g^*(x,y) = \log p(x|y) + c(x)$, where this $c(x)$ can be any function of $x$ \citep{ma2018noise}. This $g^*(x,y)$ is not directly meaningful in a statistical sense, however, by integrating out $y'$, we have
		\beq
		\EE_{p(y')}\left[e^{g^*(x,y')-g^*(x,y)}\right] = \EE_{p(y')}\left[ \frac{p(x|y')}{p(x|y)} \right] = \frac{p(x)}{p(x|y)} = \frac{p(x)p(y)}{p(x,y)}, 
		%\EE_{p(y')}\left[e^{c_{g^*}(x,y,y')}\right] = \frac{p(x)}{p(x|y)} = \frac{p(x)p(y)}{p(x,y)}, 
		\eeq
		which is the likelihood ratio between the marginals and joint. On the other hand, based on the Fenchel-Legendre inequality (\ref{eq:logt}), we know for fixed $g(x,y)$ our $\FLO$ bound in (\ref{eq:flo_defn}) can be maximized with $u_g(x,y) = \log \EE_{p(y')}\left[e^{g(x,y')-g(x,y)}\right]$. Putting these all together we have $u_{g^*}(x,y) = - \log \frac{p(x,y)}{p(x)p(y)}. $
		%\beq
		%\label{eq:upmi}
		%u_{g^*}(x,y) = - \log \frac{p(x,y)}{p(x)p(y)}. 
		%\eeq
		%This proves $u(x,y)$ essentially learns the negative {\it point-wise mutual information} ($\PMI$). 
		This shows the $u_{\phi}(x,y)$ introduced in $\FLO$ actually tries to recover the negative $\PMI$. Comparing to the competing MI bounds that only optimizes for $g_{\theta}$, eliminating the drift term $c(x)$ reveals $\FLO$ enjoys the appealing {\it self-normalizing} property \citep{gutmann2010noise} that helps stabilize training. Plugging $(g^*, u_{g^*})$ into (\ref{eq:flo_defn}), we readily see $I_{\FLO}(X;Y|u_{g^*}, g^*) = I(X;Y)$, proving $\FLO$ is a tight MI bound.
		%This is significant because the $g(x,y)$ used by competing MI bounds are not directly meaningful due to the arbitrary drift term $c(x)$. 
		%Eliminating $c(x)$ reveals $\FLO$ enjoys the appealing {\it self-normalizing} property \citep{gutmann2010noise}.
		
		\begin{prop}
			\label{thm:fince}
			The $\FLO$ estimator is tight, the eqaulity holds when $g(x,y) = \log p(x|y) + c(x)$ for arbitrary function $c(x)$ and $u(x,y)=-\log \frac{p(x,y)}{p(x)p(y)}$.
			%\beq
			%\label{eq:fnce}
			%\begin{array}{l}
			%I(X;Y) = -\min_{u,g} \mathlarger\{\EE_{p(x,y)p(y')}[u(x,y) +e^{-u(x,y)+c(x,y,y';g)}]\mathlarger\}+1
			%%I(X;Y) = -\min_{u,g} \mathlarger\{\EE_{p(x,y)p(y')}[u(X,Y) \\
			%%[5pt]
			%%+\exp(-u(X,Y)+g(X,Y')-g(X,Y))]\mathlarger\}+1
			%\end{array}
			%\eeq
		\end{prop}
		
		\begin{col}
			\label{thm:fince_mi}
			Let $(g^*, u_{g^*})$ be the maximizers for (\ref{eq:flo_defn}), then $I(X;Y) = \EE_{p(x,y)}[-u_{g^*}(x,y)]$.
			%Let $u^*(x,y)$ be the solution for (\ref{eq:fnce}), then
			%\beq
			%I(X,Y) = \EE_{p(x,y)}[-u^*(X,Y)]. 
			%\eeq
		\end{col}
		%, an appealing {\it self-normalizing} property \citep{gutmann2010noise}, while the $g(x,y)$ used by competing contrastive MI bounds are not directly meaningful (because of the arbitrary drift term $c(x)$). On another note, we can further simplify $u(x,y)$ to recover the more compact, yet empirically sub-optimal $\TUBA$ bound defined in \cite{poole2019variational} (see Table \ref{tab:vb}, derivations in the Appendix). 
		
		%{\bf Statistical properties of $\FLO$.} 
		
		%Next we want to establish two important properties of $\FLO$: ($i$), this $\FLO$ bound is tight, meaning the equality from Proposition \ref{thm:flo} is attainable for some specific choice of $g_{\theta}(x, y)$ and $u_{\phi}(x, y)$; and ($ii$), its empirical mini-batch estimator is unbiased. 

		Finally, we give a simple asymptotic argument showing $\FLO$ essentially optimizes $\infonce$ with an infinite batch size. In virtue of the law of large numbers, we have the denominator in $\infonce$ converging to $\lim_{K\rightarrow \infty} \frac{1}{K} \sum_{j=1}^K e^{g_{\theta}(x_i, y_j)} \rightarrow \EE_{p(y')}[e^{g_{\theta}(x_i, y')}] = Z_{\theta}(x_i)$, and consequently it recovers the $\UBA$ bound. Since $\FLO$ is derived from $\UBA$, we can view $\FLO$ as using the optimization of $u_{\phi}(x, y)$ to amortize the difficulty of evaluating infinite number of $e^{g_{\theta}(x_i, y_j)}$ with $\infonce$.

		{\bf Efficient implementations of $\FLO$.} A lingering concern is that the newly introduced $u_{\phi}(x,y)$ can incur extra computation overhead. This is not true, as we can maximally encourage parameter sharing by jointly model $u_{\phi}(x,y)$ and $g_{\theta}(x,y)$ with a single neural network $f_{\Psi}(x,y):\CX \times \CY \rightarrow \BR^2$ with two output heads, {\it i.e.}, $[u_i, g_i] = f_{\Psi}(x_i, y_i)$. Consequently, while $\FLO$ adopts a dual critics design, it does not actually invoke extra modeling cost compared to its single-critic counterparts ({\it e.g.}, $\infonce$). Experiments show this shared parameterization in fact promotes synergies and speeds up learning (see our ablation studies in Appendix). 
		% Now we describe how to efficiently implement the proposed $\FLO$ objective. We model the critic $g(x,y)$ and $\PMI$ $u(x,y)$ with neural networks $u_{\phi}(x, y)$ and $g_{\theta}(x,y)$, respectively, and 
		%To enable efficient knowledge transfer between similar samples, we amortize the learning of the critic $g(x,y)$ and $\PMI$ $u(x,y)$ through neural networks $u_{\phi}(x, y)$ and $g_{\theta}(x,y)$.  
		%While they can be also parameterized point-wise with $\{u_i\}$ and $\{g_i\}$ for each sample $\{(x_i, y_i)\}_{i=1}^n$, amortized inference allows efficient knowledge transfer between similar samples. 
		%We note that efficient parameter sharing can be achieved by modeling $u_{\phi}(x,y)$ and $g_{\theta}(x,y)$ jointly with a single neural network $f_{\Psi}(x,y):\CX \times \CY \rightarrow \BR^2$ with two output heads, {\it i.e.}, $[u_i, g_i] = f_{\Psi}(x_i, y_i)$. Consequently, although $\FLO$ has an additional target $u_{\phi}$ relative to its single-critic counterparts, such as $\infonce$, it does not actually induce extra modeling cost. In fact, this shared parameterization promotes synergy between the two learners for better sample efficiency. 
		%, that $u_{\phi}(x,y)$ and $g_{\theta}(x,y)$ do not necessarily need to be modeled separately. Instead, they can be
		
		To further enhance the computation efficiency, we consider a massively parallelized {\it bi-linear} critic design that uses all in-batch samples as negatives. Let $g_\theta(x,y) = \tau \cdot \langle h_{\theta}(x), \tilde{h}(y) \rangle$, where $h: \CX \rightarrow \BS^p$ and $\tilde{h}: \CY \rightarrow \BS^p$ are respectively encoders that map data to unit sphere $\BS^p$ embedded in $\BR^{p+1}$, $\langle a, b \rangle = a^T b$ is the inner product operation, and $\tau>0$ is the inverse temperature parameter. Thus the evaluation of the {\it Gram} matrix $G = \tau \cdot h(\BX)^T \tilde{h}(\BY)$, where $[\BX, \BY] \in \BR^{K \times (d_x+d_y)}$ is a mini-batch of $K$-paired samples and $g_{\theta}(x_i, y_j) = G_{ij}$, can be parallelized via matrix multiplication. In this setup, the diagonal terms of $G$ are the positive scores while the off-diagonal terms negative scores. A similar strategy has been widely employed in the contrastive representation learning literature ({\it e.g.}, \cite{chen2020simple})\footnote{As an important note to the community, most open source implementations for the bilinear contrastive loss have mechanically implemented $\frac{1}{T} \langle \cdot, \cdot \rangle$ following the practice from pioneering contrastive learning studies, which is numerically unstable compared to our parameterization $\tau \langle \cdot, \cdot \rangle$ proposed here.}. We can simply model the PMI critic as $u(x,y) = \texttt{MLP}(h(x), \tilde{h}(y))$, whose computation cost is almost neglectable in practice, where feature encoders $h,\tilde{h}$ dominate computing.
		%A similar approach has been widely employed in the contrastive representation learning literature ({\it e.g.}, $\SimCLR$ \citep{chen2020simple}), where the major conceptual difference is that these works consider $h, \tilde{h}$ directly as the feature encoders, and consequently encouraging a flattened representations by design. 
		% For applications where more compact summaries of representation are desired, one will need to decouple the encoder and the contrast critic(s). 

		\begin{table*}[t!]
			\begin{center}
				\caption{Comparison of popular variational MI estimators. \label{tab:vb} Here $g(x,y), u(x,y)$ and $u(x)$ are variational functions to be optimized, $\sigma(u) = \frac{1}{1+\exp(-u)}$ is the Sigmoid function, $\CE[f(u), \eta]$ denotes exponential average of function $f(u)$ with decay parameter $\eta \in (0,1)$,  and $\alpha \in [0,1]$ is the balancing parameter used by $\ANCE$ trading off bias and variance between $\infonce$ and $\TUBA$. 
					%To highlight the contrastive view, we use $(x, p_{\oplus})$ to denote samples drawn from the joint density $p(x,y)$, and $(x, y_{\ominus})$ to denote samples drawn from the product of marginal $p(x)p(y)$. 
					we use $(x_i, y_i)$ to denote positive samples from the joint density $p(x,y)$, and $(x_i, y_j)$ or $(x_{k}', y_{k}')$ to denote negative samples drawn from the product of marginal $p(x)p(y)$.
					In context, $y_\oplus$ and $y_\ominus$ have the intuitive interpretation of positive and negative samples. We exclude variational upper bounds here because their computations typically involve the explicit knowledge of conditional likelihoods. }
				% \vspace{5pt}
				% See Appendix for more details.
				\resizebox{\textwidth}{!}{%
					\begin{tabular}{cccccccc}
						\toprule
						%Name & Objective & Tractability & Stages & Sample & Tightness & Variance &  \\
						Name & Objective & Bias & Var. & Converge \\
						\midrule
						\multicolumn{5}{c}{$(x_i, y_i) \stackrel{iid}{\sim} p(x, y), \,\, (x_k', y_k') \stackrel{iid}{\sim} p(x)p(y), \,\, m_{\alpha,u}(x, y_{1:K}) \triangleq \alpha \frac{1}{K}\left\{\sum_{k=1}^{K} \exp(g(x, y_k))\right\} + (1-\alpha) \exp(u(x))$} \\
						[4pt]
						%\texttt{BA} \citep{barber2003information} & \\
						%\texttt{DV} \citep{donsker1983asymptotic} & $g(x, y_{\oplus}) - g(x, y_{\ominus})$ \\
						\texttt{DV} \citep{donsker1983asymptotic} & $g(x_i, y_i) - \log (\sum_{k=1}^{K}\exp(g(x_k', y_k'))/K)$ & high & high & no \\
						[4pt]
						\texttt{MINE} \citep{belghazi2018mutual} & $g(x_i, y_i) - \log(\CE[\exp(g(x_i, y_j)), \eta])$ & low & high & no\\
						[4pt]
						\texttt{NWJ} \citep{nguyen2010estimating} & $g(x_i, y_i) - \exp(g(x_i, y_j)-1)$ & low & high & no \\
						[4pt]
						\texttt{JSD} \citep{hjelm2019learning} & $g^*(x_i, y_i) - \exp(g^*(x_i, y_j)-1)$ & low & high & no \\
						[4pt]
						& $g^* \xleftarrow{\argmax} \{ \log \sigma(g(x_i, y_i)) + \log \sigma(-g(x_i ,y_j)) \}$  \\
						[4pt]
						\texttt{TUBA} \citep{poole2019variational} & $g(x_i, y_i)+u(x_i)+1 - \exp(g(x_i ,y_j)-u(x_i))$ & low & high & no \\
						[4pt]
						%\multicolumn{4}{c}{$m_{\alpha,u}(x, y^{1:K}) \triangleq \alpha \frac{1}{K}\left\{\sum_{k=1}^{K} \exp(g(x, y^k))\right\} + (1-\alpha) \exp(u(x)), m(\cdot) \triangleq m_{0,u}(\cdot)$} \\
						%[4pt]
						\texttt{InfoNCE} \citep{oord2018representation} & $g(x_i,y_i) - \log(\sum_j \exp(g(x_i, y_j))/K)$ & high & low & no\\
						[4pt]
						%$\ANCE$ \citep{poole2019variational} & $g(x,y_{\oplus}) - g(x,y_{\ominus}) - \log (m_{\alpha, u}(x, \{ y_{\oplus}, y_{\ominus}^{1:K-1}))$ & flex & flex\\
						%\multicolumn{2}{c}{} & 
						$\ANCE$ \citep{poole2019variational} & \multicolumn{3}{c}{$g(x_i, y_i) - g(x_i, y_j) - \log (m_{\alpha, u}(x,  y_{1:K}))+\log (m_{\alpha, u}(x_k', y_k'))$} & no \\
						[2pt]
						\multicolumn{4}{c}{$\ANCE$ interpolates between low-bias high-var ($\alpha\rightarrow 1$, $\NWJ$) to high-bias low-var ($\alpha\rightarrow 0$, $\infonce$)} \\
						%& flex & flex\\
						%[4pt]
						%& \hfill \hspace{3.3em} \\
						%[4pt]
						%& \hfill $+\log (m_{\alpha, u}(x, y_{\ominus}^{1:K}))$\hspace{3.3em} \\
						%\texttt{}
						[5pt]
						\texttt{FLO} (ours) & $-u(x_i ,y_i) - \exp(-u(x_i, y_i)+g(x_i, y_j)-g(x_i, y_i))$ & {\bf \textcolor{teal}{low}} & {\bf \textcolor{teal}{moderate}} & {\bf \textcolor{teal}{yes}}\\
						\bottomrule
					\end{tabular}
				}
				\vspace{-1.5em}
			\end{center}
		\end{table*}

		\vspace{-5pt}
		\subsection{Connections to the existing MI bounds} 
		\vspace{-5pt}
		
		Due to space limitations, we elaborate the connections to the existing MI bounds here, and have relegated an extended related work discussion in a broader context to the Appendix. 
		
		{\bf From $\log$-partition approximation to MI bounds.} To embrace a more holistic understanding, we list popular variational MI bounds together with our $\FLO$ in Table \ref{tab:vb}, and visualize their connections in Figure \ref{fig:schematic}. With the exception of $\JSD$, these bounds can be viewed from the perspective of unnormalized statistical modeling, as they differ in how the $\log$ partition function $\log Z(x)$ is estimated. We broadly categorize these estimators into two families: the $\log$-family ($\DV$, $\MINE$, $\infonce$) and the exponential-family ($\NWJ$, $\TUBA$, $\FLO$). In the $\log$-family, $\DV$ and $\infonce$ are multi-sample estimators that leverage direct Monte-Carlo estimates $\hat{Z}$ for $\log Z(x)$, and these two differ in whether to include the positive sample in the denominator or not. To avoid the excessive in-batch computation of the normalizer and the associated memory drain, $\MINE$ further employed an {\it exponential moving average} (EMA) to aggregate the normalizer across batches. Note for the $\log$-family estimators, their variational gaps are partly caused by the $\log$-transformation on finite-sample average due to Jensen's inequality ({\it i.e.}, $\log Z = \log \EE[\hat{Z}] \geq \EE [\log \hat{Z}]$). In contrast, the objective of exponential-family estimators do not involve such $\log$-transformation, since they can all be derived from the Fenchel-Legendre inequality: $\NWJ$ directly applies the Fenchel dual of $f$-divergence for MI \citep{nowozin2016f}, while $\TUBA$ exploits this inequality to compute the $\log$ partition $\log Z(x) = \log \EE_{p(y')}[\exp(g(x,y'))]$. Motivated from a contrastive view, our $\FLO$ applies the Fenchel-Legendre inequality to the $\log$-partition of contrast scores. 
		
		{\bf A contrastive view for MI estimation.} The MI estimators can also be categorized based on how they contrast the samples. For instance,  $\NWJ$ and $\TUBA$ are generally considered to be non-contrastive estimators, as their objectives do not compare positive samples against negative samples on the same scale ({\it i.e.}, $\log$ versus $\exp$), and this might explain their lack of effectiveness in representation learning applications. For $\JSD$, it depends on a two-stage estimation procedure similar to that in adversarial training to assess the MI, by explicitly contrasting positive and negative samples to estimate the likelihood ratio. This strategy has been reported to be unstable in many empirical settings. The $\log$-family estimators can be considered as a multi-sample, single-stage generalization of $\JSD$. However, the $\DV$ objective can go unbounded thus resulting in a large variance, and the contrastive signal is decoupled by the EMA operation in $\MINE$. Designed from contrastive perspectives, $\infonce$ trades bound tightness for a lower estimation variance, which is found to be crucial in representation learning applications. Our $\FLO$ formalizes the contrastive view for exponential-family MI estimation, and  bridges existing bounds: the $\PMI$ normalizer $\exp(-u(x,y))$ is a more principled treatment than the EMA in $\MINE$, and compared to $\DV$ the positive and negative samples are explicitly contrasted and adaptively normalized.

		{\bf Important FLO variants.} We now demonstrate that $\FLO$ is a flexible framework that not only recovers existing bounds, but also derives novel bounds such as 
		\beq
		\begin{array}{l}
			I_{\FDV} \triangleq \texttt{StopGrad}[I_{\DV}(\{ (x_i, y_i) \})] + \frac{\sum_{j} \exp(c_\theta(x_i, y_i, y_j))}{\texttt{StopGrad}[\sum_{j} \exp(
				c_\theta(x_i, y_i, y_j))]} - 1.
		\end{array}
		\label{eq:fdv}
		\eeq
		% Recall the optimal $u^*(x,y)$ given $g_{\theta}(x,y)$ is in the form of $-g_{\theta}(x,y) + s_{\psi}(x)$, and parameterizing $u(x,y)$ in this way recovers the $\TUBA$ bound. 
		Recall the optimal $g^*(x,y) = \log p(x|y) + c(x)$ and $u^*(x,y) = -\log \frac{p(x,y)}{p(x)p(y)}$, which motivates us to parameterize $u(x,y)$ in the form of $-g_{\theta}(x,y) + s_{\psi}(x)$, where $s_{\psi}(x)$ models the arbitrary drift $c(x)$, and this recovers the $\TUBA$ bound. 
		Additionally, we note that ($i$) fixing either of $u$ and $g$, and optimizing the other also gives a valid lower bound to MI; and ($ii$) a carefully chosen multi-input $u(\{ (x_i, y_i) \})$ can be computationally appealing. As a concrete example, if we set $u_{\phi}$ to $\mathfrak{u}_{\theta}(\{ (x_i, y_i) \}) \leftarrow \log \left(\frac{1}{K}\sum\nolimits_{j} e^{c(x_i, y_i, y_j; g_{\theta})} \right)$
		% \beq
		% %\mathfrak{u}_{\theta}(\{ (x_i, y_i) \}) = \log \left(\frac{1}{K}\sum\nolimits_{j} \exp(g_{\theta}(x_i, y_j) - g_{\theta}(x_i, y_i)) \right)
		% \mathfrak{u}_{\theta}(\{ (x_i, y_i) \}) \leftarrow \log \left(\frac{1}{K}\sum\nolimits_{j} e^{c(x_i, y_i, y_j; g_{\theta})} \right)
		% \label{eq:ug}
		% %_{i=1}^K
		% \eeq 
		and update $u_{\theta}(x,y)$ while artificially keeping the critic $g_{\theta}(x,y)$ fixed \footnote{That is to say $g_{\theta}$ in $u_{\phi}$ is an independent copy of $g_{\theta}$.}, then $\FLO$ falls back to $\DV$. Alternatively, we can consider the Fenchel dual version of it: using the same multi-input $\mathfrak{u}_{\theta}(\{ (x_i, y_i) \})$ above, treat $u_{\phi}$ as fixed and only update $g_{\theta}$, and this gives us the novel MI objective in (\ref{eq:fdv}), we call it {\it Fenchel-Donsker-Varadhan} (FDV) estimator.
		%\vspace{-1em}
		%where we have used $\hat{g}, \hat{I}$ to denote evaluation-only mode for the corresponding functions (because they are the ``fixed'' $u_{\phi}$ and do not backpropagate parameter gradients). 

		\vspace{-6pt}
		\subsection{Gradient and convergence analysis of FLO}
		\vspace{-3pt}
		\label{sec:flo_convergence}
		
		In this section, we will establish that FLO better optimizes the MI because its gradient is more accurate than competing variational bounds such as $\NWJ$ and $\TUBA$; also, we provide the first convergence analysis for variational MI estimation by showing $\FLO$ converges under SGD.

		First, recall most tractable variational MI bounds are derived from and upper bounded by the intractable $\UBA$ bound \citep{poole2019variational}.  For instance, with the same critic $g_{\theta}$ we have $I_{\NWJ} \leq I_{\TUBA} \leq I_\UBA$. So if we can show $\nabla_{\theta} I_{\FLO} \approx \nabla_{\theta} I_\UBA$ then $\FLO$ is better optimized. To simplify notations, we denote $c_{\theta}(x,y,y')\triangleq g_{\theta}(x,y') - g_{\theta}(x, y)$ and $\CE_{\theta}(x,y) \triangleq 1/\EE_{p(y')}[e^{c_{\theta}(x,y,y')}]$, and we can easily verify
		\beq
		\EE_{p(y')}\left[\nabla_{\theta}\left\{ e^{c_{\theta}(x,y,y')} \right\}\right] = \nabla_{\theta} \left\{ \frac{1}{\CE_{\theta}(x,y)} \right\} = - \frac{\nabla \CE_{\theta}(x,y)}{(\CE_{\theta}(x,y))^2} = -\frac{\nabla_{\theta} \log \CE_{\theta}(x,y)}{\CE_{\theta}(x,y)} . 
		\eeq
		Since for fixed $g_{\theta}(x,y)$ the corresponding optimal $u_{\theta}^*(x,y)$ maximizing $I_{\FLO}(u_{\phi},g_{\theta}) \triangleq 1 - \left\{ u_{\phi}(x,y) + \EE_{p(y')}[ e^{-u_{\phi}(x,y) + c(x,y,y';g_{\theta})}] \right\}$  is given by $u_{\theta}^*(x, y) = \log \EE_{p(y')}[ e^{c_{\theta}(x,y,y')}] = -\log \CE_{\theta}(x,y)$ (using (\ref{eq:logt})), we see that the term $e^{-u_{\phi}(x,y)}$ is essentially optimized to approximate $\CE_{\theta}(x,y)$. To emphasize this point, we now write $\hat{\CE}_{\theta}(x,y) \triangleq e^{-u_{\phi}(x,y)}$. When this approximation is sufficiently accurate ({\it i.e.}, $\CE_{\theta} \approx \hat{\CE}_{\theta}$), we can see that $\nabla I_{\FLO}$ approximates $\nabla I_{\UBA}$ as follows
		\beq
		\label{eq:grad_approx}
		\hspace{-.4em}
		\begin{array}{l}
			\nabla_{\theta}\{ I_{\FLO}(u_{\phi},g_{\theta}) \} 
			= - \EE_{xy}\left[ e^{-u_{\phi}(x,y)}\EE_{y'}[ \nabla_{\theta} e^{c_{\theta}(x, y, y')}]\right]
			=  \EE_{xy}\left[ \frac{\hat{\CE}_{\theta}(x,y)}{\CE_{\theta}(x,y)} \nabla_{\theta} \log \CE_{\theta}(x,y)\right] \\
			[8pt]
			\hspace{7.5em} \approx \EE_{xy}\left[  \nabla_{\theta} \log \CE_{\theta}(x,y) \right] = \nabla_{\theta} \left\{ \EE_{p(x,y)}[\log \CE_{\theta}(x,y)] \right\} = \nabla_{\theta} \{I_{\UBA}(g_\theta)\}.
		\end{array}
		\eeq

		We can prove $\FLO$ will converge under much weaker conditions, even when this approximation $\hat{u}(x,y)$ is rough. The intuition is as follows: in (\ref{eq:grad_approx}), the term $\frac{\hat{\CE}_{\theta_t}}{\CE_{\theta_t}}$  only rescales the gradient, so the optimizer is still proceeding in the same direction as $\UBA$ in SGD. The informal version of our result is summarized in the Proposition below (see the Appendix for the formal version and proof).

		\begin{figure}[H]
			\vspace{-.5em}
			\begin{minipage}{0.59\textwidth}
				\begin{prop}[Convergence of $\FLO$, informal version]
					\label{thm:conv}
					Let $\{\eta_t\}_{t=1}^{\infty}$ be the stochastic {\it Robbins-Monro} sequence of learning rates: $\sum_t \EE[\tilde{\eta}_t] = \infty$ and $\sum_t \EE[\tilde{\eta}_t^2] < \infty$. If $\frac{\hat{\CE}_{\theta_t}}{\CE_{\theta_t}}$ is bounded between $[a,b]$ ($0<a<b<\infty$), then under the stochastic gradient descent scheme described in Algorithm \ref{alg:flo}, $\theta_t$ converges to a stationary point of $I_{\UBA}(g_{\theta})$ with probability $1$, {\it i.e.}, $\lim_{t\rightarrow\infty}\| \nabla I_{\UBA}(g_{\theta_t}) \|=0$. Additionally assume $I_{\UBA}$ is convex with respect to $\theta$, then $\FLO$ converges with probability $1$ to the global optimum $\theta^*$ of $I_{\UBA}$ from any initial point $\theta_0$.
				\end{prop}
			\end{minipage}
			\hspace{3pt}
			\scalebox{.9}{
				\begin{minipage}{.4\textwidth}
					\vspace{-1.5em}
					\begin{algorithm}[H]
						%   \caption{Fenchel-InfoNCE}
						\caption{$\FLO$}
						\label{alg:flo}
						\begin{algorithmic}
							%\small
							\STATE Empirical data $\hat{p}_d = \{ (x_i, y_i) \}_{i=1}^n$ \\
							[1pt]
							\STATE Model parameters $\Psi = (\theta, \phi)$ \\
							[1pt]
							\FOR{$t=1,2,\cdots$}
							\STATE Sample $i,j \stackrel{iid}{\sim} [n]$ \\
							[1pt]
							\STATE $u_{ii} = u_{\phi}(x_i, y_i), g_{ii} = g_{\theta}(x_i, y_i)$, \\
							\STATE $g_{ij} = g_{\theta}(x_i, y_{j})$\\
							[1pt]
							\STATE $\CF = u_{ii} + \exp(-u_{ii}+g_{ij}-g_{ii})$\\
							[1pt]
							$\Psi_{t} = \Psi_{t} - \eta_t \nabla_{\Psi} \CF$
							\ENDFOR\\
						\end{algorithmic}
					\end{algorithm}
				\end{minipage}
			}
			\vspace{-1.5em}
		\end{figure}

		Importantly, this marks the first convergence result for variational MI estimators. The convergence analyses for MI estimation is non-trivial and scarce even for those standard statistical estimators \citep{paninski2003estimation, gao2015efficient, rainforth2018nesting}. The lack of convergence guarantees has led to a proliferation of unstable MI-estimators used in practice (in particular, $\DV$, $\JSD$, and  $\MINE$) that critically rely on various empirical hacks to work well (see discussions in \cite{song2020understanding}). Our work establishes a family of variational MI estimators that provably converges, a contribution we consider significant as it fills an important gap in current literature on both theoretical and practical notes.

		\begin{figure*}[t!]
			\begin{center}
				\includegraphics[width=1.2\textwidth]{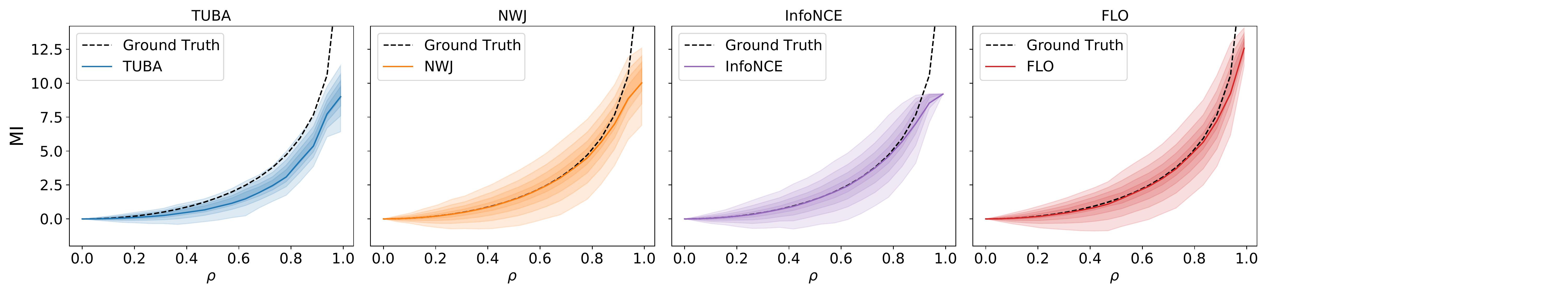}
			\end{center}
			\vspace{-1.3em}
			\caption{Bias-variance plot for popular variational MI bounds with the $10$-D Gaussians. Estimators that are more concentrated around the dashed line is considered better (low-bias, low-variance). In the more challenging high-MI regime, $\FLO$ shows a clear advantage over competing alternatives, where FLO pays less price in variance to achieve even better accuracy when tight estimation is impossible.
				%While in the low-MI regime all models except $\TUBA$ performed similarly, in the more challenging high-MI regime $\FLO$ shows a clear advantage over competing alternatives in both bias and variance. 
				\label{fig:cmp_var_10d}}
			% \vspace{-1.em}
		\end{figure*}

		\begin{figure*}[t!]
			\begin{center}
				\includegraphics[width=.335\textwidth,trim={0 0 .5in 0},clip]{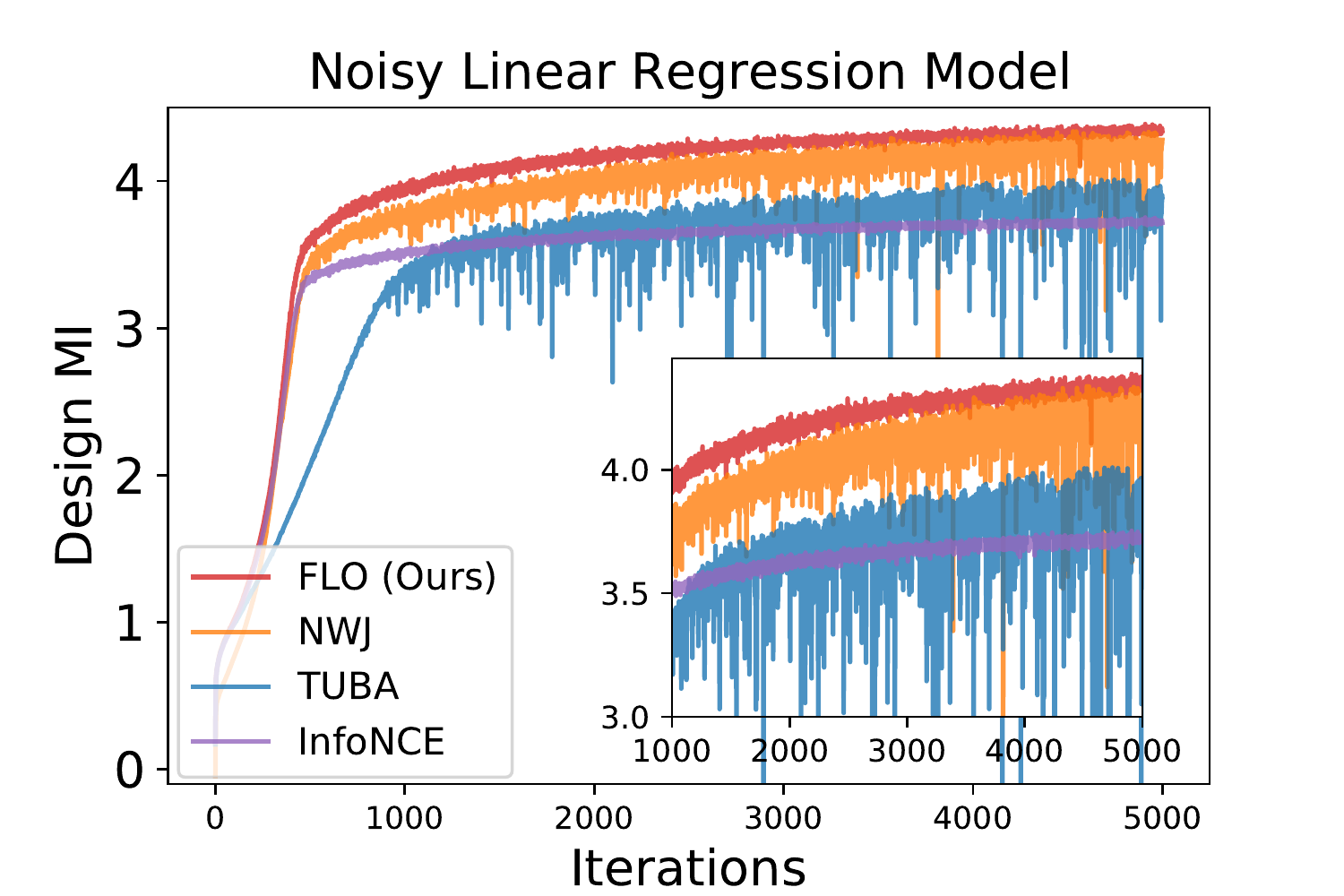}
				\includegraphics[width=.306\textwidth,trim={.5in 0 .5in 0},clip]{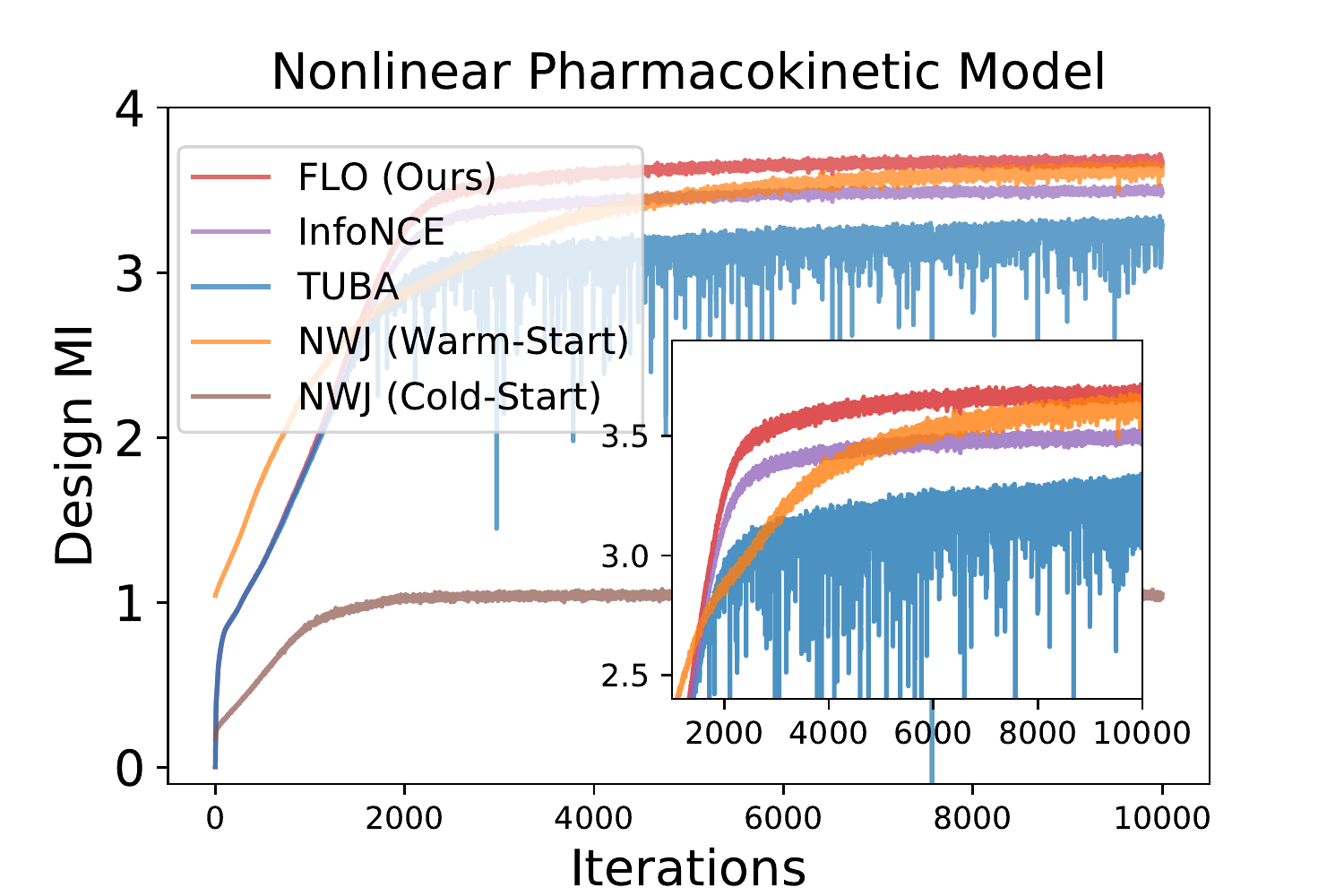}
				\includegraphics[width=.32\textwidth,trim={.3in 0 .5in 0},clip]{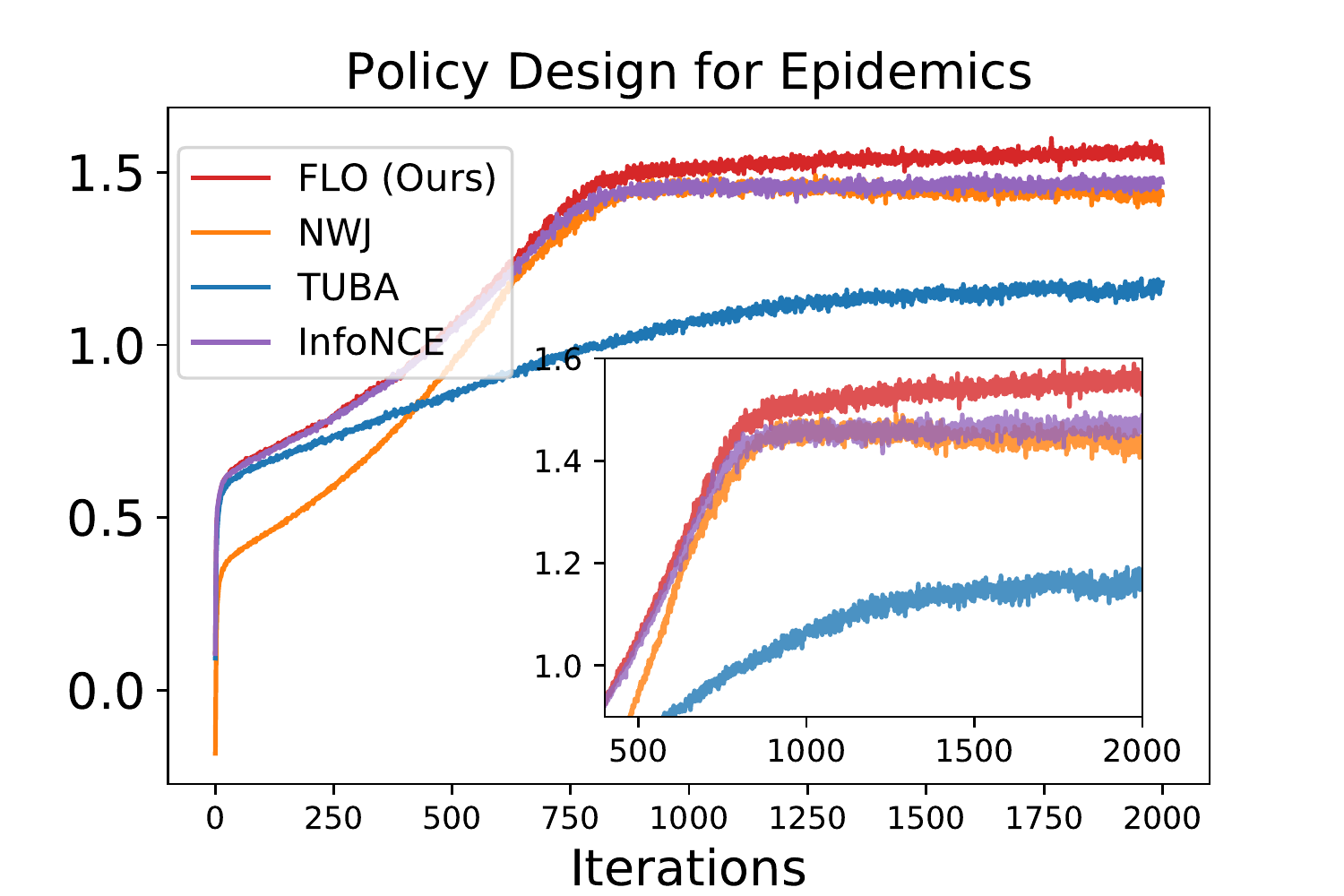}
			\end{center}
			\vspace{-1.em}
			\caption{Bayesian Optimal Experiment Design results. $\FLO$ consistently performs best, demonstrating superior strength in learning efficiency and robustness. $\NWJ$ takes the runner-up, but it has larger variance and is sensitive to network initializations. $\infonce$ is less competitive due to low sample inefficiency, but its smaller variance helps in the more challenging dynamic case. 
				\label{fig:design}}
			\vspace{-1.5em}
		\end{figure*}

		\vspace{-3pt}
		\section{Experiments}
		\vspace{-3pt}
		\label{sec:exp}
		
		We consider an extensive range of tasks to validate $\FLO$ and benchmark it against state-of-the-art solutions. To underscore the practical significance of MI in efficient machine learning, we demonstrate example applications from data collection (in statistical parlance, experimental design), self-supervised pre-training, to meta/transfer-learning. Limited by space, we present only the key results in the main text, and defer ablation studies and details of our experimental setups to the Appendix. Our code is available from \url{https://github.com/qingguo666/FLO}. All experiments are implemented with \texttt{PyTorch}.
		% synthetic and real-world tasks
		%  and executed on NVIDIA V100 GPUs

		%\begin{figure*}[t!]
		%\begin{center}
		%\includegraphics[width=.9\textwidth]{figures/toy/contour.pdf}
		%%\includegraphics[height=2in, width=.8\textwidth]{figures/blank}
		%\end{center}
		%\vspace{-1.5em}
		%\caption{Estimated $\hat{g}(x,y), -\hat{u}(x,y)$ and the true $\PMI$ $\log \frac{p(x,y)}{p(x) p(y)}$. \label{fig:pmi}}
		%\vspace{-1.em}
		%\end{figure*}
		
		%\begin{figure*}[t!]
		%\begin{center}
		%\includegraphics[height=2in]{figures/toy/multisample_bounds.pdf}
		%\end{center}
		%\vspace{-1.5em}
		%\caption{Bias variance plot for the popular MI bounds with the $10$-D Gaussians.  \label{fig:cmp_var_10d}}
		%\vspace{-1.5em}
		%% [This is a unfair comparison as MI estimators other than $\infonce$ is single-sample estimator. To be fixed in next iteration.]
		%\end{figure*}

		% ($a$) 
		% ($b$) {\it cubic}, the same as ($a$) but with $y$ transformed by $\tilde{y} = (W y)^3$, where $W$ is a random linear mixing matrix. In theory, $I(X;Y) = I(X;\tilde{Y})$ for non-degenerate $W$. 

		\begin{figure}[H]
			\vspace{-1.em}
			\begin{minipage}{.65\textwidth}
				{\bf Comparison to baseline MI bounds.} We start by comparing $\FLO$ to the following popular competing variational estimators: $\NWJ$, $\TUBA$, and $\infonce$. We use the bilinear critic implementation for all models which maximally encourages both sample efficiency and code simplicity, and this strategy does perform best based on our observations. We consider the synthetic benchmark from \citep{poole2019variational}, where $(X\in\BR^d,Y\in\BR^d)$ is jointly standard Gaussian with diagonal cross-correlation parameterized by $\rho\in[0,1)$. We report $d=10$ and $\rho \in [0, 0.99]$ here (other studies only report $\rho$ up to $0.9$, which is less challenging.), providing a reasonable coverage of the range of MI one may encounter in empirical settings. 
			\end{minipage}
			\hspace{3pt}
			\begin{minipage}{.32\textwidth}
				\begin{center}
					\includegraphics[width=1\textwidth,trim={0 0 0.1in 0},clip]{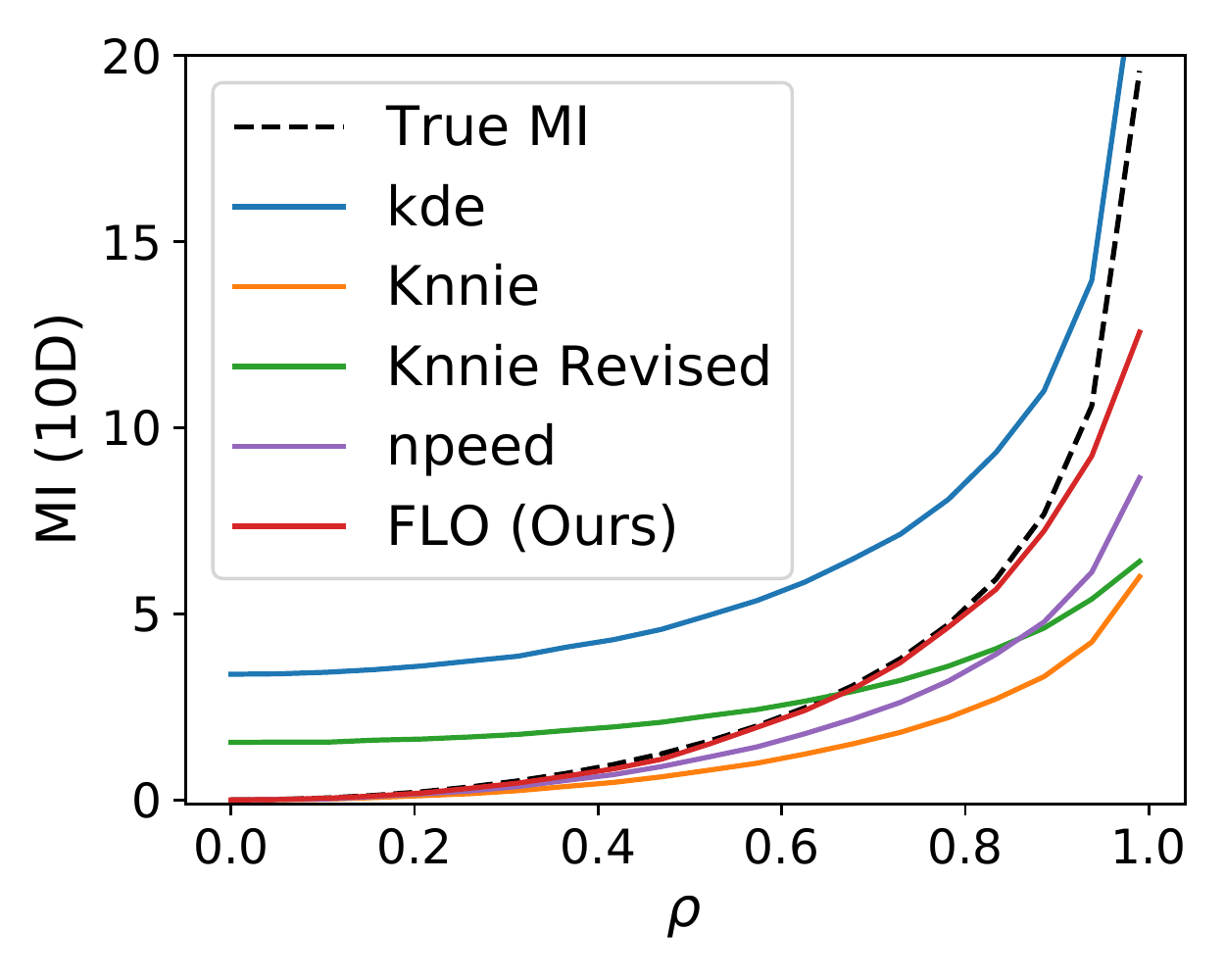}
				\end{center}
				\vspace{-1.8em}
				\caption{$\FLO$ compares favorably to classical MI estimators. \label{fig:classic}}
			\end{minipage}
			\vspace{-2.5em}
		\end{figure}
		To focus on the bias-variance trade-off, we plot the decimal quantiles in addition to the estimated MI in Figure \ref{fig:cmp_var_10d}, where $\FLO$ significantly outperformed its variational counterparts in the more challenging high-MI regime. In Figure \ref{fig:classic}, we show $\FLO$ also beats classical MI estimators \citep{kraskov2004estimating, ver2013information, gao2018demystifying}. In the Appendix \ref{sec:alt-appendix}, we further 
		% demonstrate $\FLO$ compare favorably to recent entropy-based estimator $\texttt{MIND}$, and 
		discuss recent works on parametric estimators \citep{cheng2020club, brekelmans2021improving} and alternative information metrics \citep{xu2020theory}.
		%We also compared to a recent work \texttt{MIND} claims to be the SOTA in sample efficiency, however we are unable to reproduce the results claimed by the paper and \texttt{MIND} performed poorly in the standard tests (see Appendix).
		%\footnote{Most studies only report $\rho$ up to $0.9$, which is less challenging.}

		% Now we want to show how different parameterization schemes affect the performance and learning efficiency for $\FLO$. In Figure \ref{fig:onetwo}, we visualize the learning dynamics of $\FLO$ using a shared network  and that with two separate networks. The parameter sharing not only cuts computations, it also helps to learn faster. There is no discernible difference in performance and $\FLO$-separate used twice much of iterations to converge.

		{\bf Bayesian optimal experiment design (BOED).} We next direct our attention to BOED, a topic of significant interest shared by the statistical and machine learning communities \citep{chaloner1995bayesian, wu2011experiments, hernandez2014predictive, foster2020unified}. The performance of machine learning models crucially relies on the quality of data supplied for training, and BOED is a principled framework that optimizes the data collection procedure (in statistical parlance, conducting {\it experiments}) \citep{foster2019variational}. Mathematically, let $x$ be the data to be collected, $\theta$ be the parameters to be inferred, and $d$ be the experiment parameters the investigator can manipulate ({\it a.k.a}, the {\it design parameters}), BOED tries to find the optimal data collection procedure that is expected to generate data that is most informative about the underlying model parameters, {\it i.e.}, solves for $\argmax_{d} I(x;\theta;d)$. In this study, we focus on the more generic scenario where explicit likelihoods are not available, but we can still sample from the data generating procedure \citep{kleinegesse2020bayesian, kleinegesse2021gradient}.
		
		We consider three carefully-selected models from recent literature for their progressive practical significance and the challenges involved \citep{foster2021deep, ivanova2021implicit, kleinegesse2021sequential}: static designs of ($i$) a simple linear regression model and ($ii$) a complex nonlinear pharmacokinetic model for drug development; and the dynamic policy design for ($iii$) epidemic disease surveillance and intervention ({\it e.g.}, for Covid-19 modeling). Designs with higher MI are more favorable, because it implies the data carries more information. In Figure \ref{fig:design} we compare design optimization curves using different MI optimization strategies, where $\FLO$ consistently leads. Popular $\NWJ$ and $\infonce$ reports different tradeoffs that are less susceptible to $\FLO$. We also examine the $\FLO$ predicted posteriors and confirm they are consistent with the ground-truth parameters (Figure \ref{fig:boed} right). For the dynamic policy optimization, we also manually inspect the design strategies reported by different models (Figure \ref{fig:boed} left,middle). Consistent with human judgement, $\FLO$ policy better assigns budgeted surveillance resources at different stages of pandemic progression. 

		\begin{figure}[t!]
			\begin{center}
				\includegraphics[width=.35\textwidth,trim={0 0 .5in 0},clip]{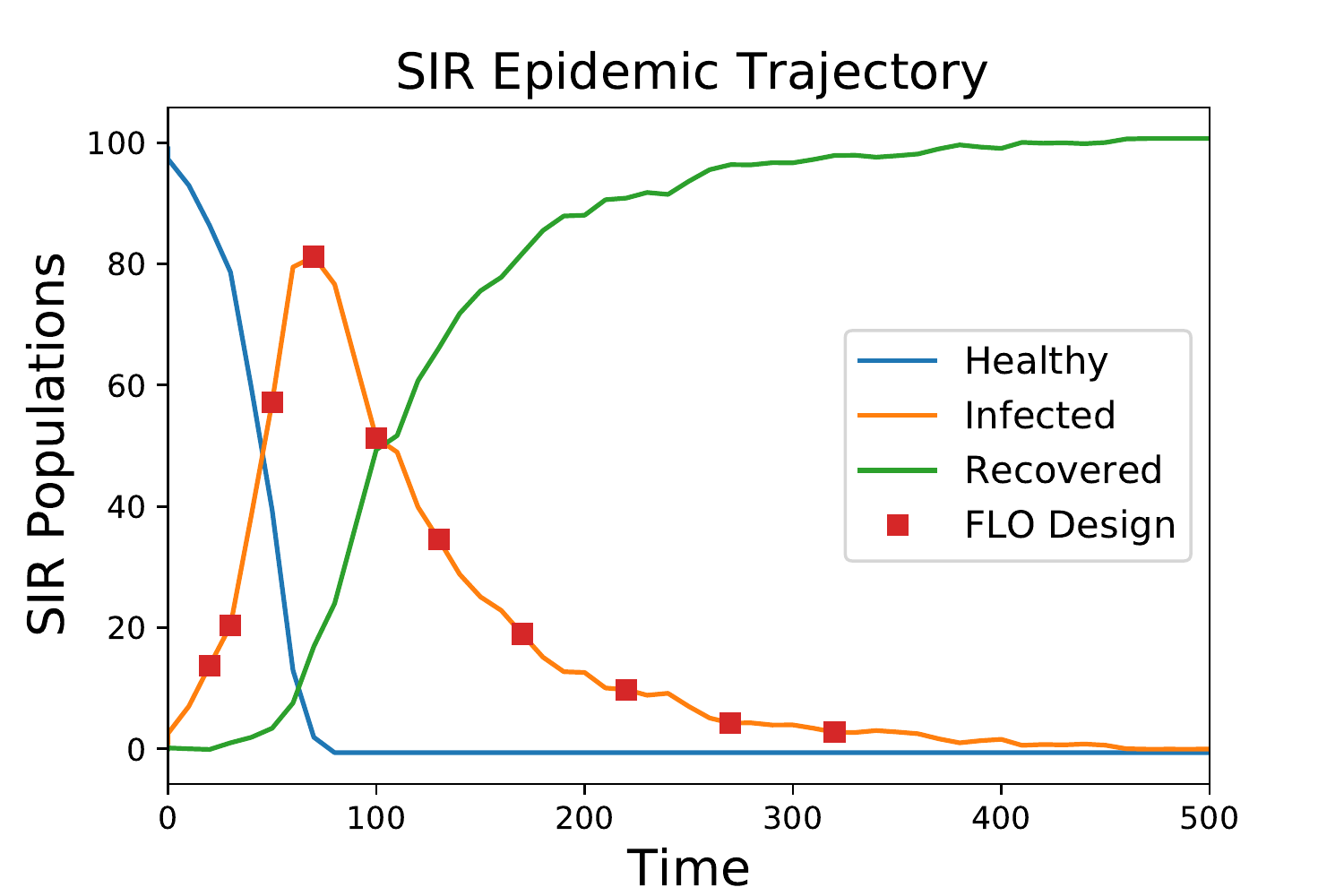}
				\includegraphics[width=.32\textwidth,trim={0.08in 0 0in 0},clip]{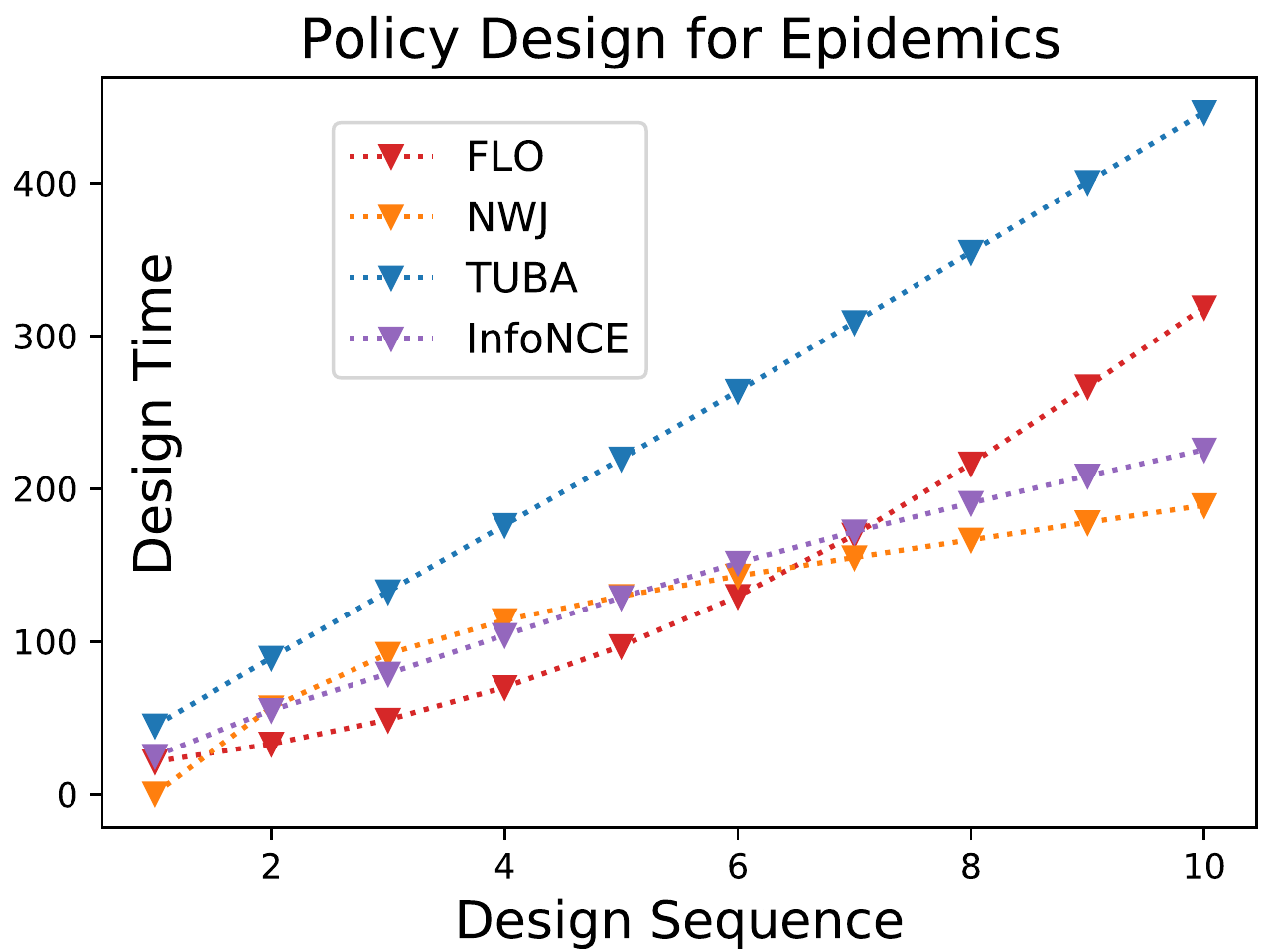}
				\includegraphics[height=.255\textwidth,trim={0 0 0.6in 0},clip]{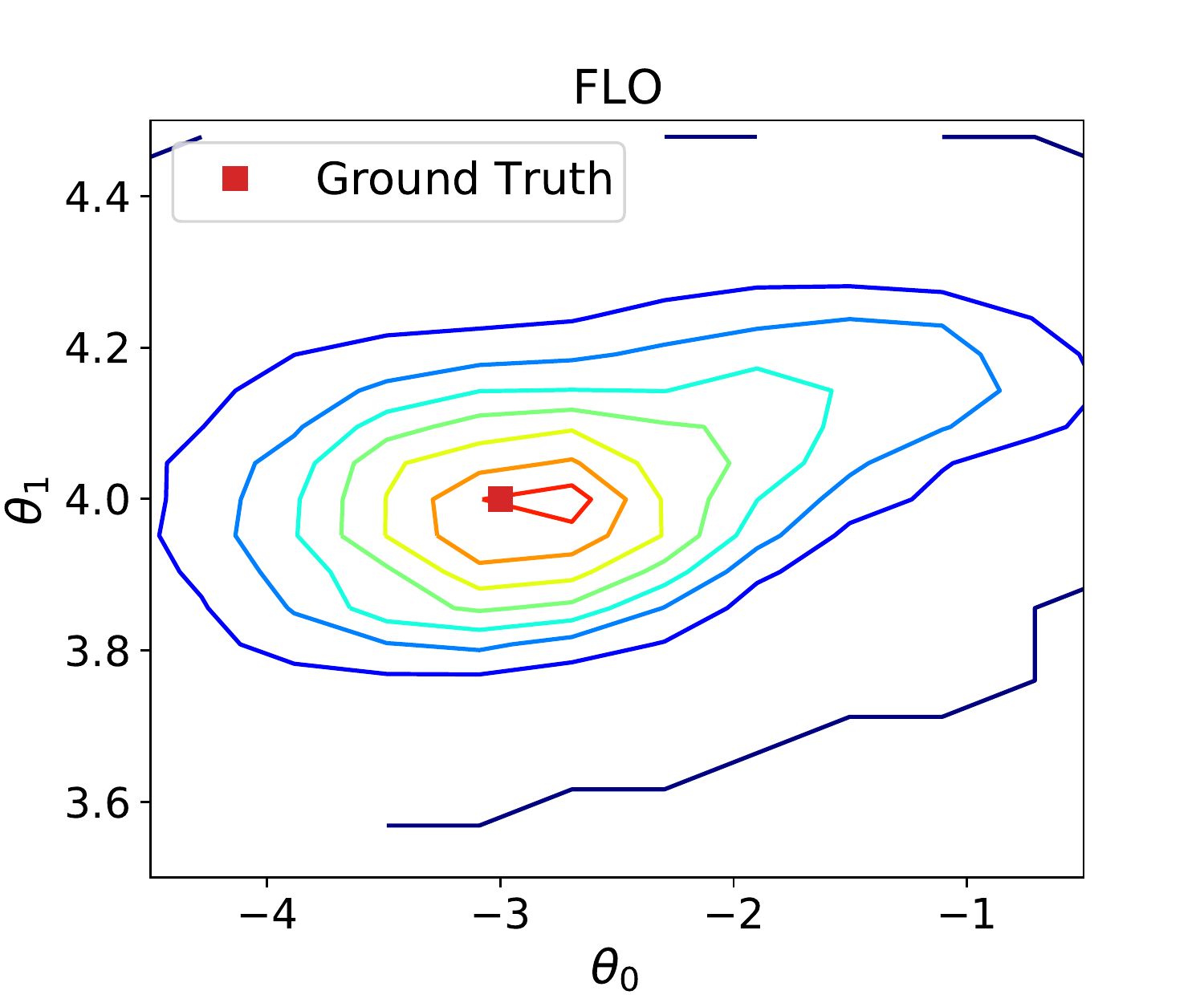}
			\end{center}
			\vspace{-1.5em}
			\caption{Diagnosis of learned sequential designs. The disease surveillance windows designed by $\FLO$ makes more sense: measures more frequently as infection spikes, and more sparsely when the pandemic slowly fades. The estimated parameter posterior (right) is consistent with the ground truth. \label{fig:boed}}
			\vspace{-1.em}
		\end{figure}

		{\bf A novel meta-learning framework.} A second application of our work is to meta-learning, an area attracting substantial recent interest. %We care because meta-learning research not only promises performance gains, but also significant reduction in both the carbon footprint and economical cost of AI applications. 
		In meta-learning, we are concerned with scenarios that at training time, there are abundant different labelled tasks, while upon deployment, only a handful of labeled instances are available to adapt the learner to a new task. Briefly, for an arbitrary loss $\ell_t(\hat{y}, y)$, where $t$ is the task identifier and $\hat{y}= f(x)$ is the prediction made by the model, we denote the risk by $R_t(f) = \EE_{p_t(x,y)}[\ell_t(f(x), y)]$. Denote $R(f)\triangleq \EE_{t\sim p(t)}[R_t(f)]$ as the expected risk for all tasks and $\hat{R}(f)$ for the mean of empirical risks computed from all training tasks. Inspired by recent information-theoretic generalization theories \citep{xu2017information}, we derived a novel, principled objective 
		\vspace{-3pt}
		\beq
		\CL_\texttt{Meta-FLO}(f) = \hat{R}(f) + \lambda \sqrt{I_{\FLO}(\hat{\CD}_t;\hat{E}_t)}, 
		\eeq
		
		\begin{figure}[H]
			\vspace{-1.em}
			\begin{minipage}{.38\textwidth}
				where $\lambda$ is known given the data size and loss function, $(\hat{\CD}_t, \hat{E}_t)$ are respectively data and task embeddings for training data, which for the first time lifts contrastive learning to the task and data distribution level. Our reasoning is that $\CL_\texttt{Meta-FLO}(f)$ theoretically bounds $R(f)$ from above, and it is relatively sharp for being data-dependent. We give more information on this in the Appendix and defer a full exposition to a dedicated paper due to independent interest and space limits here. Note other MI bounds are not suitable for this task due to resource and vari-
			\end{minipage}
			\hspace{3pt}
			\scalebox{.98}{
				\begin{minipage}{.6\textwidth}
					\vspace{-.5em}
					\renewcommand{\figurename}{Table}
					\setcounter{figure}{1}
					\caption{Multi-view representation learning on \texttt{Cifar} \label{tab:cifar}}
					\vspace{2pt}
					\scalebox{.88}{
						\begin{tabular}{ccccc}
							\toprule
							Model   & $\infonce$       & \texttt{SpecNCE} \citep{haochen2021provable} \footnote{Note \texttt{SpecNCE} does not explicitly target mutual information}      & $\FLO$ & $\FDV$  \\
							%& & & (ours) & (ours) \\
							\midrule
							MI     & $5.73\pm.07$ & $4.76\pm.08$ & $5.83\pm .08$ & {$\bs{5.93\pm.08}$} \\
							\bottomrule
						\end{tabular}
					}
					\vspace{-.5em}
					\renewcommand{\figurename}{Figure}
					\setcounter{figure}{6}
					\begin{center}
						\includegraphics[width=0.51\textwidth,trim={0 0 .6in 0},clip]{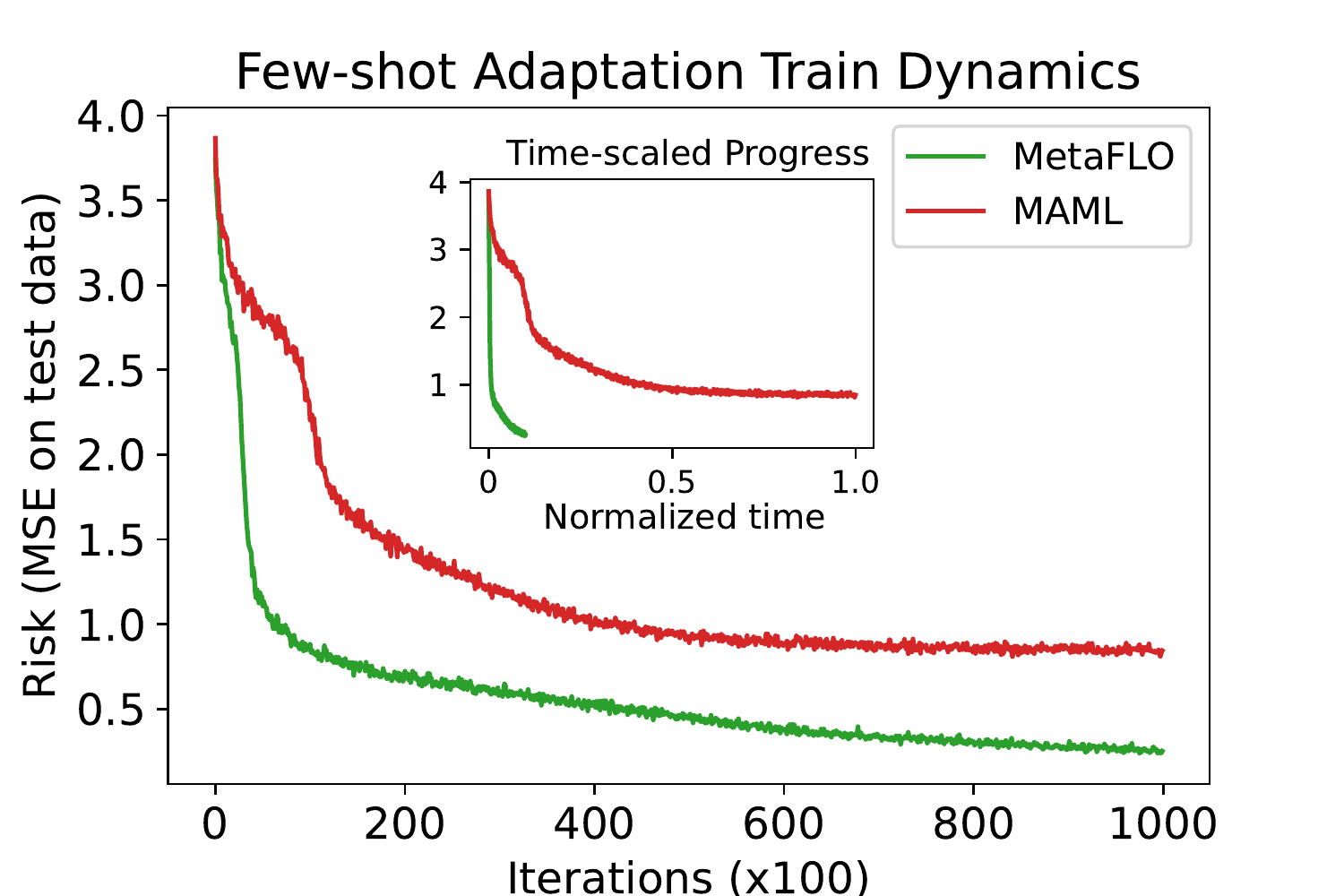}
						\includegraphics[width=0.48\textwidth,trim={.48in 0 .5in 0},clip]{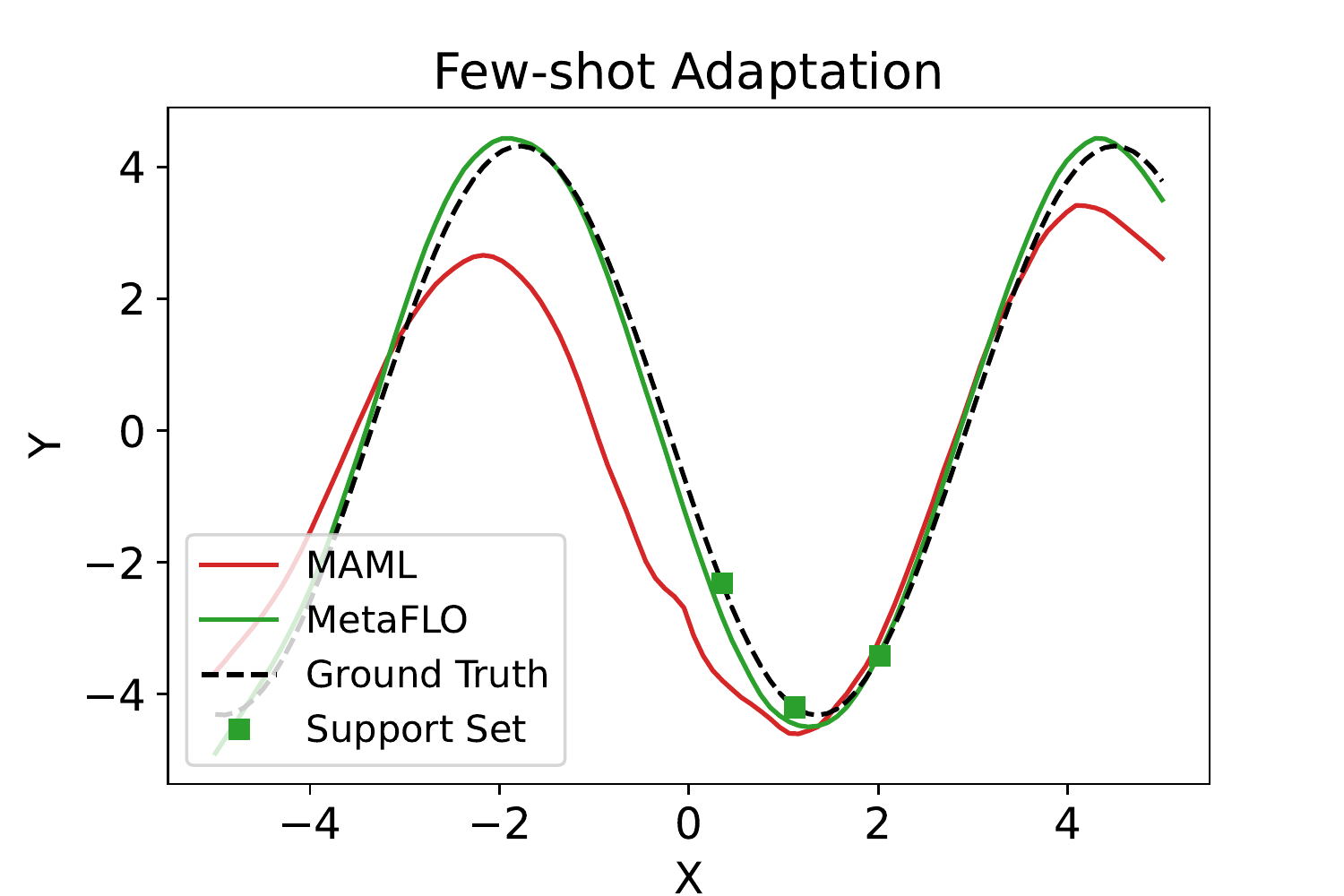}
					\end{center}
					\vspace{-1.2em}
					\caption{Few-shot adaptation with \texttt{Meta-FLO}. \label{fig:meta}}
				\end{minipage}
			}
			\vspace{-1.7em}
		\end{figure}
		ance concerns. In Figure \ref{fig:meta} we show $\texttt{Meta-FLO}$ wins big over the state-of-the-art {\it model agnostic meta-learning} (MAML) model on the regression benchmark from \citep{finn2017model}.

		%\begin{table}[t!]
		%\centering
		%\caption{Multi-view representation learning on \texttt{Cifar}
			%\label{tab:cifar}}
		%\vspace{2pt}
		%\scalebox{.88}{
			%\begin{tabular}{ccccc}
			%\toprule
			%Model   & $\infonce$       & \texttt{SpecNCE}       & $\FLO$ & $\FDV$  \\
			%%& & & (ours) & (ours) \\
			%\midrule
			%MI     & $5.73\pm.07$ & $4.76\pm.08$ & $5.83\pm .08$ & {$\bs{5.93\pm.08}$} \\
			%\bottomrule
			%\end{tabular}
			%}
		%\vspace{-1.em}
		%\end{table}
		
		%$R_{\mu}(W) \triangleq \EE_{Z\sim\mu}[\ell(w, Z)]$ and $R_S(w) \triangleq \frac{1}{m} \sum_{i=1}^m \ell(w, Z_i)$.
		
		%Consider an unknown distribution $\mu$ on an instance space $\CZ = \CX \times \CY$, and a set of independent samples $S = \{ Z_i \}_{i=1}^m$ drawn from $\mu$. Given a parameterized hypothesis space $\CW$ and a loss function $\ell: \CW \times \CZ \rightarrow \BR$,
		
		%Skipping all details, with an abundance of training tasks, 
		
		% {\bf Converging evidence from more applications.} In addition to the tasks above, we also examined $\FLO$ with extensive ablations and validated its utility of tasks such as representation learning (Table \ref{tab:cifar}) and fair classification. Due to space limits, results and analyses are delegated to the Appendix. We also point readers to other researches inspired by our work, which shows significant boosts in both performance and sample efficiency compared to $\infonce$-based SOTA solutions on \texttt{ImageNet}-scale datasets using the $\FDV$ bound \citep{chen2021simpler}.  
		
		\begin{figure}[t!]
			\centering
			\begin{minipage}{.325\textwidth}
				\centering
				\vspace{-.2em}
				\includegraphics[width=1.\textwidth]{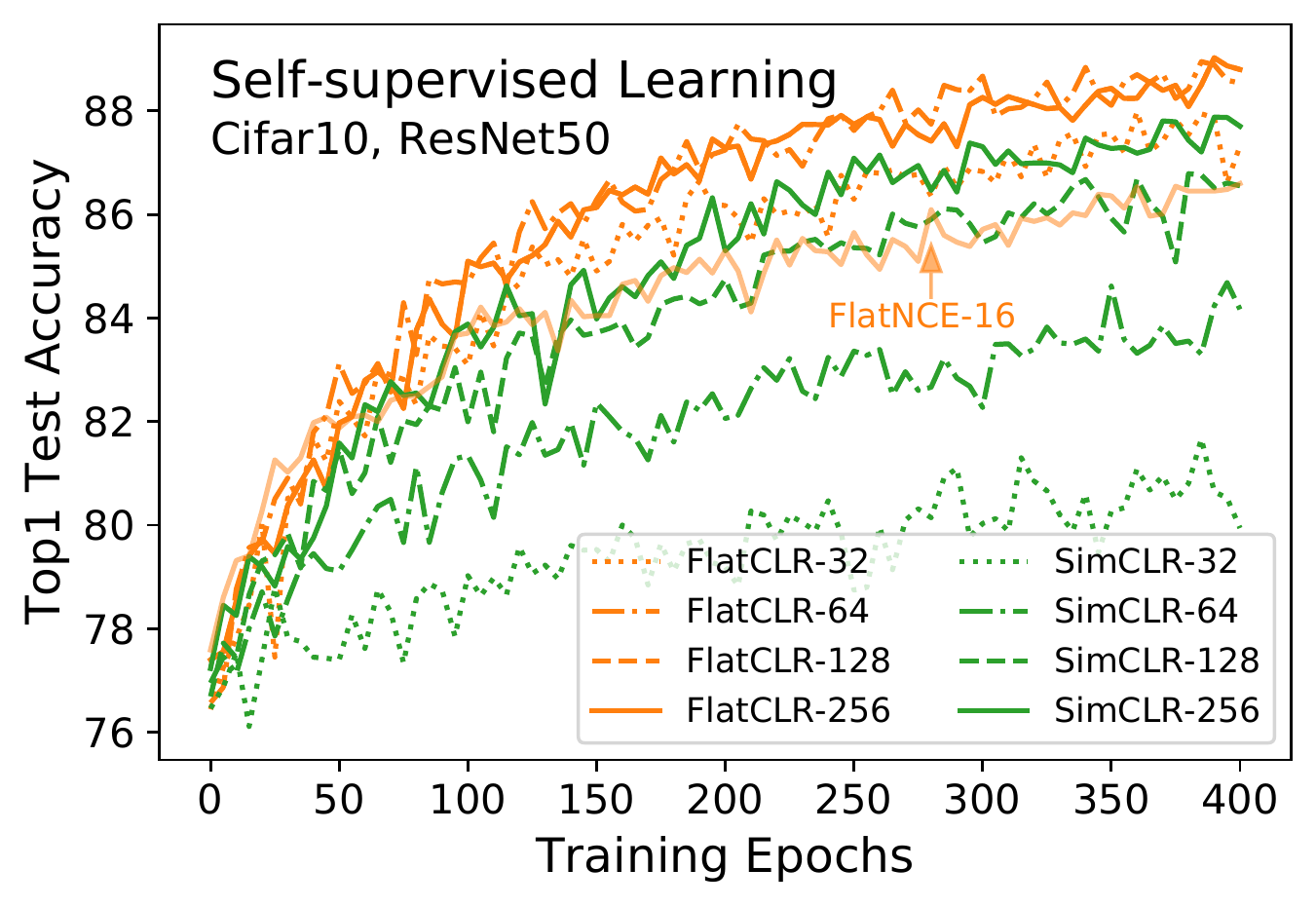}
				\vspace{-1.2em}
				\captionof{figure}{Sample efficiency comparison for $\SimCLR$ and $\FCLR$ on \texttt{Cifar10}. \label{fig:batchsize}}
			\end{minipage}%
			\hspace{1pt}
			\begin{minipage}{.245\textwidth}
				% \vspace{1.7em}
				\centering
				\vspace{-.2em}
				\includegraphics[width=.95\textwidth]{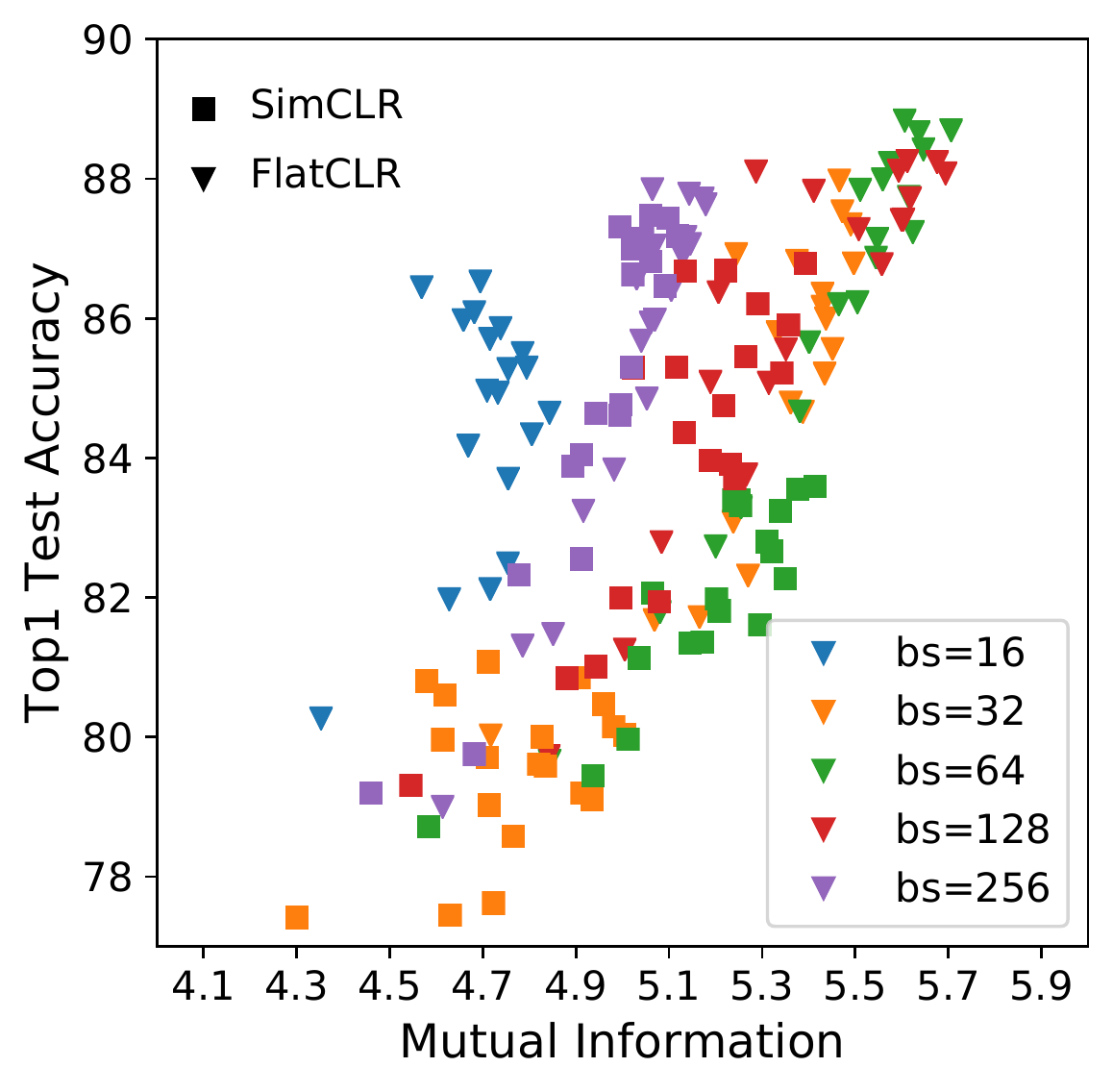}
				\vspace{-.2em}
				\captionof{figure}{Representation MI strongly correlates with performance. \label{fig:mi_acc}}
				%   \label{fig:test1}
			\end{minipage}%
			\hspace{1pt}
			\begin{minipage}{.375\textwidth}
				%\vspace{1em}
				\centering
				\includegraphics[width=1.\textwidth]{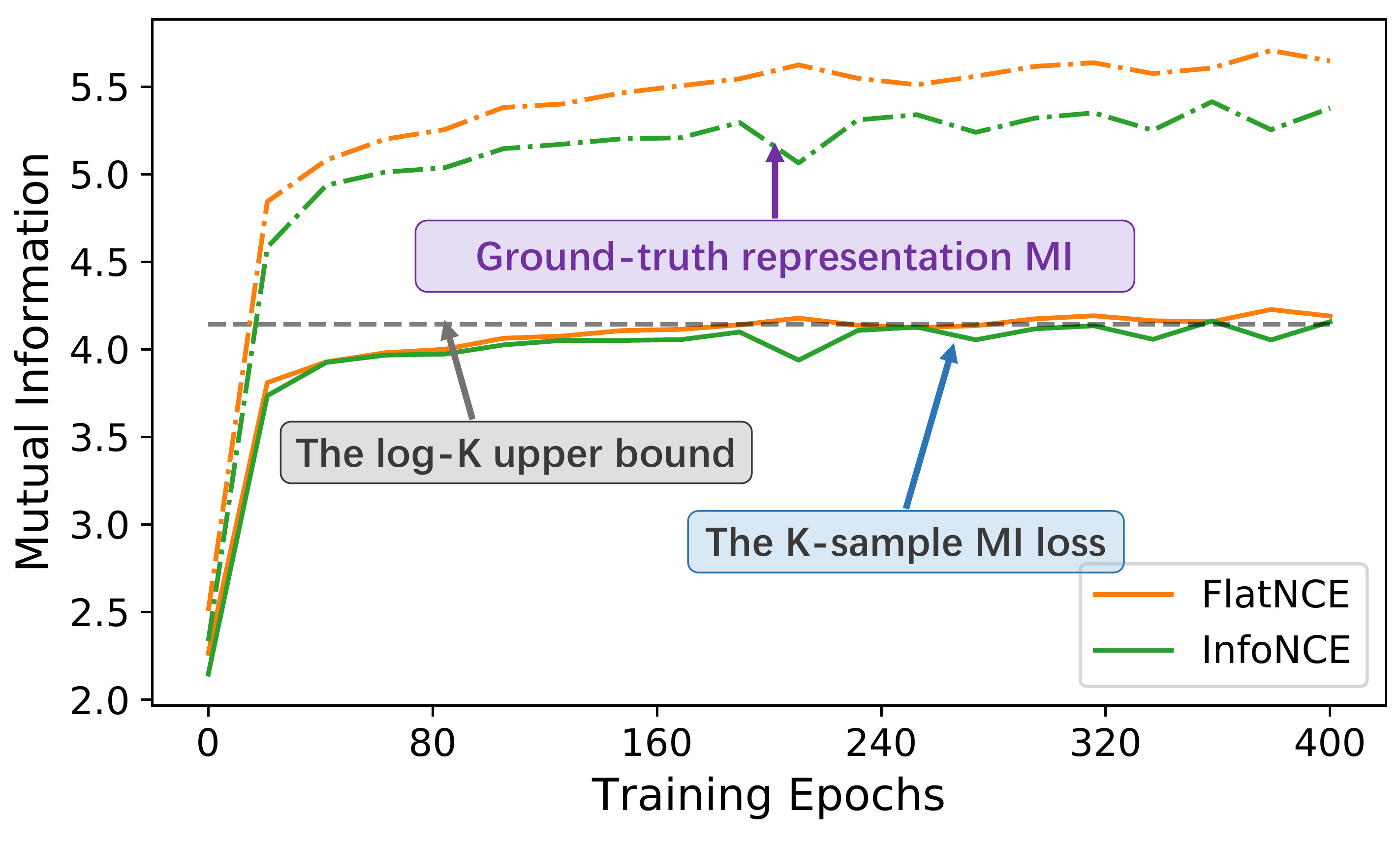}
				\vspace{-1.2em}
				\captionof{figure}{ 
					{$\FLAT$} better optimizes the true MI with the same mini-batch size. (\texttt{Cifar10} SSL training)  \label{fig:mi_opt} 
					%  $\FLAT$ continues to robustly optimize the ground-truth MI for representation even after the finite sample loss has saturated at $\log K$. \label{fig:mi_opt} 
				}
				%   \label{fig:test2}
			\end{minipage}
			\vspace{-1.em}
		\end{figure}

		{\bf Self-supervised learning (SSL)}. Finally, we wrap our experiments with one of the prime applications of contrastive MI estimation in machine learning: SSL for model pre-training. Here we focus on how $\FLO$-inspired objectives can improve the current practice of SSL, and given this topic's independent interest, we defer in-depth discussions in our dedicated work \citep{chen2021simpler} where SSL-specific problems such as training diagnosis and low-precision numerical overflow are explored in detail. In this experiment, we follow the SSL setup described in the $\SimCLR$ paper \citep{chen2020simple}: in the pre-training phase, we optimize the mutual information between difference augmentations of the same image ({\it i.e.}, scaling, rotation, color jitting, {\it etc.}); and use linear probing accuracy as our perfromance criteria. 
		We compare the effectiveness of the $\infonce$-based $\SimCLR$ framework \citep{chen2020simple} to our $\FLO$-based alternatives. To ensure fair comparison, we have used the $\FDV$ variant defined in Eq. (\ref{eq:fdv}) as our training objective, so that we are not introducing extra parameters to model $u(x,y)$. We call our new model $\FCLR$ because, perhaps counter-intuitively, the second term in Eq. (\ref{eq:fdv}) contributing all the learning signal is constant one in value ({\it i.e.}, being flat). In  Figure \ref{fig:batchsize} and \ref{fig:mi_acc}, we show our new model $\FCLR$ shows superior sample efficiency compared to the SOTA $\SimCLR$ (a $8 \times$ boost for the same performance, $\FCLR$-$32$ versus $\SimCLR$-$256$). This result is significant because $\SimCLR$'s crucial reliance on large-batch training is a well-known limitation \citep{haochen2021provable, yuan2022provable, lee2022r}. Figure \ref{fig:mi_opt} shows typical training curves with the respective models. Note that while the empirical estimates of MI are tied between the two methods, $\FDV$ optimized representation enjoys a better ground-truth MI \footnote{Ground-truth MI is approximated by $\infonce$ using a very large negative sample pool ($100 \times$ mini-batch).}, which can be explained by its robustness to the numerical overflow issue (see \citep{chen2021simpler} for details). Further comparisons on the ground-truth MI estimation with different estimators can be found in Table \ref{tab:cifar}. 
		
		% estimator not only better optimizes MI, and this translates into better performance and learning efficiency. On 

		% \begin{figure}[t!]
			% \begin{center}
				% \hspace*{-2in}
				% \includegraphics[width=.32\textwidth,height=11em,trim={0 0 .5in 0},clip]{figures/design/SIRprocess}
				% \includegraphics[width=.32\textwidth,height=10.5em,trim={0.08in 0 0in 0},clip]{figures/design/SIRDesign}
				% \includegraphics[width=.3\textwidth,trim={0 0 0.6in 0},clip]{figures/design/posteriorFLO1}
				% \includegraphics[width=.32\textwidth,height=10.3em,trim={0 0 0.1in 0},clip]{figures/toy/kde-new}
				% \hspace*{-2in}
				% \end{center}
			% \vspace{-1.5em}
			% \caption{(left two) Policy comparisons for different design models. Budgeted disease surveillance windows designed by $\FLO$ makes more sense because it measures more frequently as infection spikes, and more sparsely when the pandemic slowly fades.(middle) Design optimized parameter posterior consistent with the ground truth; (right) $\FLO$ bound compares favorably to classical MI estimators. \label{fig:misc}}
			% \vspace{-1.8em}
			% \end{figure}
		
		\vspace{-10pt}
		\section{Conclusion}
		\vspace{-8pt}
		
		We have described a new framework for the contrastive estimation of mutual information from energy modeling perspectives. Our work not only encapsulates popular variational MI bounds but also inspires novel objectives such as $\FLO$ and $\FDV$, which comes with strong theoretical guarantees. In future work, we will leverage our theoretical insights to improve practical applications involving MI estimation, such as representation learning, fairness, and in particular, data efficient learning.

		\vspace{-8pt}
		\section*{Acknowledgements}
		\vspace{-6pt}
		% This research was supported in part by NIH/NIDDK R01-DK123062, NIH/NIBIB R01-EB025020, NIH/NINDS 1R61NS120246, DARPA, DOE, ONR and NSF. J. Chen was partially supported by Shanghai Municipal Science and Technology Major Project (No.2018SHZDZX01) and National Key R\&D Program of China (No.2018AAA0100303).
		The authors would like to thank the anonymous reviewers for their insightful comments.
		Q Guo gratefully appreciate the support of Amazon Fellowship. X Deng would like to thank the Advanced Research Computing program at Virginia Tech and Virginia’s Commonwealth Cyber Initiative (CCI) AI testbed for providing computational resources, also appreciate the CCI and CCI-Coastal grants to Virginia Tech. Part of this work is done before C Tao joined Amazon, and he was funded by National Science Foundation Grant No. 1934964. This work used the Extreme Science and Engineering Discovery Environment (XSEDE), which is supported by National Science Foundation grant number ACI-1548562 \citep{towns2014xsede} and used the Extreme Science and Engineering Discovery Environment (XSEDE) PSC Bridges-2 and SDSC Expanse at the service-provider through allocation TG-ELE200002 and TG-CIS210044. 
		
		% \newpage
		\bibliography{finice}
		\bibliographystyle{plain}

		\section*{Checklist}

		%%% BEGIN INSTRUCTIONS %%%
		% The checklist follows the references.  Please
		% read the checklist guidelines carefully for information on how to answer these
		% questions.  For each question, change the default \answerTODO{} to \answerYes{},
		% \answerNo{}, or \answerNA{}.  You are strongly encouraged to include a {\bf
			% justification to your answer}, either by referencing the appropriate section of
		% your paper or providing a brief inline description.  For example:
		% \begin{itemize}
			%   \item Did you include the license to the code and datasets? \answerYes{See Section~\ref{gen_inst}.}
			%   \item Did you include the license to the code and datasets? \answerNo{The code and the data are proprietary.}
			%   \item Did you include the license to the code and datasets? \answerNA{}
			% \end{itemize}
		% Please do not modify the questions and only use the provided macros for your
		% answers.  Note that the Checklist section does not count towards the page
		% limit.  In your paper, please delete this instructions block and only keep the
		% Checklist section heading above along with the questions/answers below.
		%%% END INSTRUCTIONS %%%

		\begin{enumerate}

			\item For all authors...
			\begin{enumerate}
				\item Do the main claims made in the abstract and introduction accurately reflect the paper's contributions and scope?
				\answerYes{}
				\item Did you describe the limitations of your work?
				\answerYes{}
				\item Did you discuss any potential negative societal impacts of your work?
				\answerNA{}
				\item Have you read the ethics review guidelines and ensured that your paper conforms to them?
				\answerNA{}
			\end{enumerate}

			\item If you are including theoretical results...
			\begin{enumerate}
				\item Did you state the full set of assumptions of all theoretical results?
				\answerYes{}
				\item Did you include complete proofs of all theoretical results?
				\answerYes{}
			\end{enumerate}

			\item If you ran experiments...
			\begin{enumerate}
				\item Did you include the code, data, and instructions needed to reproduce the main experimental results (either in the supplemental material or as a URL)?
				\answerYes{They are in the Supplementary Material.}
				\item Did you specify all the training details (e.g., data splits, hyperparameters, how they were chosen)?
				\answerYes{They are summarized in the Supplementary Material.}
				\item Did you report error bars (e.g., with respect to the random seed after running experiments multiple times)?
				\answerYes{Smaller variance is the highlight of this paper.}
				\item Did you include the total amount of compute and the type of resources used (e.g., type of GPUs, internal cluster, or cloud provider)?
				\answerYes{}
			\end{enumerate}

			\item If you are using existing assets (e.g., code, data, models) or curating/releasing new assets...
			\begin{enumerate}
				\item If your work uses existing assets, did you cite the creators?
				\answerYes{}
				\item Did you mention the license of the assets?
				\answerNA{}
				\item Did you include any new assets either in the supplemental material or as a URL?
				\answerNA{}
				\item Did you discuss whether and how consent was obtained from people whose data you're using/curating?
				\answerNA{}
				\item Did you discuss whether the data you are using/curating contains personally identifiable information or offensive content?
				\answerNA{}
			\end{enumerate}

			\item If you used crowdsourcing or conducted research with human subjects...
			\begin{enumerate}
				\item Did you include the full text of instructions given to participants and screenshots, if applicable?
				\answerNA{}
				\item Did you describe any potential participant risks, with links to Institutional Review Board (IRB) approvals, if applicable?
				\answerNA{}
				\item Did you include the estimated hourly wage paid to participants and the total amount spent on participant compensation?
				\answerNA{}
			\end{enumerate}

		\end{enumerate}

		%%%%%%%%%%%%%%%%%%%%%%%%%%%%%%%%%%%%%%%%%%%%%%%%%%%%%%%%%%%%
		\newpage
		
		\appendix
		
		{\LARGE \bf Appendix}
		% \section*{Appendix}
		%
		%
		%Optionally include extra information (complete proofs, additional experiments and plots) in the appendix.
		%This section will often be part of the supplemental material.

		% \appendix
		
		\renewcommand{\thetable}{S\arabic{table}}
		\renewcommand{\thefigure}{S\arabic{figure}}
		\renewcommand{\theequation}{S\arabic{equation}}
		
		% %%%%% To create table of contents only for Appendix %%%%%%
		
		% \usepackage{}\addcontentsline{toc}{section}{Appendix} % Add the appendix text to the document TOC
		% \part{Appendix} % Start the appendix part
		% \parttoc % Insert the appendix TOC
		
		\section{Proof of Proposition 2.1 (InfoNCE Properties and derivation for some popular variational MI bounds)}
		\label{sec:infonce_appendix}
		
		\begin{proof}
			
			Now let us prove $\infonce$ is a lower bound to MI and under proper conditions this estimate is tight. Our proof is based on establishing that $\infonce$ is a multi-sample extension of the $\NWJ$ bound. For completeness, we first repeat the proof for $\BA$ and $\UBA$ below, and then show $\UBA$ leads to $\NWJ$ and its multi-sample variant $\infonce$. 
			
			We can bound MI from below using an variational distribution $q(y|x)$ as follows: 
			\beqs
			I(X, Y) & = & \EE_{p(x,y)}\left[\log \frac{p(x,y)}{p(x)p(y)}\right]\\
			& = & \EE_{p(x,y)}\left[\log \frac{p(y|x)p(x)q(y|x)}{p(x)p(y)q(y|x)}\right] \text{  {\it \quad \# q(y|x) is the variational distribution}}\\
			& = & \EE_{p(x,y)}\left[\log \frac{q(y|x)}{p(y)}\right] + \EE_{p(x)}[\KL(p(y|x)||q(y|x))]\\
			& \geq & \EE_{p(x,y)}\left[\log q(y|x) - \log p(y) \right] \triangleq I_{\BA}(X,Y;q) \label{eq:ba}
			\eeqs
			In sample-based estimation of MI, we do not know the ground-truth marginal density $p(y)$, which makes the above $\BA$ bound impractical. However, we can carefully choose an energy-based variational density that ``cancels out'' $p(y)$:
			\beq
			q_f(y|x) = \frac{p(y)}{Z_f(x)}e^{f(x,y)}, \quad Z_f(x) \triangleq \EE_{p(y)}[e^{f(x,y)}]. 
			\eeq
			This auxiliary function $f(x,y)$ is known as the tilting function in importance weighting literature. Hereafter, we will refer to it the {\it critic function} in accordance with the nomenclature used in contrastive learning literature.  The partition function $Z_f(x)$ normalizes this $q(y|x)$. Plugging this $q_f(y|x)$ into $I_\BA$ yields:
			\beqs
			I_{\BA}(X,Y;q_f) & = & \EE_{p(x,y)}[f(x,y)+\log (p(y))-\log (Z(x))-\log p(y)]\\
			& = & \EE_{p(x,y)}[f(x,y)]-\EE_{p(x)}[\log (Z_f(x))] \triangleq I_{\UBA}(X,Y;f) \label{eq:uba}
			\eeqs
			
			For $x,a>0$, we have inequality $\log (x)\leq\frac{x}{a}+\log (a)-1$. By setting $x\leftarrow Z(y)$ and $a\leftarrow e$, we have 
			\beq
			\log (Z(y))\leq e^{-1}{Z(y)}.
			\eeq 
			Plugging this result into (\ref{eq:uba}) we recover the celebrated $\NWJ$ bound, which lower bounds $I_{\UBA}$: 
			\beq
			I_{\UBA}(X,Y) \geq \EE_{p(x,y)}[f(x,y)]-e^{-1}\EE_{p(x)}[Z_f(x)] \triangleq I_{\NWJ}(X,Y;f). \label{eq:nwj}
			\eeq
			When $f(x,y)$ takes the value of 
			\beq
			f^*(x,y) = 1 + \log \frac{p(x|y)}{p(x)}, 
			\eeq
			this bound is sharp.

			%Here, $f(x,y)$ must be self-normalized , yielding a unique optimal critic $f^*(x,y)=1+\log \frac{p(x|y)}{p(x)}$
			
			We next extend these bounds to the multi-sample setting. In this setup, we are given one paired sample $(x_1, y_1)$ from $p(x,y)$ ({\it i.e.}, the positive sample) and $K-1$ samples independently drawn from $p(y)$ ({\it i.e.}, the negative samples). Note that when we average over $x$ wrt $p(x)$ to compute the MI, this equivalent to comparing positive pairs from $p(x,y)$ and negative pairs artificially constructed by $p(x)p(y)$. By the independence between $X_1$ and $Y_{k>1}$, we have 
			
			%$I(X_1,Y_1)$ is given samples from $p(x,y)$ and access to $K-1$ additional samples $y_{2:K}\sim r^{K-1}(y_{2:K})=\prod^K_{j=2}p(y_j)$ (potentially from a different distribution than $Y_1$). For any random variable $Z$ independent from $X$ and $Y$. $I(X;Y,Z) = I(X;Y)$, therefore
			\beq
			%I(X,Y_1) = \EE_{r^{K-1}(y_{2:K})}[I(X;Y_1)]=I(X;Y_1,Y_{2:K})
			I(X;Y_{1:K}) = \EE_{p(x_1,y_1)\prod_{k>1}p(y_k)}\left[\frac{p(x_1,y_1)\prod_{k>1}p(y_k)}{p(x_1) \prod_k p(y_k)}\right] = \EE_{p(x_1,y_1)}\left[\frac{p(x_1,y_1)}{p(x_1)p(y_1)} \right] = I(X;Y)
			\eeq
			
			So for arbitrary multi-sample critic $f(x; y_{1:K})$, we know 
			\beq
			I(X;Y) = I(X_1;Y_{1:K}) \geq I_{\NWJ}(X_1,Y_{1:K};f) = \EE_{p(x_1,y_1)\prod_{k>1}p(y_k)}[f(x_1,y_{1:K})]-e^{-1}\EE_{p(x)}[Z_f(x)]
			\eeq
			
			Now let us set 
			\beq
			\tilde{f}(x_1;y_{1:K}) = 1 + \log \frac{e^{g(x_1, y_1)}}{m(x_1;y_{1:K})}, \quad m(x_1;y_{1:K}) = \frac{1}{K} \sum_k e^{g(x_1,y_k)}. 
			\label{eq:fg}
			\eeq
			
			\begin{equation*}
				\begin{split}
					I_{\NWJ}(X_1,Y_{1:K};\tilde{f}) = & \EE_{p(x_1,y_1)p^{K-1}(y_k)}\left[ 1+ \log \frac{e^{g(x_1, y_1)}}{m(x_1;y_{1:K})} \right] - \EE_{p(x')p^K(y')} \left[e^{-1+1 +\log \frac{e^{g(x_1', y_1')}}{m(x_1'; y_{1:K}' )} }\right] \\
					& = \EE_{p(x_1,y_1)p^{K-1}(y_k)}\left[ 1+ \log \frac{e^{g(x_1, y_1)}}{m(x_1;y_{1:K})} \right] - \EE_{p(x')p^K(y')} \left[ \frac{e^{g(x_1', y_1')}}{m(x_1'; y_{1:K}' )} \right] 
				\end{split}
			\end{equation*}
			
			Due to the symmetry of $\{y_k\}_{k=1}^K$, we have 
			\beq
			\EE_{p(x')p^K(y')} \left[ \frac{e^{g(x_1', y_1')}}{m(x_1'; y_{1:K}' )} \right]  = \EE_{p(x')p^K(y')} \left[ \frac{e^{g(x_1', y_k')}}{m(x_1'; y_{1:K}' )} \right].
			\eeq
			So this gives
			\beq
			\EE_{p(x')p^K(y')} \left[ \frac{e^{g(x_1', y_1')}}{m(x_1'; y_{1:K}' )} \right] = \EE_{p(x')p^K(y')} \left[ \frac{\frac{1}{K}e^{g(x_1', y_k')}}{m(x_1'; y_{1:K}' )} \right]  = 1,
			\eeq
			and one can easily see this recovers the $K$-sample $\infonce$ defined in (3)
			\beq
			I_{\NWJ}(X_1,Y_{1:K};\tilde{f}) = \EE_{p(x_1,y_1)p^{K-1}(y_k)}\left[\log \frac{e^{g(x_1, y_1)}}{m(x_1;y_{1:K})} \right] = I_{\infonce}^K(X;Y|g)
			\eeq
			
			Now we need to show this bound is sharp when $K\rightarrow \infty$. We only need to show that for some choice of $g(x,y)$, the inequality holds asymptotically. Recall the $\NWJ$'s optimal critic takes value of $f^*(x,y) = 1+ \frac{p(x|y)}{p(x)}$, so with reference to (\ref{eq:fg}) let us plug in $g^*(x,y) = \frac{p(y|x)}{p(y)}$ into $\infonce$
			
			\beqs
			\CL_K^* & = & \EE_{p^K} \left[ \log \left( \frac{f^*(x_k,y_k)}{f^*(x_k,y_k) + \sum_{k'\neq k} f^*(x_k,y_{k'})} \right) \right] + \log K \\
			& = & -\EE\left[ \log \left( 1+\frac{p(y)}{p(y|x)} \sum_{k'} \frac{p(y_{k'}|x_k)}{p(y_{k'})} \right)\right] + \log K \\
			& \approx & -\EE\left[ \log \left( 1+\frac{p(y)}{p(y|x)} (K-1) \EE_{y_{k'}} \frac{p(y_{k'}|x_k)}{p(y_{k'})} \right)\right] + \log K\\
			& = &  -\EE\left[ \log \left( 1+\frac{p(y_k)}{p(y_k|x_k)} (K-1) \right)\right] + \log K\\
			& \approx & -\EE\left[\log\frac{p(y)}{p(y|x)}\right] - \log (K-1) + \log K\\
			(K\rightarrow \infty) & \rightarrow & I(X;Y)
			\eeqs
			This concludes our proof.
		\end{proof}
		
		\section{Proof of Proposition 2.2 ($\FLO$ lower bounds MI)}
		
		\begin{proof}
			The proof is given in line 133-140 in the main text. Basically we have applied the Fenchel duality trick to the $\log$ term in the $\UBA$ bound. Note that unlike $\UBA$, our $\FLO$ bound can be unbiased estimated with finite samples (as $\UBA$ requires an infinite sum inside its $\log$ term, which makes finite-sample empirical estimate biased per Jensen's inequality). 
			% Equation (17) is a direct consequence of applying the Fenchel duality trick to the $\UBA$ bound. We already know $\UBA$ is sharp when $g^*(x,y) = \log p(x|y) + c(x)$, and the Fenchel duality holds when $u^*(x,y;g) = \log \EE_{p(y')}[\exp(g(x,y')-g(x,y))]$. So Equation (17) holds with $(g,u) = (g^*, u^*(g^*))$. We can also see this from equation (14).   
		\end{proof}
		
		\section{Proof of Proposition 2.3, Corollary 2.4 ($\FLO$ tightness, meaning of $u(x,y)$)}
		
		\begin{proof}
			% This is immediate from equation (14).
			The proof is given in the main text, more specifically the paragraph preceding Proposition 2.3. 
		\end{proof}

		\section{Gradient Analysis of FLO (More Detailed)}
		\vspace{-8pt} % subsec_post
		\label{sec:grad}
		
		To further understand the workings of $\FLO$, let us inspect the gradient of model parameters. Recall the intractable $\UBA$ MI estimator can be re-expressed in the following form:
		\beq
		I_{\UBA'}(g_\theta) = \EE_{p(x,y)}[-\log \EE_{p(y')}[\exp(g_{\theta}(x,y')-g_{\theta}(x,y))]]
		\eeq
		In this part, we want to establish the intuition that $\nabla_{\theta}\{ I_{\FLO}(u_\phi, g_\theta) \} \approx  \nabla_{\theta} \{I_{\UBA'}(g_\theta)\}$, where 
		\beq
		I_{\FLO}(u_{\phi},g_{\theta}) \triangleq - \left\{ u_{\phi}(x,y) + \EE_{p(y')}[ \exp(-u_{\phi}(x,y) + g_{\theta}(x,y') - g_{\theta}(x,y))] \right\}
		\eeq
		is our $\FLO$ estimator.

		%\beq
		%\nabla_{\theta}I(\theta) = - \frac{}{}
		%\eeq
		By defining
		\beq
		\CE_{\theta}(x,y) \triangleq \frac{1}{ \EE_{p(y')}[\exp(g_{\theta}(x,y')-g_{\theta}(x,y))]}, 
		\eeq
		we have
		\beq
		\nabla_{\theta} \left\{ \frac{1}{\CE_{\theta}(x,y)} \right\} = - \frac{\nabla \CE_{\theta}(x,y)}{(\CE_{\theta}(x,y))^2} = -\frac{\nabla_{\theta} \log \CE_{\theta}(x,y)}{\CE_{\theta}(x,y)},  
		\eeq
		and 
		\beqs
		\nabla_{\theta} \left\{ \frac{1}{\CE_{\theta}(x,y)} \right\} & = & \nabla_{\theta}\EE_{p(y')}[\left\{ \exp(g_{\theta}(x,y') - g_{\theta}(x,y)) \right\}] \\
		& = & \EE_{p(y')}[\nabla_{\theta}\left\{ \exp(g_{\theta}(x,y') - g_{\theta}(x,y)) \right\}] .
		\eeqs
		
		We know fixing $g_{\theta}(x,y)$, the corresponding optimal $u_{\theta}^*(x,y)$ maximizing $\FLO$ is given by 
		\beq
		u_{\theta}^*(x, y) = \log \EE_{p(y')}[ \exp(g_{\theta}(x,y') - g_{\theta}(x,y))] = -\log \CE_{\theta}(x,y).
		\eeq
		This relation implies the view that $\exp^{-u_{\phi}(x,y)}$ is optimized to approximate $\CE_{\theta}(x,y)$. And to emphasize this point, we now write $\hat{\CE}_{\theta}(x,y) \triangleq e^{-u_{\phi}(x,y)}$. Assuming this approximation is sufficiently accurate ({\it i.e.}, $\CE_{\theta} \approx \hat{\CE}_{\theta}$), we have 
		\beqs
		\nabla_{\theta}\{ I_{\FLO}(u_{\phi},g_{\theta}) \} & = & - \EE_{p(x,y)}\left[ e^{-u_{\phi}(x,y)}\EE_{p(y')}[ \nabla_{\theta} \exp(g_{\theta}(x,y') - g_{\theta}(x,y))]\right] \\
		& = &  \EE_{p(x,y)}\left[ \frac{e^{-u_{\phi}(x,y)}}{\CE_{\theta}(x,y)} \nabla_{\theta} \log \CE_{\theta}(x,y)\right]\\
		& = &  \EE_{p(x,y)}\left[ \frac{\hat{\CE}_{\theta}(x,y)}{\CE_{\theta}(x,y)} \nabla_{\theta} \log \CE_{\theta}(x,y)\right] \\
		& \approx &  \EE_{p(x,y)}\left[  \nabla_{\theta} \log \CE_{\theta}(x,y) \right] \\
		& = &  \nabla_{\theta} \left\{ \EE_{p(x,y)}[\log \CE_{\theta}(x,y)] \right\} = \nabla_{\theta} \{I_{\UBA'}(g_\theta)\}.
		\eeqs
		
		While the above relation shows we can use $\FLO$ to amortize the learning of $\UBA$, one major caveat with the above formulation is that $\hat{u}(x,y)$ has to be very accurate for it to be valid. As such, one needs to solve a cumbersome nested optimization problem: update $g_{\theta}$, then optimize $u_{\phi}$ until it converges before the next update of $g_{\theta}$. Fortunately, we can show that is unnecessary: the convergence can be established under much weaker conditions, which justifies the use of simple simultaneous stochastic gradient descent for both $(\theta, \phi)$ in the optimization of $\FLO$. 
		
		\section{Proof of Proposition 2.5 ($\FLO$ Convergence under SGD)}
		
		Our proof is based on the convergence analyses of generalized stochastic gradient descent from \citep{tao2019fenchel}. We cite the main assumptions and results below for completeness. 
		
		\begin{defn}[Generalized SGD, Problem 2.1 in \citep{tao2019fenchel}] 
			Let $h(\theta;\omega), \omega \sim p(\omega)$ be an unbiased stochastic gradient estimator for objective $f(\theta)$, $\{ \eta_t > 0 \}$ is the fixed learning rate schedule, $\{\xi_t>0\}$ is the random perturbations to the learning rate. We want to solve for $\nabla f(\theta) = 0$ with the  iterative scheme
			$
			\theta_{t+1} = \theta_t + \tilde{\eta}_t \, h(\theta_t; \omega_t), 
			$
			where $\{\omega_t\}$ are iid draws and $\tilde{\eta}_t = \eta_t \xi_t$ is the randomized learning rate. 
		\end{defn}
		
		\begin{assumption}(Standard regularity conditions for Robbins-Monro stochastic approximation, Assumption D.1 \citep{tao2019fenchel}).
			\label{thm:assum}
			\vspace{-1em}
			\begin{enumerate}
				\setlength\itemsep{2pt}
				\item[$A1.$] $h(\theta)\triangleq\EE_{\omega}[h(\theta;\omega)]$ is Lipschitz continuous;
				\item[$A2.$] The ODE $\dot{\theta} = h(\theta)$ has a unique equilibrium point $\theta^*$, which is globally asymptotically stable;
				\item[$A3.$] The sequence $\{ \theta_t \}$ is bounded with probability $1$;
				\item[$A4.$] The noise sequence $\{ \omega_t \}$ is a martingale difference sequence;
				\item[$A5.$] For some finite constants $A$ and $B$ and some norm  $\| \cdot \|$ on $\BR^d$, $\EE[\| \omega_t \|^2] \leq A + B \| \theta_t \|^2$ a.s. $\forall t \geq 1$. 
			\end{enumerate}
		\end{assumption}
		% {\it Remark.} In the context of stochastic optimization, the globally asymptotic stability can be implied, for example, when $f(\theta)$ is strict convex (recall $h(\theta) = \nabla f(\theta)$). 

		\begin{prop}[Generalized stochastic approximation, Proposition 2.2 in \citep{tao2019fenchel}] 
			\label{thm:gsa}
			Under the standard regularity conditions listed in Assumption \ref{thm:assum}, we further assume $\sum_t \EE[\tilde{\eta}_t] = \infty$ and $\sum_t \EE[\tilde{\eta}_t^2] < \infty$. 
			Then $\theta_n \rightarrow \theta^*$ with probability $1$ from any initial point $\theta_0$. 
		\end{prop}
		
		\begin{assumption}(Weaker regularity conditions for generalized Robbins-Monro stochastic approximation, Assumption G.1 in \citep{tao2019fenchel}).
			\label{thm:assum_gen}
			\vspace{-1em}
			\begin{enumerate}
				\setlength\itemsep{2pt}
				\item[$B1.$] The objective function $f(\theta)$ is second-order differentiable.
				\item[$B2.$] The objective function $f(\theta)$ has a Lipschitz-continuous gradient, i.e., there exists a constant $L$ satisfying 
				\begin{equation*}
					-LI\preceq \nabla^2f(\theta)\preceq LI,
				\end{equation*}
				\item[$B3.$] The noise has a bounded variance, i.e., there exists a constant $\sigma>0$ satisfying $\mathbb{E}\left[\left\|h(\theta_t;\omega_t) - \nabla f(\theta_t)\right\|^2\right]\leq \sigma^2$.
			\end{enumerate}
		\end{assumption}

		\begin{prop}[Weaker convergence results, Proposition G.2 in \citep{tao2019fenchel}]
			\label{thm:weak_conv}
			Under the technical conditions listed in Assumption \ref{thm:assum_gen}, the SGD solution $\{ \theta_t \}_{t>0}$ updated with generalized Robbins-Monro sequence ($\tilde{\eta}_t$: $\sum_t \EE[\tilde{\eta}_t] = \infty$ and $\sum_t \EE[\tilde{\eta}_t^2] < \infty$) converges to a stationary point of $f(\theta)$ with probability $1$ (equivalently, $\mathbb{E}\left[\|\nabla f(\theta_t)\|^2\right]\rightarrow 0$ as $t\rightarrow \infty$).
		\end{prop}
		
		\begin{proof}
			Since $\hat{\CE}_{\theta_t}/\CE_{\theta_t}$ is bounded between $[a,b]$ ($0<a<b<\infty$), results follow by a direct application of Proposition \ref{thm:gsa} and Proposition \ref{thm:weak_conv}. 
		\end{proof}

		\section{Gaussian Toy Model Experiments}
		
		First, we start validating the properties and utility of the proposed $\FLO$ estimator by comparing it to competing solutions with the Gaussian toy models. Specifically, for the $2d$-D Gaussian model with correlation $\rho$, we have $X \in \BR^d$ and $Y \in \BR^d$ with covariance structure 
		\beq
		\cov[[X]_i, [X]_j] = \delta_{ij}, \cov[[Y]_i, [Y]_j] = \delta_{ij}, \cov[[X]_i, [Y]_j] = \delta_{ij}\cdot\rho
		\eeq
		This allows us to have the ground-truth MI $I(X;Y) = -\frac{d}{2} \log (1-\rho^2)$ for reference and easily tune the difficulty of the task via varying $d$ and $\rho$.
		%$(X, Y) \in \mathbb{R}^2 \times \mathbb{R}^2$. 
		
		%$\infonce$ and other baseline in 2d-dimensional

		\subsection{Choice of baselines}
		
		We choose $\TUBA$, $\NWJ$, $\infonce$ and $\alpha$-$\infonce$ as our baselines. 
		%Based on some early feedback, we also added comparison to the entropy-based $\MIND$ \citep{samo2021inductive}. 
		Note $\alpha$-$\infonce$ results are not reported in the main paper because we do not see a clear advantage via tuning $\alpha$
		%, and $\MIND$ results are not reported in the main paper for being less competitive despite our best attempt. 
		$\NWJ$ and $\infonce$ are the two most popular estimators in practice that are employed without additional hacks. $\TUBA$ is included for its close relevance to $\FLO$ ({\it i.e.}, optimizing $u(x)$ instead of $u(x,y)$, and being non-contrastive). We do not include $\DV$ here because we find $\DV$ needs excessively a large negative sample size $K$ to work. Variants like $\MINE$ are excluded for involving additional tuning parameters or hacks which complicates our analyses. The proposed $\FDV$ estimator is also excluded from our analyses for bound comparison since it includes $\hat{I}_{\DV}$ in the estimator. Note that although not suitable for MI estimation, we find $\FDV$ works quite well in representation learning settings where the optimization of MI is targeted. This is because in $\FDV$, the primal term $\hat{I}_{\DV}$ term does not participate gradient computation, so it does not yield degenerated performance as that of $\DV$. In the results reported below, we fixed $\alpha=0.8$ for better visualization. 
		
		\subsection{Experimental setups}
		
		We use the following baseline setup for all models unless otherwise specified. For the critic functions $g(x,y)$, $u(x,y)$ and $u(x)$, we use multi-layer perceptron (MLP) network construction with hidden-layers $512 \times 512$ and \texttt{ReLU} activation. For optimizer, we use Adam and set learning rate to $10^{-4}$ unless otherwise sepcified. A default batch-size of $128$ is used for training. To report the estimated MI, we use $10k$ samples and take the average. To visualize variance, we plot the decimal quantiles at $\{10\%, 20\%, \cdots, 80\%, 90\%\}$ and color code with different shades. We sample fresh data point in each iteration to avoid overfitting the data. All models are trained for $\sim 5,000$ iterations (each epoch samples $10k$ new data points, that is $78$ iterations per epoch for a total of $50$ epochs). 
		%The result is shown in Figure 1. 
		% and then get the final estimated mutual information
		
		%Secondly, in order to have a good visualization to prove the $u(x,y)$ learns the negative point-wise mutual information, we train the $u(x,y)$ in the same way as the experiment in Figure 1. Then the paired x and y are from [-3,3] by step 0.025. Plug x and y into $\hat{u}(x,y)$.  

		\subsection{PMI approximation with $u(x,y)$} 
		
		For Figure \ref{fig:contour}, we use the $2$-D Gaussian with $\rho = 0.5$ to compare the estimated $u(x,y), g(x,y)$ with the ground-truth PMI, and the contour plot is obtained with a grid resolution of $2.5 \times 10^{-2}$. This confirms our analyses that the optimized $u(x,y)$ approximates the true PMI $-\log \frac{p(x,y)}{p(x)p(y)}$.  
		
		\begin{figure}
			\centering
			\includegraphics[width=\textwidth]{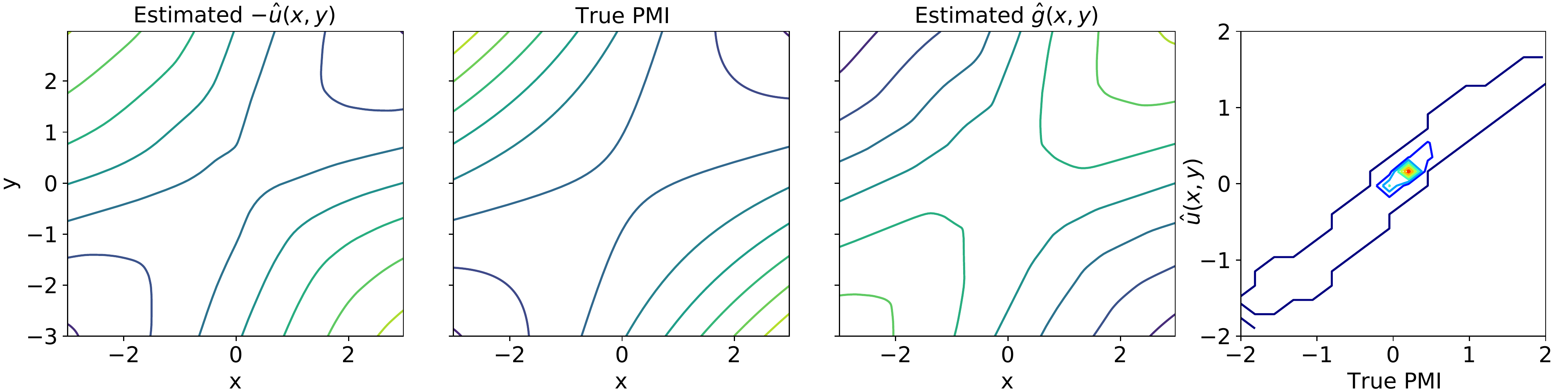}
			\caption{Comparison of estimated $u(x,y), g(x,y)$ and the ground-truth PMI $-\log \frac{p(x,y)}{p(x)p(y)}$ using the 2D Gaussian experiment. This confirms our analyses that the optimized $u(x,y)$ approximates the true PMI.}
			\label{fig:contour}
		\end{figure}
		
		\subsection{Ablation study: efficiency of parameter sharing for $g(x,y)$ and $u(x,y)$.} For the shared parameterization experiment for $\FLO$ (Figure \ref{fig:onefunc}), we used the more challenging $20$-D Gaussian with $\rho=0.5$, and trained the network with learning rate $10^{-3}$ and $10^{-4}$ respectively. We repeat the experiments for $10$ times and plot the distribution of the MI estimation trajectories. Note that we intentionally used a setup such that the MLP network architecture we used is inadequate to get a sharp estimate (both for $\FLO$ and other MI estimators), which simulates the realistic scenario that the ground-truth MI is infeasible due to architecture constraints (refer to our ablation study on the influence network capacity in Sec \ref{sec:capacity}). We observe the $\FLO$ estimator with a shared network learns faster than its separate network counterpart under both learning rates, validating the superior efficiency of parameter sharing.

		\begin{figure}
			\centering
			\includegraphics[width=.7\textwidth]{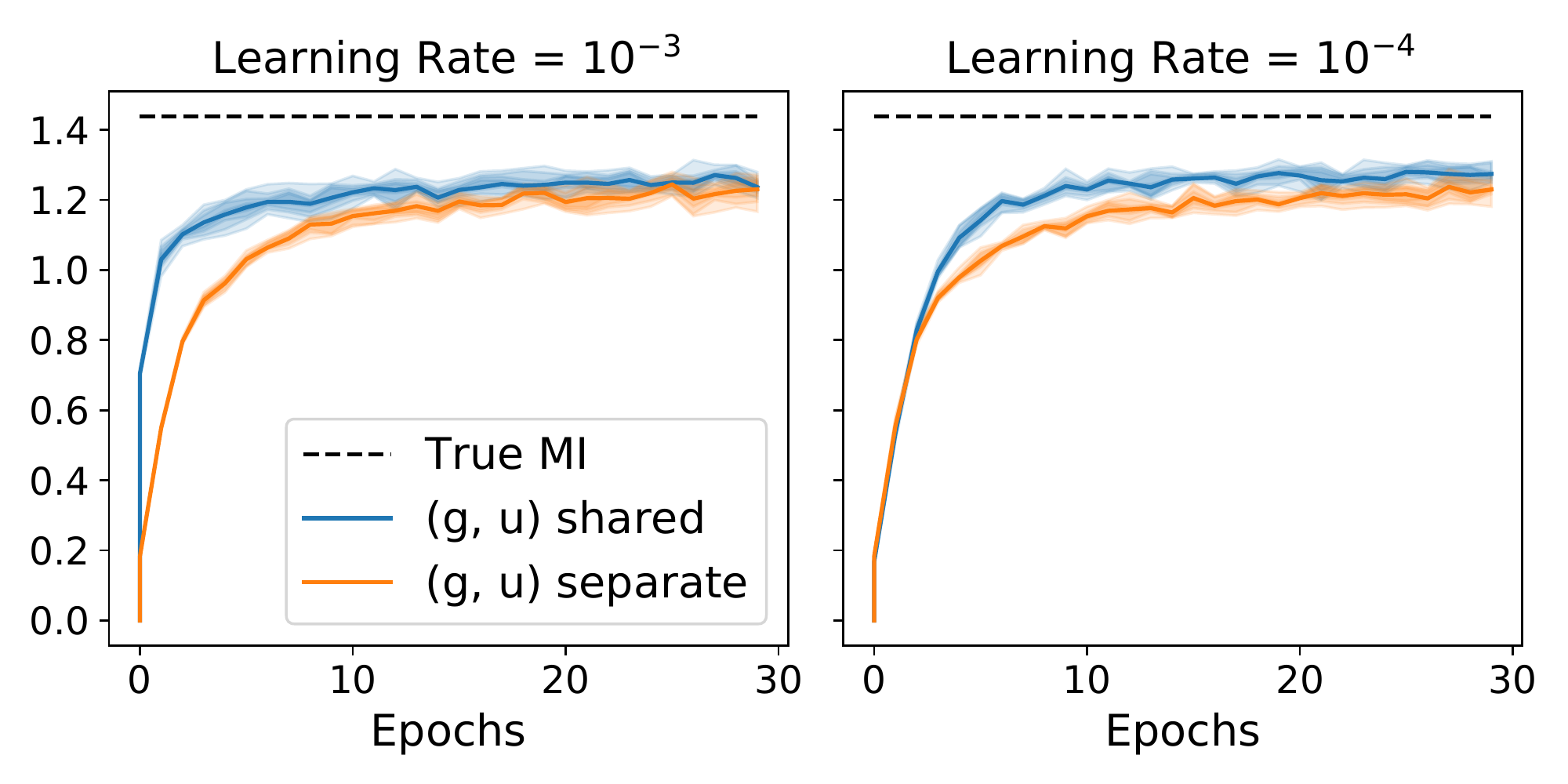}
			\vspace{-1.em}
			\caption{MI estimation with different critic parameter sharing strategies for $\FLO$: shared network and separate networks under learning rates $10^{-3}$ and $10^{-4}$ for 2-D Gaussian. Note shared parameterization not only reduced half the network size, it also learns faster.}
			\label{fig:onefunc}
		\end{figure}

		\begin{figure}[t!]
			\begin{center}
				\includegraphics[trim=0 0 5.in 0, clip, width=1.\textwidth]{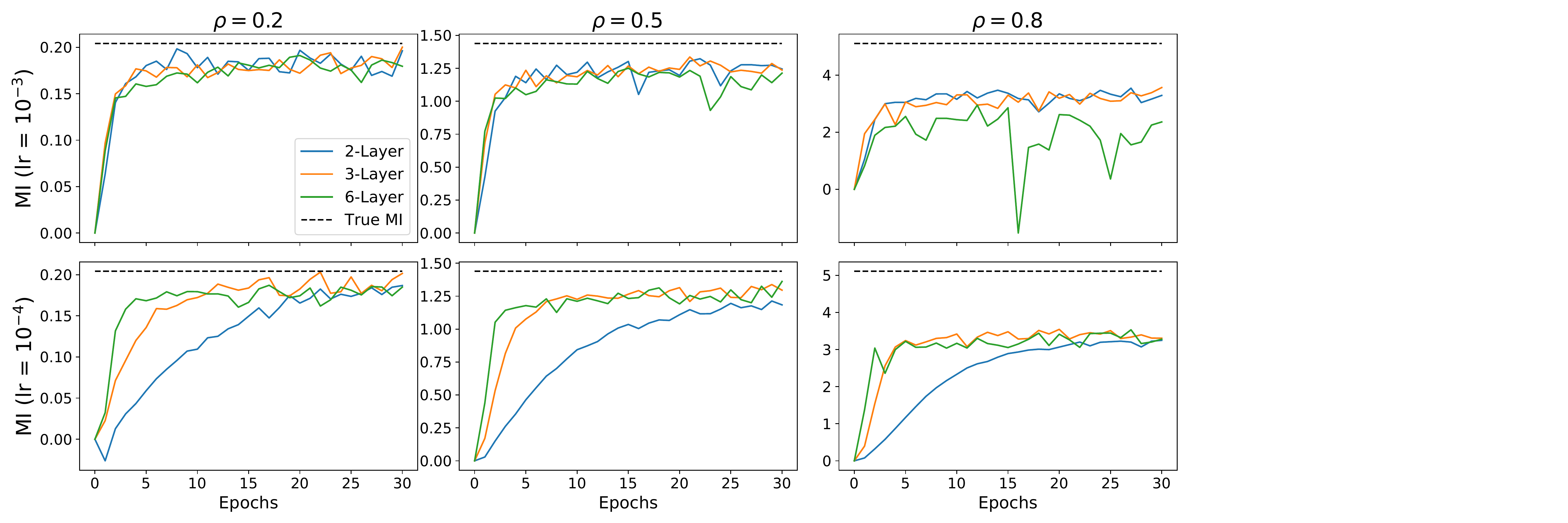}
			\end{center}
			\vspace{-1.5em}
			\caption{Abaltion study for network complexity with $\FLO$. More complex networks lead to faster convergence and better MI estimates. However, the stability is more sensitive to learning rate with a larger neural network. \label{fig:layer_cmp}}
			\vspace{-1.5em}
		\end{figure}

		\subsection{Ablation study: network capacity and MI estimation accuracy} 
		\label{sec:capacity}
		
		We further investigate how the neural network learning capacity affect MI estimation. In Figure \ref{fig:layer_cmp} we compare the training dynamics of the $\FLO$ estimator with $L$-layer neural networks, where $L \in \{2,3,6\}$ and each hidden-layer has $512$-units. A deeper network is generally considered to be more expressive. We see that using larger networks in general converge faster in terms of training iterations, and also obtain better MI estimates. However, more complex networks imply more computation per iteration, and it can be less stable when trained with larger learning rates. 
		
		\begin{figure*}[t!]
			% \centering
			% \begin{minipage}{.3\textwidth}
				% % \vspace{.2em}
				%   \centering
				% %   \begin{figure}[t!]
					% \begin{center}
						% \includegraphics[width=1\textwidth]{figures/toy/adult.pdf}
						% \end{center}
					% \vspace{-1.5em}
					% \captionof{figure}{Fair Learning Result.  \label{fig:adult}}
					%\vspace{-1.em}
					% [This is a unfair comparison as MI estimators other than $\infonce$ is single-sample estimator. To be fixed in next iteration.]
					% \end{figure}
				%   \includegraphics[width=1.\textwidth]{figures/results/ess}
				%   \vspace{-1.5em}
				%   \captionof{figure}{Effective sample size ({\it c.f.} Figure \ref{fig:mi_opt}). \label{fig:ess}}
				% \end{minipage}
			% \hspace{2pt}
			\centering
			\begin{minipage}{.48\textwidth}
				\centering
				\includegraphics[width=1\textwidth]{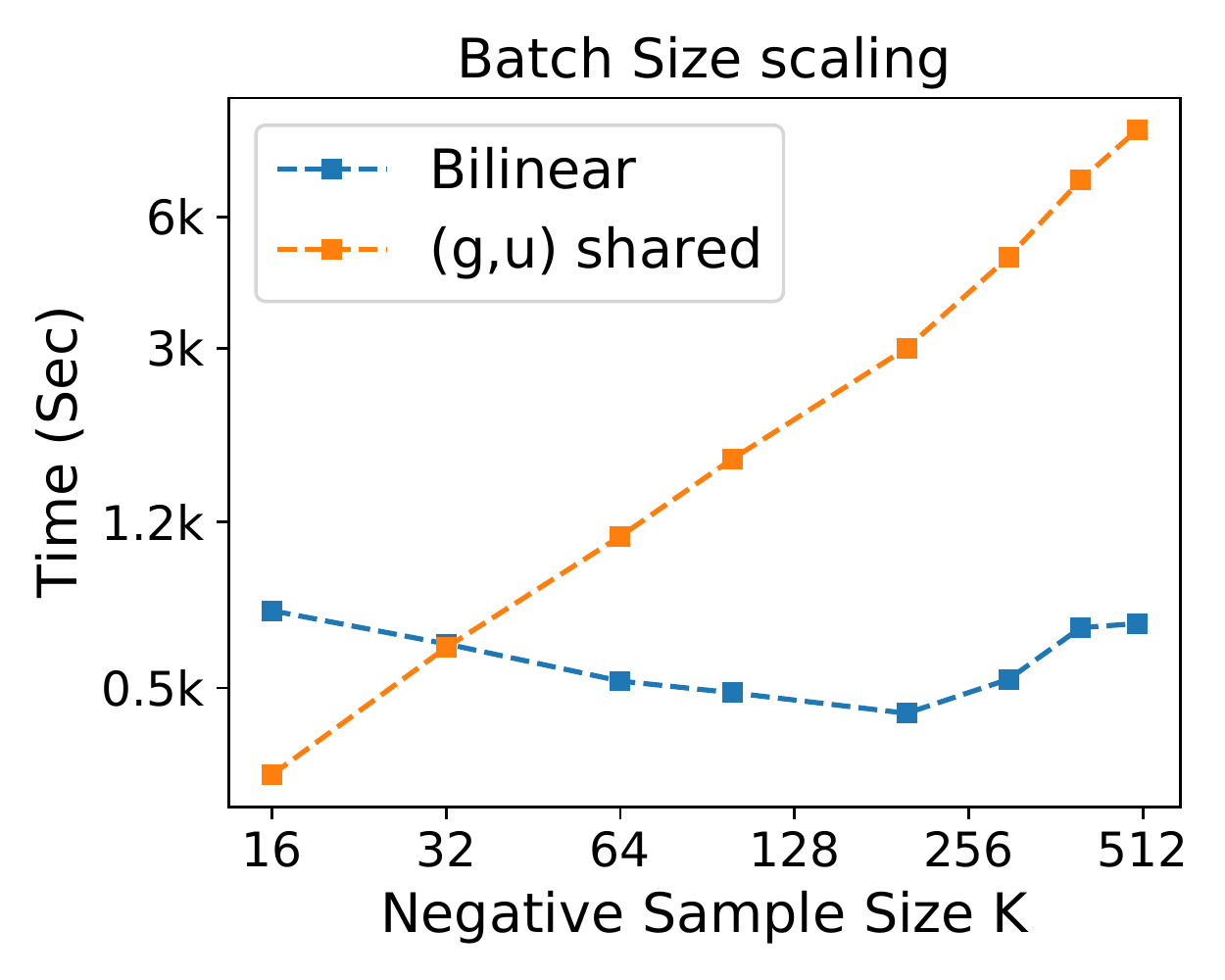}
				\vspace{-2.em}
				\captionof{figure}{Comparison of computation time of the shared MLP critic and the bi-linear critic. Overall the bilinear implementation is more efficient than the shared MLP. $\FLO$'s initial drop in computation time with growing negative sample size is due to better exploitation of parallel computation. \label{fig:bilinear}}
				%   \label{fig:test2}
			\end{minipage}
			\hspace{2pt}
			\begin{minipage}{.48\textwidth}
				\vspace{.2em}
				\centering
				\includegraphics[width=1\textwidth]{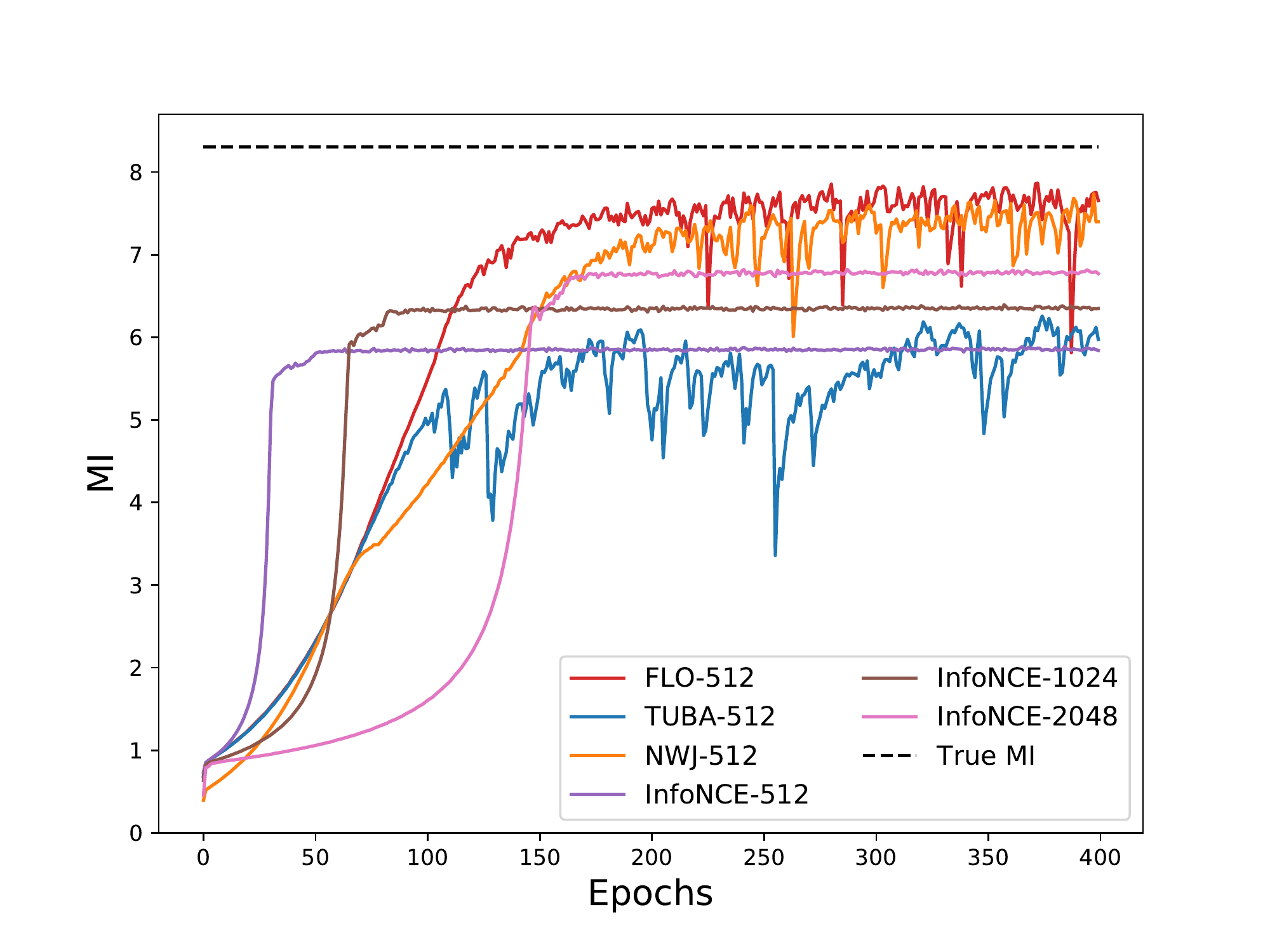}
				\vspace{-1.5em}
				\captionof{figure}{Comparison of learning dynamics with 20-D Gaussian at $\rho=0.9$. We used bi-linear critics for all bounds. Note $\infonce$ enjoys stable learning, and its convergence is fast in the small-sample regime but slow in the large-sample regime. In all cases $\infonce$ suffers form large biases. $\NWJ$ is more accurate but it learns slower. In contrast, our $\FLO$ learns fast and stably.  
					\label{fig:milearn}}
				%   \label{fig:test2}
			\end{minipage}
			\vspace{-1.em}
		\end{figure*}
		
		\subsection{Ablation study: Bi-linear critics and scaling} 
		We setup the {\it bi-linear} critic experiment as follows. For the naive baseline $\FLO$, we use the shared-network architecture for $g(x,y)$ and $u(x,y)$, and use the in-batch shuffling to create the desired number of negative samples ($\FLO$-\texttt{shuff}). For $\FLO$-$\texttt{BiL}$, we adopt the following implementation: feature encoders $h(x), \tilde{h}(y)$ are respectively modeled with three layer MLP with $512$-unit hidden layers and \texttt{ReLU} activations, and we set the output dimension to $512$. Then we concatenate the feature representation to $z = [h(x), \tilde{h}(y)]$ and fed it to the $u(x,y)$ network, which is a two-layer $128$-unit MLP. Note that is merely a convenient modeling choice and can be further optimized for efficiency. Each epoch containing $10k$ samples, and $\FLO$-\texttt{shuff} is trained with fixed batch-size. For $\FLO$-\texttt{BiL}, it is trained with batch-size set to the negative sample-size desired, because all in-batch data are served as negatives. We use the same learning rate $10^{-4}$ for both cases, and this puts large-batch training at disadvantage, as fewer iterations are executed. To compensate for this, we use $T(K) = (\frac{K}{K_0})^{\frac{1}{2}} \cdot T_0$ to set the total number of iterations for $\FLO$-\texttt{BiL}, where $(T_0, K_0)$ are respectively the baseline training iteration and negative sample size used by $\FLO$-\texttt{shuff}, and the number of negative sample K are $\{10, 50, 100, 150, 200, 250, 300, 350, 400, 450, 500\}$. We are mostly interested in computation efficiency here so we do not compare the bound. In Figure \ref{fig:bilinear}, we see the cost for training $\FLO$-\texttt{shuff} grows linearly as expected. For $\FLO$-\texttt{BiL}, a U-shape cost curve is observed. This is because bilinear implementation has three networks total, while the shared MLP only have one network. This implies more computations when the batch size is small, however, as the batch size grows, the computation overhead is amortized by better parallelism employed with the bilinear strategy, thus increasing overall efficiency until the device capacity has been reached. This explains the initial drop in cost, followed by the anticipated square-root growth.

		\subsection{Comparison of learning dynamics for different variational MI bounds} 
		
		In Figure {\ref{fig:milearn}}, we show the learning dynamics of competing estimators for the 20-D Gaussian when $\rho=0.9$. We can find $\FLO$ achieves the best accuracy, it also learns fast and stably. $\infonce$ learns very stably, yet its learning efficiency varies significantly in small-batch and large-batch setups. 
		
		% {\bf Cross-view representation learning experiment.} In this experiment, we use the training split ({\it i.e.}, $60$k) to train the cross-view representation and the prediction model based on the concatenated cross-view features. For \texttt{CCA} extraction, we use the \texttt{scikit-learn.cross\_decomposition.CCA} implementation with default settings. For all other MI-based solutions, we use the multi-sample estimators and adopt the bi-linear critic implementation as described in the bi-linear experiment for maximal efficiency. The prediction model is given by a three-layer MLP of our standard configuration, and trained on the extracted cross-view features for $50$ epochs with learning rate $10^{-3}$. 
		%To extract  \texttt{CCA} rep
		% \subsection{Additional analyses and results}

		% In Figure 4, all methods applied to $20$D Gaussian ({\it i.e.}, that is $10$ independent pairs of correlated $x$ and $y$). We repeat in-batch shuffling $K$-times for the $y$-dimensions to get the desired number of negative samples. 
		
		% We also compare the single-sample estimators for $\NWJ$, $\TUBA$ and $\FLO$ to their multi-sample $\infonce$-based counterparts in Figure \ref{fig:single}, which is the comparison made by some prior studies. In this setting, the variance single-sample estimators' variances are considerably larger, which explains their less favorable performance. 

		\begin{figure}[t!]
			\begin{center}
				\includegraphics[height=2in]{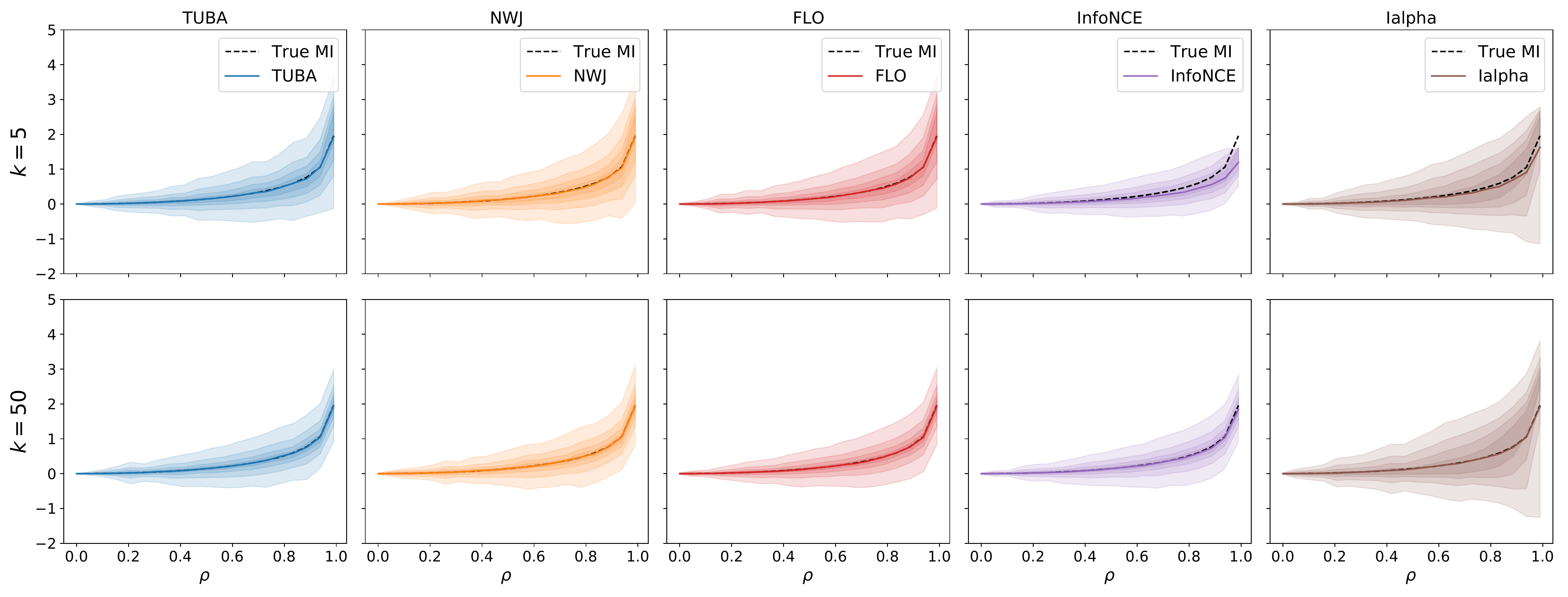}
			\end{center}
			\vspace{-1.5em}
			\caption{Bias variance plot for the popular MI bounds with the $2$-D Gaussians. In this simpler case, $\TUBA$, $\NWJ$ and $\FLO$ all give sharp estimate at $K=5$. $\alpha$-$\infonce$ gives worst variance profile. The reason is that because $\alpha$-$\infonce$ interpolates between the low-variance multi-sample $\infonce$ and high-variance single-sample $\NWJ$ (see Figure \ref{fig:cmp_var_single}), and in this case the variance from $\NWJ$ dominates.  \label{fig:cmp_var_1d_cubic}}
			\vspace{-1.em}
			% [This is a unfair comparison as MI estimators other than $\infonce$ is single-sample estimator. To be fixed in next iteration.]
		\end{figure}

		\subsection{Comprehensive analyses of bias-variance trade-offs}
		
		To supplement our results in the main paper, here we provide additional bias-variance plots for different MI estimators under various settings. In Figure \ref{fig:cmp_var_1d_cubic} we show the bias-variance plot of MI estimates for $2$-D Gaussians. In this case, the network used are sufficiently comprehensive so sharp estimate is attainable. In all cases the estimation variance grows with MI value, which is consistent with the theoretical prediction that for tight estimators, the estimation variance grows exponential with MI \citep{mcallester2018formal}. In such cases, the argument for $\infonce$'s low-variance profile no longer holds: it is actually performing sub-optimally. For complex real applications, the negative sample size used might not provide an adequate estimate of ground-truth MI ({\it i.e.}, the $\log K$ cap), and that is when $\infonce$'s low-variance profile actually helps. We also notice that, when the MI estimate is not exactly tight, but very close to the true value, the variance dropped considerably. This might provide alternative explanation (and opportunity) for the development near-optimal MI estimation theories, which is not covered in existing literature.

		% \begin{figure}[t!]
			% \begin{center}
				% \includegraphics[height=2in]{figures/toy/transform3_bounds.pdf}
				% \end{center}
			% \vspace{-1.5em}
			% \caption{Bias variance plot for the popular MI bounds with the $10$-D Gaussians with  $(R y)^3$, where $R$ is a rotation matrix. Estimates are less tight compared with the original $y$. \label{fig:cmp_var_10d_cubic}}
			% \vspace{-1.em}
			% \end{figure}
		
		\begin{figure}[t!]
			\begin{center}
				\includegraphics[width=1.\textwidth]{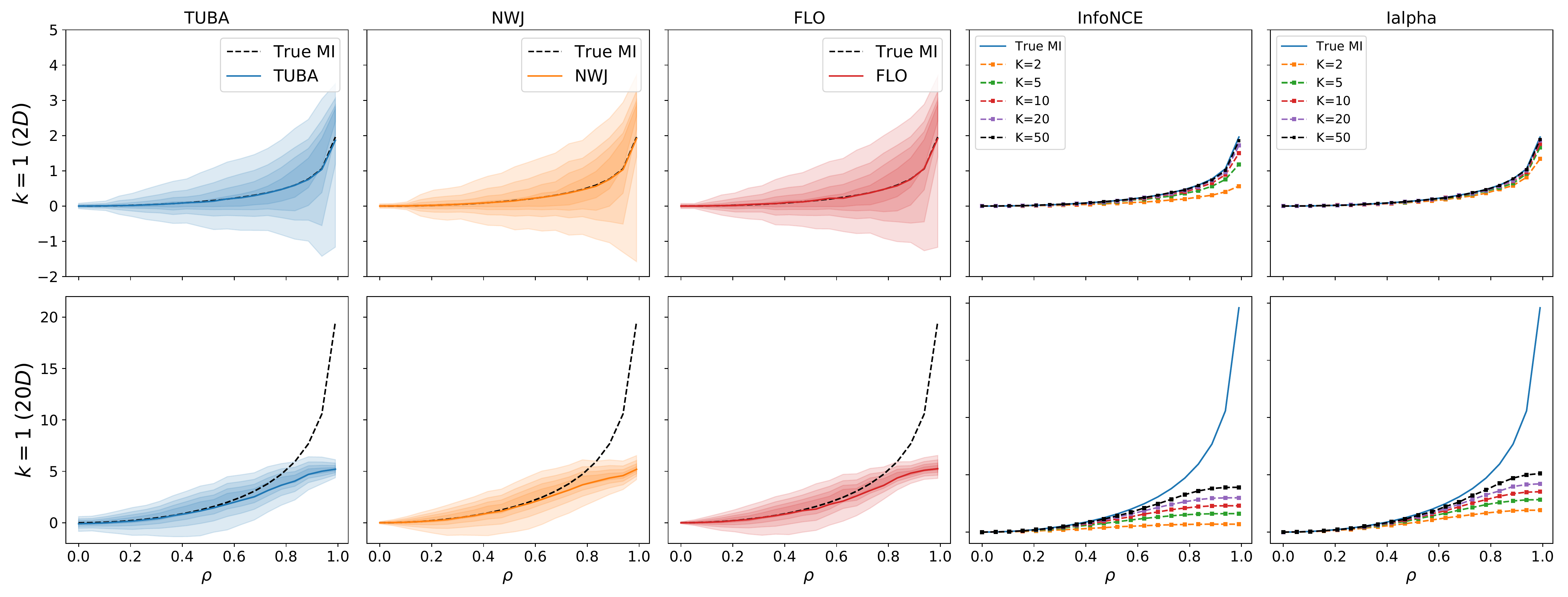}
			\end{center}
			\vspace{-1.5em}
			\caption{Bias variance plot for the popular MI bounds with the $2$-D (upper panel) and $20$-D (lower panel) Gaussians. Single-sample estimator of $\TUBA$, $\NWJ$ and $\FLO$ ({\it i.e.}, $K=1$) are compared to the multi-sample estimators of $\infonce$ and $\alpha$-$\infonce$. \label{fig:cmp_var_single}}
			\vspace{-1.em}
		\end{figure}
		
		% In Figure \ref{fig:cmp_var_10d_cubic}, we visualize the estimates for $20$-D Gaussian with $y$ randomly rotated and then apssed through a cubic function. While in theory the MI should not be affected, as all the transformations applied are invertible; in practice, however, this is a more challenging problem for empirical estimation. The estimated MI bounds are loser compared to those with original $y$. 

		We also tried the single-sample estimators for $\NWJ$, $\TUBA$ and $\FLO$ to their multi-sample $\infonce$-based counterparts (Figure \ref{fig:cmp_var_single}), which is the comparison made by some of the prior studies (Note we do not apply Bilinear tric here, thus FLO seems similar to other methods). In this setting, the variance single-sample estimators' variances are considerably larger, which explains their less favorable performance. Note that contradictory to theoretical predictions, a larger negative sample size does make $\NWJ$, $\TUBA$ and $\FLO$ tighter empirically, although the gains are much lesser compare to that of $\infonce$  (partly because these three estimators are already fairly tight relative to $\infonce$). This might be explained by a better optimization landscape due to reduced estimation variance. We conjecture that for multi-sample $\NWJ$, $\TUBA$ and $\FLO$, the performance in empirical applications such as self-supervised learning should be competitive to that of $\infonce$, which has never been reported in literature.

		%Note that there is no discernible performance gap between these two strategies (see Figure \ref{fig:nogap}).

		\begin{figure}[t!]
			\centering
			\begin{minipage}{.35\textwidth}
				% \vspace{.2em}
				\centering
				%   \begin{figure}[t!]
					\begin{center}{
							\includegraphics[width=1.\textwidth]{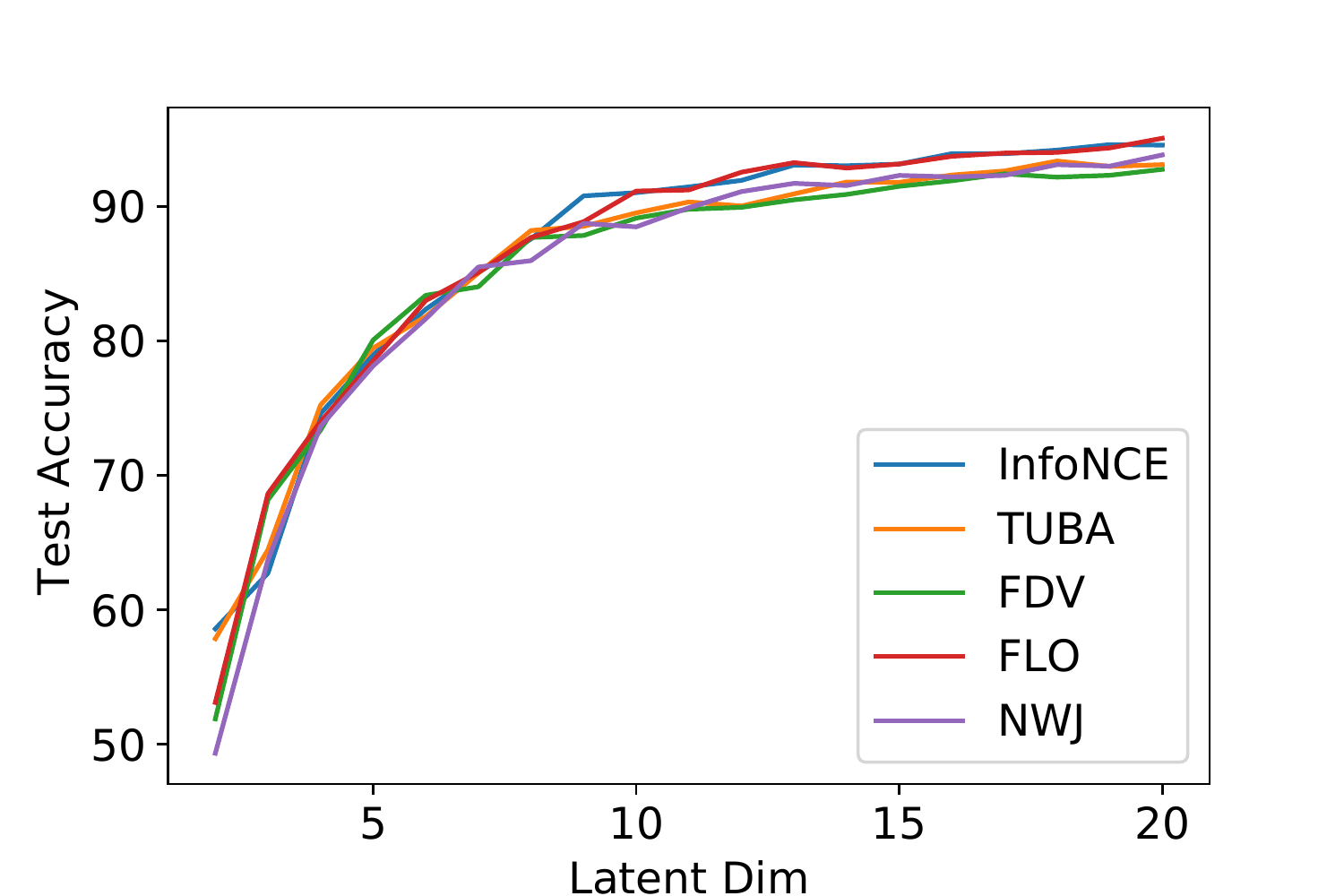}
						}
						\vspace{-1.5em}
						\captionof{figure}{\small Extended results for the cross-view representation learning. $\FDV$ works best for smaller dimensions ($\approx 5$), and for higher dimensions ($>10$) $\FLO$ and $\infonce$ give the best results. \label{fig:latentdim}}
					\end{center}
					% \vspace{-1.em}
				\end{minipage}
				%  \hspace{2pt}
				\scalebox{.95}{
					\begin{minipage}{.6\textwidth}
						\vspace{-10pt}
						\begin{center}{
								\includegraphics[width=1.\textwidth]{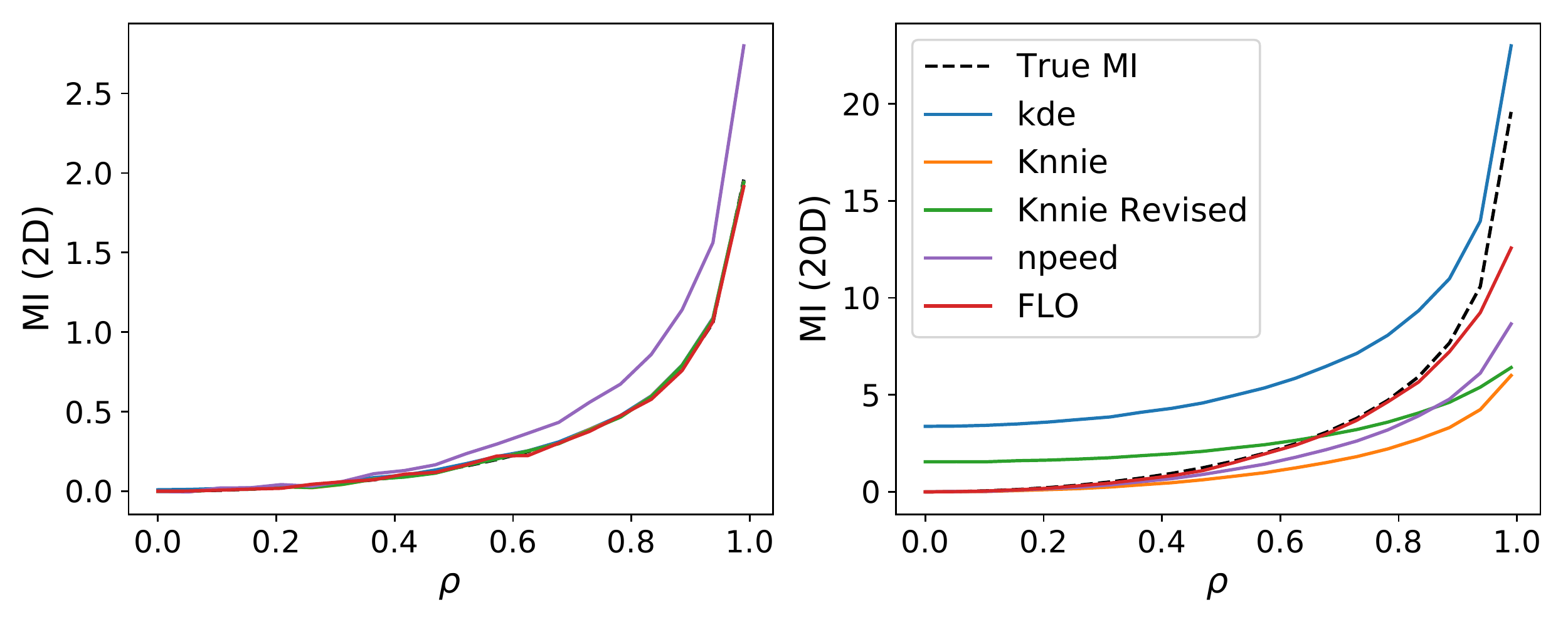}
							}
							\vspace{-1.5em}
							\captionof{figure}{\small Comparison to classical MI estimators. (left) Easy 2D Gaussian, all models perform similarly. (right) Challenging 20D Gaussian, where $\FLO$ shows better overall accuracy. Note that the kde accuracy in the high-dimensional setting is mis-judged, as it is well-known kernel-based density estimator scale poorly in high-dimensions. \label{fig:kde}}
						\end{center}
					\end{minipage}
				}
				\vspace{-1.em}
			\end{figure}

			% % \begin{wrapfigure}[10]{R}{0.48\textwidth}
				% % % \vspace{-3.5em}
				% % \scalebox{1.}{
					% % %\hspace{-.5em}
					% % \begin{minipage}{.5\textwidth}
						% \begin{figure}[H]
							% \begin{center}{
									% \includegraphics[width=.5\textwidth]{figures/toy/latentdim.pdf}
									% }
								% \vspace{-1em}
								% \caption{\small Extended results for the cross-view representation learning. We report the prediction accuracy for latent dimension from $2$ to $20$. $\FDV$ works best for smaller dimensions ($\approx 5$), and for higher dimensions ($>10$) $\FLO$ and $\infonce$ give the best results. \label{fig:latentdim}}
								% \end{center}
							% % \vspace{-1.5em}
							% \end{figure}
						% % \end{minipage}
					% % }
				% % \end{wrapfigure}
			
			% {\bf Cross-view representation learning (extended analyses).} 
			% \begin{figure}[H]
				% \begin{center}{
						% \includegraphics[width=.5\textwidth]{figures/toy/kde.pdf}
						% }
					% \vspace{-1em}
					% \caption{\small Extended results for the cross-view representation learning. We report the prediction accuracy for latent dimension from $2$ to $20$. $\FDV$ works best for smaller dimensions ($\approx 5$), and for higher dimensions ($>10$) $\FLO$ and $\infonce$ give the best results. \label{fig:kde}}
					% \end{center}
				% % \vspace{-1.5em}
				% \end{figure}

			\section{Cross-view Representation Learning (Extended Analyses)}
			
			In addition to the results reported in the paper, we investigate how different latent dimension affect the results of the cross-view representation learning. We vary the latent dimension number from $d=2$ to $d=20$, and plot label prediction accuracy for the corresponding latent representations in Figure \ref{fig:latentdim}. The same setup for the bi-linear experiment is used for the MI estimation (for all MI estimators), where the images are flattened to be fed to the MLPs. The representations are trained for $50$ epochs and the prediction model is trained for $50$ epochs. We also trained the model for another $50$ epochs and the conclusions are similar. We see that $\FDV$ works well for lower dimensions ({\it e.g.}, $d\approx 5$), and what works better for higher dimensions ($d>10$) are $\FLO$ and $\infonce$. 
			
			\section{Comparison with Classical MI Estimators}

			%\begin{figure}
			%\begin{center}
			%    \includegraphics[width = 0.4\textwidth]{figures/gauss/mind.pdf}
			%\end{center}
			%\vspace{-2em}
			%\caption{Comparison of different mutual information estimators under the 20-D Gaussian setup. $\MIND$ suffered from very large bias in the mid-to-high MI range. }
			%\label{fig:mind}
			%\end{figure}
			
			We also compare our $\FLO$ estimator to the classical MI estimators in Figure \ref{fig:kde}. The following implementations of baseline estimators for multi-dimensional data are considered: ($i$) {\it KDE}: we use kernel density estimators to approximate the joint and marginal likelihoods, then compute MI by definition; ($ii$) {\it NPEET} \footnote{\url{https://github.com/gregversteeg/NPEET}}, a variant of Kraskov's $K$-nearest neighbour (KNN) estimator \citep{kraskov2004estimating, ver2013information}; ($iii$) {\it KNNIE} \footnote{\url{https://github.com/wgao9/knnie}}, the original KNN-estimator and its revised variant \cite{gao2018demystifying}. These models are tested on $2$-D and $20$-D Gaussians with varying strength of correlation, with their hyper-parameters tuned for best performance. Note that the notation of ``best fit'' is a little bit subjective, as we will fix the hyper-parameter for all dependency strength, and what works better for weak dependency might necessarily not work well for strong dependency. We choose the parameter whose result is visually most compelling. In addition to the above, we have also considered other estimators such as maximal-likelihood density ratio \footnote{\url{https://github.com/leomuckley/maximum-likelihood-mutual-information}} \citep{suzuki2008approximating} and KNN with local non-uniformity correction \footnote{\url{https://github.com/BiuBiuBiLL/NPEET_LNC}}. However, these models either do not have a publicly available multi-dimensional implementation, or their codes do not produce reasonable results \footnote{These are third-party python implementations, so BUGs are highly likely.}.

			\section{Comparison to Parametric Variational Estimators and Bounds Targeting Alternative Information Metrics}
			\label{sec:alt-appendix}

			Parametric variational estimators are typically associated with upper bound of MI \citep{cheng2020club,poole2019variational}. Inspired by multi-sample variational bounds for likelihood estimation, \citep{brekelmans2021improving} derived a generic family of importance-weighted MI bounds that are provably tighter. These bounds usually require the additional knowledge of likelihood, and consequently they can not be directly used for data-driven MI estimations. On the other hand, these models do not suffer from the exponential scaling of variance suffered by non-parametric MI estimators. Note that MI is not the only measure to assess association between two random variables, some alternatives can potentially do better for specific applications. Examples include $\CV$ information \citep{xu2020theory}, R\'enyi information \citep{lee2022r}, and the spectral information \citep{haochen2021provable}. 
			% \citep{cheng2020club, brekelmans2021improving} and alternative information metrics \citep{xu2020theory}.
			
			\section{Regression with Sensitive Attributes (Fair Learning) Experiments}
			
			\subsection{Introduction to fair machine learning}
			Nowadays consequential decisions impacting people's lives have been increasingly made by machine learning models. Such examples include loan approval, school admission, and advertising campaign, amongst others. While automated decision making has greatly simplified our lives, concerns have been raised on (inadvertently) echoing, even amplifying societal biases. Specially, algorithms are vulnerable in inheriting discrimination from the training data and passed on such prejudices in their predictions. 
			
			To address the growing need for mitigating algorithmic biases, research has been devoted in this direction under the name fair machine learning. While discrimination can take many definitions that are not necessarily compatible, in this study we focus on the most widely recognized  criteria {\it Demographic Parity} (DP), as defined below
			
			\begin{defn}[Demographic Parity, \citep{dwork2012fairness}]
				The absolute difference between the selection rates of a decision rule $\hat{y}$ of two demographic groups defined by sensitive attribute $s$, {\it i.e.}, 
				\beq
				\DP(\hat{Y}, S) = \left| \PP(\hat{Y}=1|S=1) - \PP(\hat{Y}=1|S=0) \right|.
				\eeq
				With multiple demographic groups, it is the maximal disparities between any two groups:
				\beq
				\DP(\hat{Y}, S) = \max_{s\neq s'}\left| \PP(\hat{Y}=1|S=s) - \PP(\hat{Y}|S=s') \right|.
				\eeq
			\end{defn}
			
			\subsection{Experiment details and analyses}
			
			To scrub the sensitive information from data, we consider the {\it in-processing} setup 
			\beq
			\CL = \texttt{Loss}(\underbrace{\text{Predictor}(\text{Encoder}(x_i)), y_i}_{\text{Primary loss}}) + \lambda \underbrace{I(s_i, \text{Encoder}(x_i))}_{\text{Debiasing}}.
			\eeq
			By regularizing model training with the violation of specified fairness metric $\Delta(\hat{y}, s)$, fairness is enforced during model training. In practice, people recognize that appealing to fairness sometimes cost the utility of an algorithm ({\it e.g.}, prediction accuracy) \citep{hardt2016equality}. So most applications seek to find their own sweet points on the fairness-utility curve. In our example, it is the {\it $\DP$-error} curve. A fair-learning algorithm is consider good if it has lower error at the same level of DP control. 
			
			%Specifically, construct out-come agnostic representation $z = \Phi(x,s)$ such that $\BM(Z, S)$ is minimized, where $\BM$ is some measure of statistical dependency. Including {\it fair representation learning} and {\it data repair}. 
			
			%so one must be explicit about the type of fairness to be enforced. There are two major streams of fairness, individual fairness and group fairness. For individual fairness, two individuals with similar characteristics are supposed to receive the same algorithm recommendation, regardless of the sensitive group memberships. For group fairness, the summary of statistics of the algorithm recommendation should be similar across sensitive groups. We mainly focus on group fairness. 
			
			In this experiment, we compare our MI-based fair learning solutions to the state-of-the-art methods. {\it Adversarial debiasing} tries to maximize the prediction accuracy for while minimize the prediction accuracy for sensitivity group ID \citep{zhang2018mitigating}. We use the implementation from \texttt{AIF360}\footnote{\url{https://github.com/Trusted-AI/AIF360}} package \citep{aif360-oct-2018}.  FERMI is a density-based estimator for the {\it exponential R\'enyi mutual information} $\texttt{ERMI}\triangleq \EE_{p(x,y)}[\frac{p(x,y)}{p(x)p(y)}]$, and we use the official codebase. For evaluation, we consider the {\it adult} data set from UCI data repository \citep{asuncion2007uci}, which is the 1994 census data with $30$k samples in the train set and $15$k samples in the test set. The target task is to predict whether the income exceeds \$50k, where gender is used as protected attribute. Note that we use this binary sensitive attribute data just to demonstrate our solution is competitive to existing solutions, where mostly developed for binary sensitive groups. Our solution can extend to more general settings where the sensitive attribute is continuous and high-dimensional. 
			
			We implement our fair regression model as follows. To embrace data uncertainty, we consider latent variable model $p_\theta(y,x,z) = p_\theta(y|z)p_\theta(x|z)p(z)$, where $v=\{x,y\}$ are the observed predictor and labels.
			Under the variational inference framework \citep{kingma2014auto}, we write the $\ELBO(v;p_\theta(v,z),q_\phi(z|v))$ as
			\begin{align}
				% \begin{aligned}
					\EE_{Z\sim q_\phi(z|v)}[\log p_\theta(y|Z)] + \text{\st{$\EE_{Z\sim q_\phi(z|v)}[\log p_\theta(x|Z)]$}} - \beta \text{KL}(q_\phi(z|v)\parallel p(z)) \label{eq:elbonew}
				\end{align}
				$p(z)$ is modeled with standard Gaussian, and the approximate posterior $q_\phi(z|v)$ is modeled by a neural network parameterizing the mean and variance of the latents (we use the standard mean-field approximation so cross-covariance is set to zero), and $\beta$ is a hyperparameter controlling the relative contribution of the KL term to the objective.
				Note that unlike in the standard ELBO we have dropped the term $\EE_{Z\sim q_\phi(z|v)}[\log p_\theta(x|Z)]$ because we are not interested in modeling the covariates. Note this coincides with the {\it variational information bottleneck} (VIB) formulation \citep{alemi2016deep}. Additionally, the posterior $q_\phi(z|v)$ will not be conditioned on $y$, but only on $x$, because in practice, the labels $y$ are not available at inference time. All networks  used here are standard three-layer MLP with $512$ hidden-units.
				
				% and {\it FERMI} \citep{lowy2021fermi} on the {\it adult} dataset
				
				For Figure \ref{fig:adult}, we note that the adversarial de-biasing actually crashed in the DP range $[0.1, 0.18]$, so the results have to be removed. Since interpolation is used to connect different data points, it makes the adversarial scheme look good in this DP range, which is not the case. FERMI also gave unstable estimation in the DP range $[0.1, 0.18]$. Among the MI-based solutions, $\NWJ$ was most unstable. Performance-wise, $\infonce$, $\TUBA$ and $\FDV$ are mostly tied, with the latter two slightly better in the ``more fair'' solutions ({\it i.e.}, at the low DP end).
				
				% \subsection{Complex sensitive attributes}
				
				\section{Self-supervised Learning}
				
				%We further consider the self-supervised learning setup. In particular, we use MNIST data and a ResNet-10 network for feature extraction. To optimize representation without labels, we optimize the MI between two views of the digits: randomly rotated (0-30 degree) and resize \& cropped (scale 0.5-1.0). We train with learning rate $10^{-4}$ for $50$-epochs, and report the performance by training a linear classifier using the learned representation. For \texttt{CCA} extraction, we use the \texttt{scikit-learn.cross\_decomposition.CCA} implementation with default settings. For all other MI-based solutions, we use the multi-sample estimators and adopt the bi-linear critic implementation as described in the bi-linear experiment for maximal efficiency. The prediction model is given by a three-layer MLP of our standard configuration, and trained on the extracted cross-view features for $50$ epochs with learning rate See Table \ref{tab:ssl} for results. In this experiment, $\FDV$ worked best.
				
				Our codebase is modified from a public \texttt{PyTorch} implementation\footnote{\url{https://github.com/sthalles/SimCLR}}. Specifically, we train $256$-dimensional feature representations by maximizing the self-MI between two random views of data, and report the test set classification accuracy using a linear classifier trained to convergence. We report performance based on \texttt{ResNet-50}, and some of the learning dynamics analyses are based on \texttt{ResNet-18} for reasons of memory constraints. Hyper-parameters are adapted from the original $\SimCLR$ paper. For the large-batch scaling experiment, we first grid-search the best learning rate for the base batch-size, then grow the learning rate linearly with batch-size. 
				
				% \begin{table}[t!]
					% \caption{Self-supervised learning on MNIST with ResNet-50. \label{tab:ssl}}
					% \begin{center}
						%     \begin{tabular}{cccccc}
							% \toprule
							% Model    & $\NWJ$   & $\TUBA$  & $\infonce$ & $\FLO$   & $\FDV$   \\ 
							% \midrule
							% Accuracy & 95.37 & 96.28 & 96.90   & {\bf 97.12} & 95.43 \\ \bottomrule
							% \end{tabular}
						% \end{center}
					
					% \end{table}
				
				\begin{figure}[t!]
					\centering
					\begin{minipage}{.35\textwidth}
						% \vspace{.2em}
						\centering
						%   \begin{figure}[t!]
							\begin{center}
								\includegraphics[width=1\textwidth]{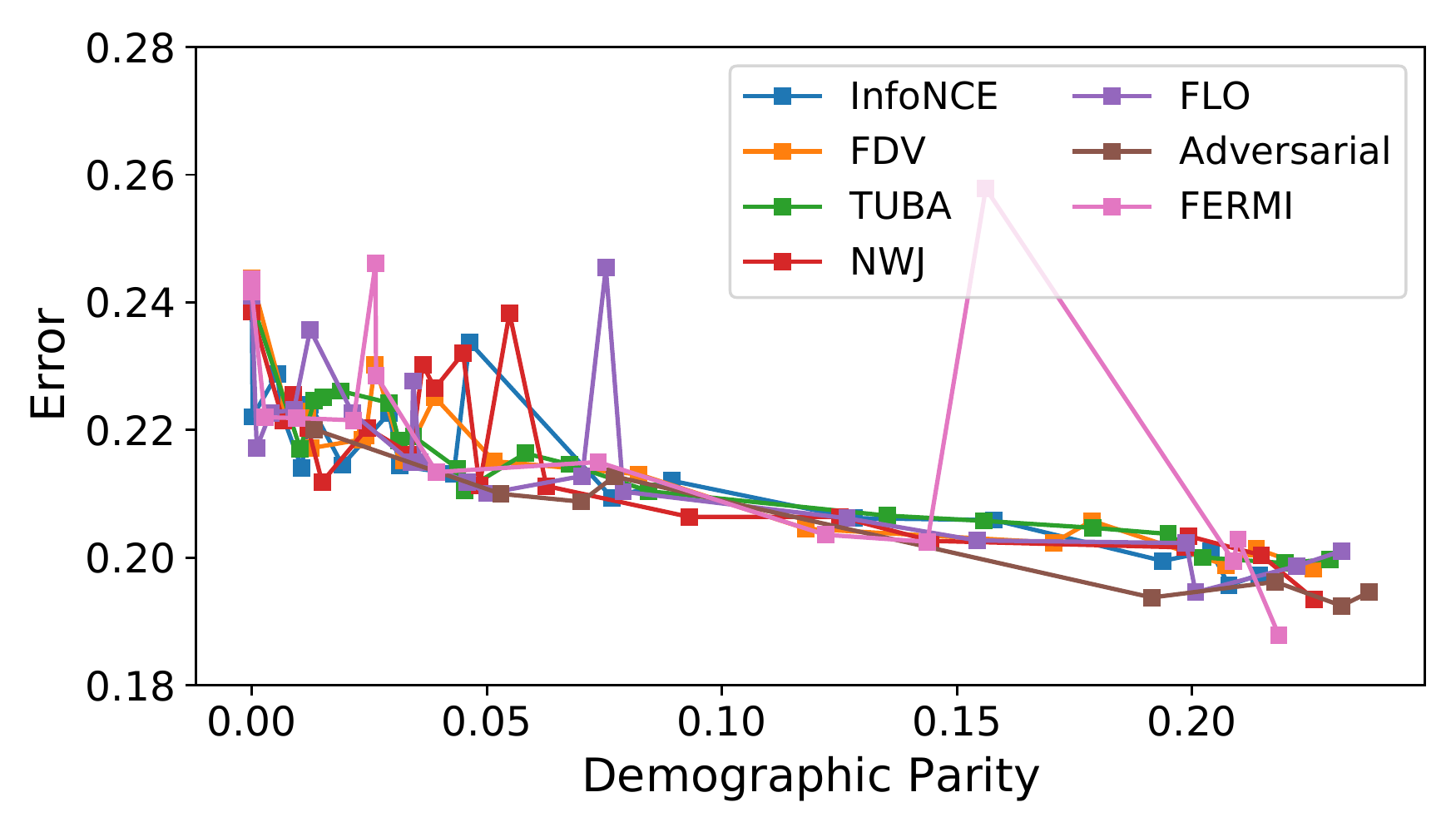}
							\end{center}
							\vspace{-1.5em}
							\captionof{figure}{Fair Learning Result.  \label{fig:adult}}
							%\vspace{-1.em}
							% [This is a unfair comparison as MI estimators other than $\infonce$ is single-sample estimator. To be fixed in next iteration.]
							% \end{figure}
						%   \includegraphics[width=1.\textwidth]{figures/results/ess}
						%   \vspace{-1.5em}
						%   \captionof{figure}{Effective sample size ({\it c.f.} Figure \ref{fig:mi_opt}). \label{fig:ess}}
					\end{minipage}
					%  \hspace{2pt}
					\scalebox{.95}{
						\begin{minipage}{.6\textwidth}
							\vspace{-1.5em}
							\centering
							\setlength{\tabcolsep}{2pt}
							\captionof{table}{\texttt{MNIST} cross-view results.
								\label{tab:ssl}}
							\vspace{.5em}
							\begin{tabular}{ccccccc}
								\toprule
								Model   & CCA       & NWJ       & TUBA  & InfoNCE  & FLO & FDV   \\
								\midrule
								Accuracy    & 67.78  &    76.71 & 79.49 & 79.27 & 79.47 &{\bf 80.14}      \\
								[3pt]
								$\hat{I}(x_l, x_r)$     & NA          & 5.73 & 4.78 & 4.65 & 4.84 & 4.67         \\
								\bottomrule
							\end{tabular}
							%   \includegraphics[width=1.\textwidth]{figures/results/ess-acc}
							%   \vspace{-2.em}
							%   \captionof{figure}{ESS scheduling results. \label{fig:ess-acc}}
						\end{minipage}
					}
					\vspace{-1.em}
				\end{figure}

				\section{Bayesian Experimental Design}
				
				\subsection{Noisy Linear Model}
				Our setup is the same as the Noisy Linear Model in \citep{kleinegesse2020bayesian}. We use 10  individual experimental designs. For encoder $\theta$ and encoder $y$, we use MLP with 2-layer, 128-dim hidden layer, and set the feature dim as 512. We train models in 5000 epochs, the batch size is 64, and the learning rate is $2*10^{-5}$. Four MI estimators (NWJ, TUBA, InfoNCE, and FLO) has been compared in this experiment and we got four optimized designs. Then, we use MCMC to estimate the posterior of the parameters.
				
				\subsection{Pharmacokinetic Model}
				The settings of this experiment refer to the Pharmacokinetic Model of \citep{kleinegesse2020bayesian}. We use 10  individual experimental designs. The MLP is with 2-layer, 128-dim hidden layer, and set the output feature dim as 512. We train 10000 epochs with learning rate is $10^{-5}$ via four methods (NWJ, TUBA, InfoNCE, FLO).
				
				\subsection{SIR Model}
				We here consider the spread of a disease within a population of N individuals, mod- elled by stochastic versions of the well-known SIR \citep{allen2008mathematical}. a susceptible state $S(t)$ and can then move to an infectious state $I(t)$ with an infection rate of $\beta$. These infectious individuals then move to a recovered state $R(t)$ with a recovery rate of $\gamma$, after which they can no longer be infected. The SIR model, governed by the state changes $S(t) \rightarrow I(t) \rightarrow R(t)$, thus has two model parameters $\thetav_1 = (\beta, \gamma)$. 
				% In the SEIR model, indicated by $m = 2$, susceptibles first move to an additional exposed state $E(t)$, where individuals are infected but not yet infectious. Afterwards, they move to the infectious state $I(t)$ with a rate of $\sigma$. The SEIR model (m = 2), governed by $S(t) \rightarrow E(t) \rightarrow I(t) \rightarrow R(t)$, thus has three model parameters $\thetav_2 = (\beta, \sigma, \gamma)$. We further make the common assumption that the total population size $N$ stays constant.
				
				The stochastic versions of these epidemiological processes are usually defined by a continuous-time Markov chain (CTMC), from which we can sample via the Gillespie algorithm \citep{allen2017primer}. However, this generally yields discrete population states that have undefined gradients. In order to test our gradient-based algorithm, we thus resort to an alternative simulation algorithm that uses stochastic differential equations (SDEs), where gradients can be approximated. 
				%Figure 7 shows example simulations of the SDE-based SIR and SEIR models, generated according to the method below.
				
				We first define population vectors $X_1(t) = (S(t), I(t))$ for the SIR model and $X_2(t) = (S(t), E(t), I(t))$ for the SEIR model. We can effectively ignore the population of recovered because the total population is fixed. The system of Itô SDEs for the above epidemiological processes is
				\beq
				\ud \Xv(t) = \fv(\Xv(t))\ud t + \Gv(\Xv(t)) \ud \Wv(t),
				\eeq
				where $\fv$ is the drift term, $\Gv$ is the diffusion term and $\Wv$ is the Wiener process. Euler-Maruyama algorithm is used to simulate the sample paths of the above SDEs. 
				% \texttt{torchsde} is used here for this. 
				
				\beq
				\fv_{\SIR} = \left( \begin{array}{c}
					- \beta\frac{S(t)I(t)}{N}\\
					\beta\frac{S(t)I(t)}{N} - \gamma I(t)
				\end{array}
				\right), \Gv_{\SIR} = \left( 
				\begin{array}{cc}
					-\sqrt{\beta\frac{S(t)I(t)}{N}} & 0\\
					\sqrt{\beta\frac{S(t)I(t)}{N}} & -\sqrt{\gamma I(t)}
				\end{array}
				\right)
				\eeq
				
				% $
				% \fv_{\SEIR} = \left( \begin{array}{c}
					% - \beta\frac{S(t)I(t)}{N}\\
					% \beta\frac{S(t)I(t)}{N} - \sigma E(t) \\
					% \sigma E(t) - \gamma I(t)
					% \end{array}
				% \right), \Gv_{\SEIR} = \left( 
				% \begin{array}{ccc}
					% -\sqrt{\beta\frac{S(t)I(t)}{N}} & 0 & 0\\
					% \sqrt{\beta\frac{S(t)I(t)}{N}} & -\sqrt{\gamma E(t)} & 0\\
					% 0 & \sqrt{\sigma E(t)} & -\sqrt{\gamma I(t)}
					% \end{array}
				% \right)
				% $
				
				We use the infection rate ($I$) as 0.1 and the recovery ($R$) rate as 0.01. The independent priors are N(0.1,0.02) and N(0.01, 0.002). The initial infection number is 10. We update MI one time after updating sampler three steps.We use RNN network with 2 layer 64 dim hidden layer construction to decoder the sequential design.

				% {\bf Model parameterization.} Now we want to show how different parameterization schemes affect the performance and learning efficiency for $\FLO$. In Figure 4 (paper), we visualize the learning dynamics of $\FLO$ using a shared network for  $(g_{\theta}, u_{\phi})$ and that with two separate networks. The parameter sharing not only cuts computations, it also helps to learn faster. There is no discernible difference in performance and $\FLO$-separate used twice much of iterations to converge. Next, we compare the bi-linear critic implementation ($\FLO$-\texttt{BiL}) to the standard MLP with paired inputs $(x,y)$ (Figure \ref{fig:bilinear}). In the bi-linear case $K$ is tied to batch-size so we scale the computation budget with $T(K) = (\frac{K}{K_0})^{\frac{1}{2}} \cdot T_0$, where $(T_0, K_0)$ are respectively the baseline training iteration and negative sample size used by $\FLO$-\texttt{MLP}. We see $\FLO$-\texttt{BiL} has drastically reduced computations (Figure \ref{fig:bilinear}). 

				\section{Meta Learning}
				
				{\bf Intuitions.} Now let us describe the new $\metaflo$ model for meta-learning. Given a model space $\CM$ and a loss function $\ell: \CM \times \CZ \rightarrow \BR$, the true risk and the empirical risk of $f\in \CM$ are respectively defined as $R_{t}(f) \triangleq \EE_{Z\sim\mu_t}[\ell(f, Z)]$ and $\hat{R}_t(f;\BS_t) \triangleq \frac{1}{m} \sum_{i=1}^m \ell(f, Z_i)$.Let us denote $R_{\tau}$ is the generalization error for the task distribution $\tau$ where all tasks originate, and $\hat{R}_{\tau}$ is the empirical estimate.  Our heuristic is simple, that is to optimize a tractable upper bound of the generalization risk given by
				\beq
				R_{\tau} \leq \underbrace{\hat{R}_{\tau}}_{\text{Utility}} + \underbrace{| R_{\tau}  - \hat{R}_{\tau}  |}_{\text{Generalization}} \triangleq \CL_{\upper}. 
				\label{eq:upper}
				\eeq
				% We can consider $\CL_{\upper}$ as a recapitalization of the standard utility-generalization trade-offs in machine learning. 
				
				For meta-learning, we sample $n$-tasks for training and $n'$-tasks for testing, respectively denoted as $\BS_{1:n}$ and $\BS\te_{1:n'}$. We further decouple the learning algorithm into two parts: the {\it meta-learner} $\CA_{\meta}(\BS_{1:n})$ that consumes all train data to get the {\it meta-model} $f_{\meta}$, and then {\it task-adaptation learner} $\CA_{\adapt}(f_{\meta}, \BS_t)$ which adapts the meta-model to the individual task data $\BS_t$ to get task model $f_t$. For parameterized models such as deep nets, we denote $\Theta$ as our {\it meta parameters} and $E_t$ as {\it task-parameters}, that is to say $\Theta \triangleq \CA_{\meta}(\BS_{1:n})$, $E_t \triangleq \CA_{\adapt}(\Theta, \BS_t)$, where $\Theta, E_t$ can be understood as weights of deep nets. In subsequent discussions, we will also call $E_t$ the {\it task-embedding}. We can define the population {\it meta-risk} as $R_{\tau}(\Theta) \triangleq \EE_{t, \Theta=\CA_{\meta}(\BS_{1:n})}[\EE_{E_t=\CA_{\adapt}(\Theta,\BS_t)}[R_{t}(f_{E_t})]]$, and similarly for the empirical risk $\hat{R}_{\tau}$ evaluated on the query set $\BQ_t$. Our model is based on the following inequality \citep{anonymous2022metaflo}:
				
				\beq
				\lim_{n\rightarrow\infty}| \EE[R -\hat{R}] | \leq  \sqrt{\frac{2\sigma^2}{m} I(E_{t};\BS_{t}|\Theta)}
				\label{eq:meta_bound}
				\eeq

				which gives the main objective $\CL_\texttt{Meta-FLO}(f) = \hat{R}(f) + \lambda \sqrt{I_{\FLO}(\hat{\CD}_t;\hat{E}_t)}$. We summarize our model architecture in Figure \ref{fig:arch}. 
				
				%  Let $\BS_t = \{ Z_i \}_{i=1}^m \sim \mu_t^m$ be a set of $m$ independent samples drawn from $\mu_t$, which we identify as the training dataset, which is also known as the {\it support set}, or an {\it episode}, in meta/few-shot learning contexts\footnote{With slight abuse of notation, we also use $\BS$ to denote a realization of training data ({\it i.e.}, $\{ z_i \}_{i=1}^m$), which simplifies the text.}. 
				%and a set of independent samples $\BS = \{ Z_i \}_{i=1}^m$ drawn from $\mu_t$. 
				
				\begin{figure}[t!]
					\begin{center}
						{
							\includegraphics[width=.7\textwidth]{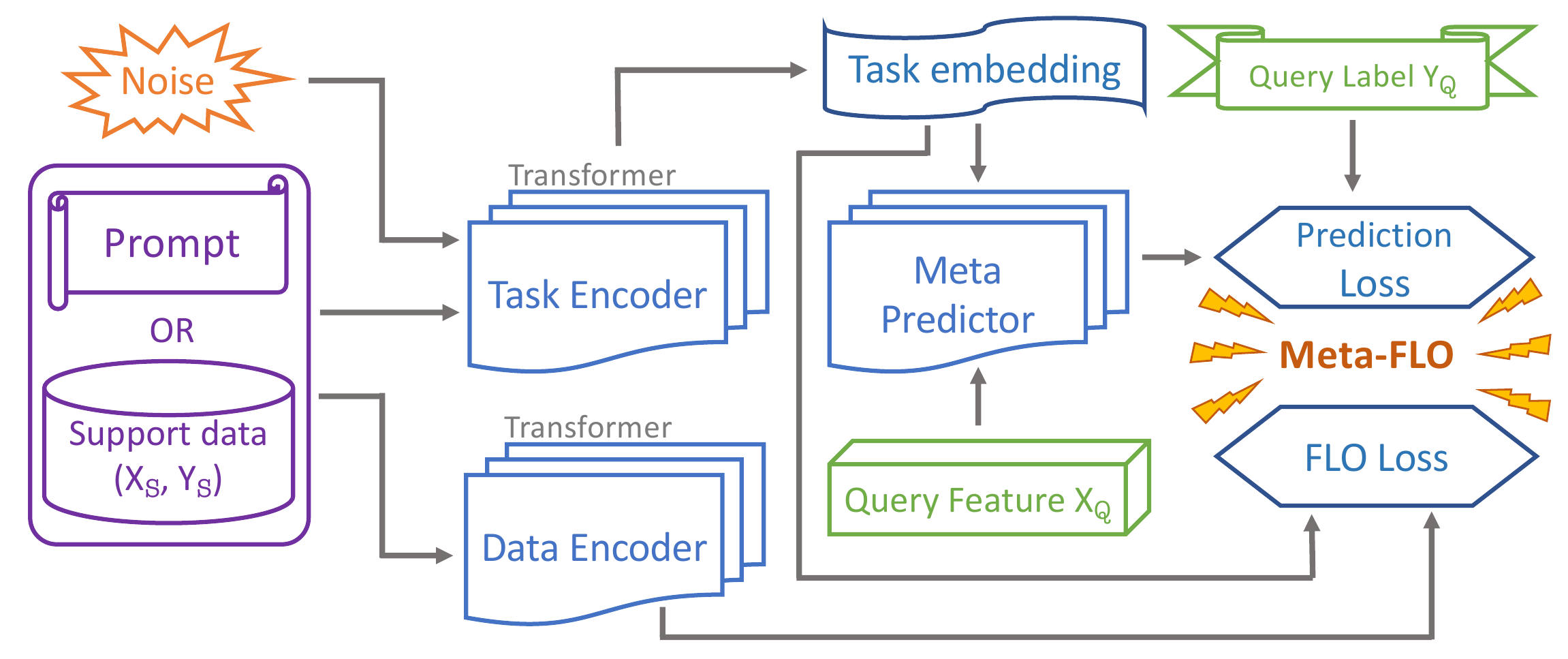}
						}
					\end{center}
					%\vskip -.1in
					%\caption{.}
					\vspace{-1.em}
					\caption{Model architecture of $\metaflo$. \label{fig:arch}}
					\vspace{-1.em}
				\end{figure}
				
				The sin-wave adaptation experiment involves regressing from the input ($x\sim \text{Uniform}([-5,5])$) to the output of a sine wave $\kappa \sin(x-\gamma)$, where amplitude $\kappa \sim \text{Uniform}([0.1,5])$ and phase ($\gamma \sim \text{Uniform}([0,\pi]$) of the sinusoid vary for each task. We use mean-squared error (MSE) as our loss and set the support-size = $3$ and query-size = $2$. We use simple three-layer MLPs for all the models: regressor, prompt encoder, and $\FLO$ critics, with hidden units all set to $[512, 512]$. During training, we use an episode-size of $64$. For MAML, we use the first-order implementation (FOMAML), and set inner learning rate to $\alpha=10^{-4}$. For $\metaflo$, we set regularization strength to $\lambda=10^{-2}$.

			\end{document}